\documentclass{article}
\usepackage{iclr2026_conference,times}

\usepackage{amsmath,amssymb,amsfonts,amsthm,mathtools}
\usepackage{bm}
\usepackage{graphicx}
\usepackage{tikz}
\usetikzlibrary{arrows.meta}
\usepackage{hyperref}
\usepackage{url}
\usepackage{booktabs}
\usepackage{subcaption}
\usepackage{microtype}

\theoremstyle{plain}
\newtheorem{theorem}{Theorem}
\newtheorem{lemma}[theorem]{Lemma}
\newtheorem{proposition}[theorem]{Proposition}
\newtheorem{corollary}[theorem]{Corollary}
\theoremstyle{definition}
\newtheorem{definition}[theorem]{Definition}

\DeclareMathOperator{\Lip}{Lip}
\DeclareMathOperator{\dist}{dist}

\DeclareMathOperator{\conv}{conv}
\DeclareMathOperator{\Cov}{Cov}
\DeclareMathOperator{\diag}{diag}

\DeclareMathOperator{\tr}{tr}
\DeclareMathOperator{\mmse}{mmse}

\DeclareMathOperator*{\argmax}{arg\,max}
\DeclareMathOperator*{\argmin}{arg\,min}
\DeclareMathOperator*{\esssup}{ess\,sup}
\newcommand{\R}{\mathbb{R}}
\newcommand{\E}{\mathbb{E}}
\newcommand{\Prob}{\mathbb{P}}
\newcommand{\N}{\mathcal{N}}
\newcommand{\gauss}{\gamma}
\newcommand{\eps}{\varepsilon}
\newcommand{\push}{{}_\#}
\newcommand{\norm}[1]{\left\lVert #1 \right\rVert}
\newcommand{\abs}[1]{\left\lvert #1 \right\rvert}
\newcommand{\inner}[2]{\langle #1, #2 \rangle}
\newcommand{\Jint}{\mathfrak{J}}

\newcommand{\DABI}{\mathrm{DABI}}
\newcommand{\CCI}{\mathrm{CCI}}

\title{Why Do Few-Step Text Latents Fail When Image Latents Work?\\
Non-Commitment at Sharp Categorical Readouts}

\author{Zhongyao Wang\\
Fudan University \quad Shanghai Innovation Institute\\
\texttt{25113050137@m.fudan.edu.cn}}

\iclrfinalcopy  

\begin{document}
\maketitle

\begin{abstract}
Deterministic few-step generation succeeds on continuous image latents but collapses to incoherent text on continuous text latents, and we show the cause is geometric rather than a training or scaling deficiency: a smooth, regularity-limited deterministic map cannot resolve a discrete branch choice before a sharp categorical readout, so few-step failure is governed by decoder sharpness, not transport accuracy.
In the \emph{overlapping} regime of real text autoencoders, we prove (Theorem~\ref{thm:noncommit}) that the posterior-mean terminal step flips tokens at the rate of the latent mass in an $O(s(t))$ tube around decision boundaries; the rate is set by decoder sharpness, not transport accuracy.
Two diagnostics, DABI (readout sharpness) and CCI (categorical commitment), measured on published checkpoints show that four independently built continuous-text decoders amplify a boundary-aligned perturbation far beyond a norm-matched isotropic one ($\DABI$ from $5{\times}10^2$ to ${>}10^5$), while image decoders have $\DABI \approx 1$.
Two mechanisms escape the continuous bound: categorical commitment (autoregressive decoders succeed despite \emph{sharper} readouts; removing it collapses generation) and stochastic re-injection (deterministic ODE at $K{=}4$ gives PPL $294$ versus SDE $50$ on the same model).
In the idealized \emph{separated} regime we prove matching sharp transport laws, including a dimension phase diagram: the deterministic stiffness needed to separate $M$ modes grows as $\Theta(\sqrt{\log M})$ once the latent dimension is $\Omega(\log M)$ (and as $M^{1/n}$ in fixed dimension), with a depth-$B$ hierarchy giving a $\sqrt{B}$-smaller per-step peak (Theorems~\ref{thm:gamma}--\ref{thm:hier}); a coarea identity links these to the overlapping tube (Theorem~\ref{thm:bridge}).
The result is an accuracy--depth--stiffness tradeoff: within the deterministic-continuous class the cost is irreducible, and both escapes step outside it.
\end{abstract}

\begin{figure}[t]
\centering
\resizebox{\textwidth}{!}{%
\begin{tikzpicture}[>={Latex[length=2mm]}, font=\small]
\node[font=\small\bfseries] at (2.1,6.20) {image latent};
\node[font=\small\bfseries] at (10.8,6.20) {text latent};
\draw[gray!35] (8.35,1.85) -- (8.35,5.95);
\shade[inner color=blue!22, outer color=blue!2] (2.1,4.4) circle (1.72);
\draw[blue!28, thin] (2.0,4.45) ellipse (1.50 and 1.26);
\draw[blue!34, thin] (2.15,4.40) ellipse (1.05 and 0.92);
\draw[blue!40, thin] (2.10,4.42) ellipse (0.62 and 0.55);
\shade[inner color=gray!70, outer color=gray!8] (2.10,4.42) circle (0.16);
\draw[densely dashed, gray!55] (2.10,4.42) circle (0.30);
\draw[->, gray!70] (3.95,4.4) -- (4.55,4.4);
\node[draw=gray!55, rounded corners, fill=gray!5, minimum width=1.3cm, minimum height=0.72cm, font=\footnotesize] (decA) at (5.25,4.4) {decoder};
\draw[->, gray!70] (5.95,4.4) -- (6.55,4.4);
\node[draw=gray!50, rounded corners, fill=blue!3, minimum width=1.55cm, minimum height=1.45cm] (imgA) at (7.35,4.4) {};
\begin{scope}
  \clip[rounded corners] (6.575,3.675) rectangle (8.125,5.125);
  \fill[blue!4] (6.575,3.675) rectangle (8.125,5.125);
  \fill[orange!35] (6.95,4.78) circle (0.13);
  \fill[blue!30] (6.70,3.95) -- (7.10,4.70) -- (7.50,3.95) -- cycle;
  \fill[blue!42] (7.25,3.95) -- (7.68,4.90) -- (8.10,3.95) -- cycle;
  \draw[blue!30] (6.575,3.95) -- (8.125,3.95);
\end{scope}
\draw[green!55!black, line width=1.1pt] (4.55,2.30) -- (4.70,2.13) -- (5.00,2.50);
\node[anchor=west] at (5.10,2.32) {$\DABI \approx 1$};
\fill[red!7]    (9.00,2.80) -- (10.20,2.80) -- (10.05,3.95) -- (9.00,3.95) -- cycle;
\fill[orange!9] (10.20,2.80) -- (11.40,2.80) -- (11.55,3.70) -- (10.05,3.95) -- cycle;
\fill[red!11]   (11.40,2.80) -- (12.60,2.80) -- (12.60,3.75) -- (11.55,3.70) -- cycle;
\fill[red!12]   (9.00,3.95) -- (10.05,3.95) -- (10.35,5.05) -- (9.00,5.00) -- cycle;
\fill[red!5]    (10.05,3.95) -- (11.55,3.70) -- (11.25,4.85) -- (10.35,5.05) -- cycle;
\fill[orange!11](11.55,3.70) -- (12.60,3.75) -- (12.60,4.90) -- (11.25,4.85) -- cycle;
\fill[orange!8] (9.00,5.00) -- (10.35,5.05) -- (10.15,6.00) -- (9.00,6.00) -- cycle;
\fill[red!9]    (10.35,5.05) -- (11.25,4.85) -- (11.50,6.00) -- (10.15,6.00) -- cycle;
\fill[red!13]   (11.25,4.85) -- (12.60,4.90) -- (12.60,6.00) -- (11.50,6.00) -- cycle;
\draw[gray!60] (9.00,2.80) rectangle (12.60,6.00);
\draw[gray!60] (9.00,3.95) -- (10.05,3.95) -- (11.55,3.70) -- (12.60,3.75);
\draw[gray!60] (9.00,5.00) -- (10.35,5.05) -- (11.25,4.85) -- (12.60,4.90);
\draw[gray!60] (10.20,2.80) -- (10.05,3.95) -- (10.35,5.05) -- (10.15,6.00);
\draw[gray!60] (11.40,2.80) -- (11.55,3.70) -- (11.25,4.85) -- (11.50,6.00);
\foreach \x/\y/\w in {9.55/3.35/know, 10.78/3.28/good, 12.02/3.24/time,
  9.52/4.42/work, 11.05/4.22/love, 11.98/4.30/need,
  9.55/5.52/make, 10.86/5.50/world, 12.02/5.46/here}
  \node[font=\scriptsize, text=gray!55] at (\x,\y) {\w};
\shade[inner color=gray!70, outer color=gray!8] (10.20,4.50) circle (0.16);
\draw[densely dashed, gray!55] (10.20,4.50) circle (0.30);
\draw[->, red!75!black, line width=1pt] (10.62,4.46) -- (9.92,4.55);
\node[red!75!black, font=\scriptsize] at (10.25,4.86) {flip};
\draw[->, gray!70] (12.80,4.4) -- (13.40,4.4);
\node[draw=gray!55, rounded corners, fill=gray!5, minimum width=1.3cm, minimum height=0.72cm, font=\footnotesize] (decB) at (14.10,4.4) {decoder};
\draw[->, gray!70] (14.80,4.4) -- (15.40,4.4);
\node[draw=gray!45, rounded corners, fill=gray!3, minimum width=1.75cm, minimum height=1.0cm, font=\ttfamily\scriptsize, text=gray!60, align=center] (txtB) at (16.30,4.4) {w\#rk ?o\\ v7e th\$};
\draw[red!75!black, line width=1.1pt] (13.55,2.50) -- (13.85,2.13);
\draw[red!75!black, line width=1.1pt] (13.85,2.50) -- (13.55,2.13);
\node[anchor=west] at (13.98,2.32) {$\DABI \gg 1$};
\node[font=\small, text=gray!55] at (8.6,1.30) {same smooth few-step map: a sharp readout flips text, not images};
\end{tikzpicture}%
}
\caption{\textbf{Why deterministic few-step generation works for images but not text.} A smooth few-step map delivers each latent only to within an $O(s(t))$ posterior-mean blur (the fuzzy disk). \textbf{(left)} An image decoder is smooth and has no categorical readout, so the blur is absorbed (decoder amplification $\DABI\approx1$) and the output stays correct. \textbf{(right)} A text decoder reads out by $\argmax_y w_y^\top z$, partitioning the latent into sharp argmax cells; the \emph{same} blur, landing on a decision boundary, flips the token ($\DABI\gg1$) and the text becomes incoherent. Failure is set by the sharpness of the readout, not by transport accuracy (Theorem~\ref{thm:noncommit}). Two mechanisms leave the deterministic-continuous class and escape the bound: categorical commitment (autoregressive/masked decoders) and stochastic re-injection (SDE); the $\DABI{\times}\CCI$ taxonomy (Section~\ref{sec:taxonomy}) predicts which systems fail.}
\label{fig:teaser}
\end{figure}
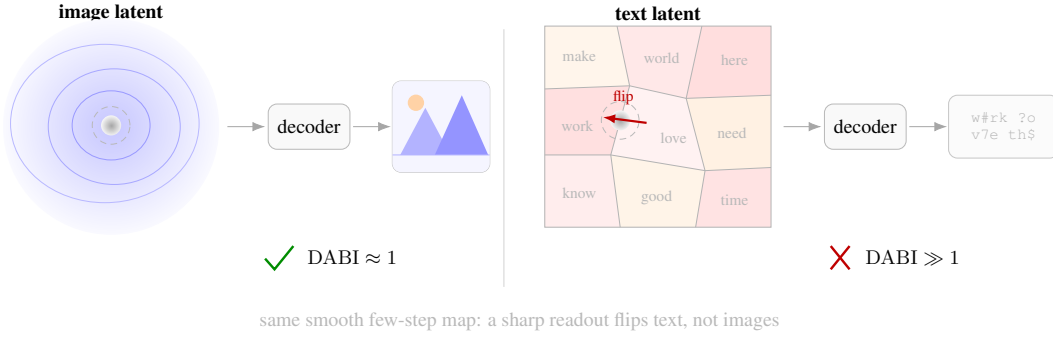

\section{Introduction}\label{sec:intro}

Deterministic few-step generation of images has advanced rapidly.
Consistency models~\citep{song2023consistency}, rectified flow~\citep{liu2023rectified}, and progressive distillation~\citep{salimans2022progressive} produce high-quality images in $1$--$4$ network evaluations by learning smooth transport maps from noise to data.
The same approach, applied to continuous text latents, fails (Figure~\ref{fig:teaser}).
A deterministic few-step generator of the latent space of a text autoencoder (such as ELF~\citep{he2026elf}) produces incoherent text, with repeated tokens and multilingual fragments, at $K \le 16$ steps, becoming usable only at $K \ge 32$--$64$.
Our evidence indicates this failure is not primarily explained by training, architecture, or scheduling, but is structural: when the decoder's readout boundaries are sharp relative to the transport residual, a smooth deterministic map cannot resolve the discrete branch choice within a few steps.
We formalize this for the posterior-mean terminal step and verify the prediction on published checkpoints; the bound applies to deterministic transport, which stochastic and categorical generators escape.

\paragraph{The failure mode.}
A text autoencoder encodes discrete tokens into a continuous latent $z = E(x)$ and decodes via $D_W(z) = \argmax_y w_y^\top z$.
A deterministic few-step generator is a composition of smooth maps with bounded Lipschitz constant; the readout amplifies small structured displacements near decision boundaries into token errors.
This failure mode appears in two regimes.
In the \emph{overlapping} regime (real text autoencoders), we prove (Theorem~\ref{thm:noncommit}) that the posterior-mean terminal step flips tokens at the readout-calibrated non-commitment rate: the latent mass in an $O(s(t))$ tube around decision boundaries. Failure is therefore set by decoder sharpness, not transport accuracy.
In the idealized \emph{separated} regime, transport to $M$ well-separated modes costs interface-energy stiffness (Theorems~\ref{thm:gamma}--\ref{thm:hier}).
The separated-mode $\Jint$ does \emph{not} predict real ELF residuals, yet a coarea identity (Theorem~\ref{thm:bridge}) links the two as contractions of one composite-interface geometry.

\paragraph{Main results.}
\emph{(1) Non-commitment theorem.}
The flip rate equals the readout-calibrated non-commitment $\ge$ Bayes non-commitment; asymptotically, $\Prob(\hat{Y}_t \ne Y) \asymp \Prob[\delta^*(X,Y) \lesssim s(t)]$.
An oracle roll-in on ELF (we noise a clean latent and integrate the deterministic ODE back to $t{=}1$ from each step; Figure~\ref{fig:rollin}) confirms this: $12$--$41\%$ terminal flip compounds to ${\approx}90\%$, all $K$ collapse onto one $s(t)$ curve, and the residual is structured ($10$--$33\times$ isotropic) with a ${\approx}1\%$ boundary-normal fraction.
\emph{(2) Image/text dichotomy.}
We define DABI (Decoder Amplification of Boundary-aligned Inputs): $\DABI = \Delta\mathrm{CE}_{\mathrm{struct}}/\Delta\mathrm{CE}_{\mathrm{rand}}$ at matched perturbation norm.
Under one margin-normal probe applied identically to all four text codecs, a boundary-aligned perturbation flips $47$--$77\%$ of tokens while a norm-matched isotropic one flips ${\approx}0\%$ (ELF $\DABI = 508\times$ $[446,580]$; LangFlow~\citep{langflow2026}, CoLa-DLM~\citep{cola2026}, Cosmos~\citep{cosmos2025} have isotropic response at the floor, so $\DABI \gg 10^{4}$; Table~\ref{tab:dabi_roster}, $95\%$ bootstrap CIs).
For the two flow models with a deterministic generator, the realized terminal residual already triggers this sharpness: it flips $21\times$ (ELF) and ${\approx}2{,}600\times$ (LangFlow) more tokens than an isotropic residual of the same norm.
By contrast, $4$ published image VAEs $\times$ $5$ NFE have $\DABI \in [0.85, 1.94]$: an image/text gap of orders of magnitude with no overlap.
\emph{(3) Escape mechanisms.}
\emph{Stochastic:} on the same ELF-B teacher~\citep{he2026elf}, ODE $K{=}4$ PPL${}=294$ vs SDE PPL${}=50$; even the distilled PD student~\citep{he2026elfpd} shows a persistent ODE-to-SDE gap ($15$--$38\%$ for $K \ge 2$).
\emph{Categorical:} AR decoders have $\DABI = 915$--$19{,}805\times$ (sharper than ELF) yet succeed because categorical commitment makes the generator discontinuous (Lemma~\ref{lem:escape}); removing commitment collapses generation (Dream: $5\times$ PPL jump).
\emph{(4) Sharp transport laws.}
$\lim_{\Lambda\to\infty}\Lambda\inf W_p^p = \Jint_{p,\sigma}$ (Theorem~\ref{thm:gamma}, hard-gradient $\Gamma$-limit); $\mathcal{P}_{n,M} \asymp \sqrt{\log M}$ for $n \ge C\log M$ (Theorem~\ref{thm:dim}); depth-$B$ hierarchy achieves peak ratio $\sqrt{B}$ (Theorem~\ref{thm:hier}).

\paragraph{Scope.}
All lower bounds apply to \emph{deterministic} transport.
This is an accuracy--depth--stiffness tradeoff, not an impossibility: two escapes (categorical commitment, stochastic re-injection) leave the continuous-deterministic class.

\paragraph{Related work.}
Theorem~\ref{thm:gamma} is the hard-Lipschitz counterpart of the Baldo/Fonseca--Tartar multiwell $\Gamma$-limit \citep{baldo1990,fonseca1989gradient}.
The Gaussian-width identity connects to generic chaining \citep{talagrand2005chaining} and Milman--Neeman \citep{milman2022gaussian}.
\citet{salmona2022lipschitz} showed Lipschitz pushforwards are limited on multimodal targets; our contribution is the dimension-sensitive interface law and the categorical-escape identification.
FMLM \citep{fmlm2026} is an instantiation of categorical escape (Section~\ref{sec:recent}).
ELF+PD \citep{he2026elfpd} confirms the ODE-vs-SDE gap on the model we study.
Extended related work is in Appendix~\ref{app:related}.

\section{Setup and Definitions}\label{sec:setup}

\paragraph{Source and target.}
The source is a probability measure $\sigma = \rho\,dx$ on $\R^n$ with strictly positive $C^1$ density (the standard Gaussian $\gauss_n$ being the leading example).
The target is concentrated near $M$ codebook points $\mathcal{C} = \{c_1, \ldots, c_M\} \subset \R^d$ with weights $\pi_i > 0$; we write $\mu = \sum_i \pi_i \delta_{c_i}$.
In the separated regime, $\min_{i \ne j} \norm{c_i - c_j} \ge \Delta > 0$.
In the overlapping regime (real text autoencoders), per-token posteriors overlap and $\Delta$ is not a meaningful parameter.

\paragraph{Linear readout and normalized margin.}
A linear readout decodes $z \in \R^d$ to a token: $D_W(z) = \argmax_y w_y^\top z$, where $W = [w_1, \ldots, w_V]$ is the readout matrix with vocabulary size $V$.
The normalized margin of the correct token $y$ against competitor $j$ is
\begin{equation}\label{eq:margin}
\delta^*_{yj}(z) = \frac{(w_y - w_j)^\top z}{\norm{w_y - w_j}},
\qquad
\delta^*(z, y) = \min_{j \ne y} \delta^*_{yj}(z).
\end{equation}
The decision boundaries are the hyperplanes $\Sigma_{ij} = \{z : w_i^\top z = w_j^\top z\}$.

\paragraph{Flow interpolant and posterior mean.}
Let $X$ denote the clean latent and $U_t = tX + (1-t)s_0 \eps$, $\eps \sim \gauss_n$, the flow interpolant at time $t$ with noise scale $s_0$.
The population-optimal $x$-prediction (terminal Euler step) is the posterior mean $m_t(u) = \E[X \mid U_t = u]$.
The effective noise scale at time $t$ is $s(t) = s_0(1-t)/t$; small $t$ means large noise, large $t$ means small noise.

\paragraph{Barrier metric and interface energy.}
For wells $\mathcal{C}$, put $V(y) = \dist(y, \mathcal{C})^p$ and define the barrier metric
\begin{equation}\label{eq:barrier}
\kappa_p(i,j) = \inf_{\xi: c_i \to c_j} \int_0^1 V(\xi(s)) \norm{\dot{\xi}(s)}\,ds,
\end{equation}
the least $V$-weighted path length between two wells.
The weighted interface energy and its prescribed-mass profile are
\begin{equation}\label{eq:Jdef}
\mathcal{E}_{p,\sigma}(\mathcal{B}; \mathcal{C}) = \sum_{i < j} \kappa_p(i,j)\, P_\sigma(B_i, B_j),
\qquad
\Jint_{p,\sigma}(\pi; \mathcal{C}) = \inf_{\sigma(B_i) = \pi_i} \mathcal{E}_{p,\sigma}(\mathcal{B}; \mathcal{C}),
\end{equation}
where $P_\sigma(B_i, B_j) = \int_{\partial^* B_i \cap \partial^* B_j} \rho\,d\mathcal{H}^{n-1}$ is the weighted perimeter and the infimum is over Caccioppoli partitions with prescribed masses.

\begin{definition}[DABI: decoder amplification of boundary-aligned inputs]\label{def:dabi}
Given a decoder $D$, let $r_{\mathrm{struct}}$ be a unit \emph{boundary-aligned} direction and $r_{\mathrm{rand}}$ a norm-matched isotropic Gaussian perturbation.
At perturbation fraction $f$:
\[
\DABI(D) = \frac{\Delta \mathrm{CE}(f \cdot r_{\mathrm{struct}})}{\Delta \mathrm{CE}(f \cdot r_{\mathrm{rand}})}\bigg|_{f=1}.
\]
We use two boundary-aligned directions.
The \emph{margin-normal} direction points toward the nearest decision boundary (the readout normal $w_y - w_j$, or its local gradient for a nonlinear decoder) and reaches that boundary at $f = 1$.
It is the canonical worst-case readout-sharpness probe; because it is defined for \emph{any} categorical readout, it gives one apples-to-apples number across all codecs.
The \emph{realized-residual} direction is the actual posterior-mean residual of a deterministic generator, available only when the system has one; it measures how much that readout sharpness affects real generation.
Text decoders with hard $\argmax$ boundaries have $\DABI \gg 1$ under both; smooth image decoders (no categorical readout) have $\DABI \approx 1$ on the realized residual.
For text, $\Delta\mathrm{CE}$ is the cross-entropy increase; for images, it is the pixel-$L^2$ increase.
Each is the natural output-space loss for its domain.
The ratio $\DABI$ is dimensionless within each domain: it measures how much more damaging a structured perturbation is than an isotropic one at the same norm, which is well-defined regardless of loss scale.
Cross-domain comparison is meaningful because $\DABI \approx 1$ indicates isotropic sensitivity (no boundary structure) while $\DABI \gg 1$ indicates that boundaries concentrate the sensitivity.
\end{definition}

\begin{definition}[CCI: categorical commitment index]\label{def:cci}
The CCI of a generator at step $k$ is the fraction of positions where an \emph{irreversible} discrete selection (sampling from a categorical distribution followed by fixing the token for all later steps) has been made by step $k$.
At termination: autoregressive generation has $\CCI = 1$ (every position committed); continuous flow decoding has $\CCI = 0$ (no discrete selection ever occurs).
Masked diffusion with $K$ steps commits a growing subset of positions per step, reaching $\CCI = 1$ at termination; the generation-time average $\overline{\CCI} \in (0,1)$ reflects the schedule.
\end{definition}

\section{The Non-Commitment Mechanism}\label{sec:noncommit}

This section establishes the central result: a deterministic terminal step outputs a posterior mean whose decoded-token flip rate is determined by the readout decision-boundary geometry.

\begin{theorem}[readout-calibrated non-commitment]\label{thm:noncommit}
Assume clean decoding is exact: $Y = D_W(X) = \argmax_y w_y^\top X$ almost surely (no label noise).
Let $\hat{Y}_t = D_W(m_t(U_t))$ be the decoded token of the posterior-mean terminal step, and $q_t(y \mid u) = \Prob(Y = y \mid U_t = u)$ the posterior over tokens.
Then:
\begin{enumerate}
\item[\textup{(i)}] \textbf{Equality and Bayes bound.}
\begin{equation}\label{eq:noncommit-eq}
\Prob(\hat{Y}_t \ne Y) = \E\big[1 - q_t(D_W(m_t(U_t)) \mid U_t)\big] \;\ge\; \E\big[1 - \max_y q_t(y \mid U_t)\big],
\end{equation}
with $\E[-\log p_W(Y \mid m_t(U_t))] \ge (\log 2)\, \E[1 - \max_y q_t(y \mid U_t)]$ for the cross-entropy.

\item[\textup{(ii)}] \textbf{Active-facet tube law.} If $\rho_X$ is smooth with regular decision boundaries and no positive-mass triple intersections, then as $s(t) = s_0(1-t)/t \downarrow 0$,
\begin{equation}\label{eq:tube}
\Prob(\hat{Y}_t \ne Y) = \sqrt{\frac{2}{\pi}}\, s(t)\, \sum_{i<j} \int_{F_{ij}} \rho_X\, d\mathcal{H}^{d-1} + o(s(t)) \;\asymp\; \Prob\big[\delta^*(X,Y) \lesssim s(t)\big],
\end{equation}
where $F_{ij} = \partial\Omega_i \cap \partial\Omega_j$ is the \emph{active co-maximal facet} between the argmax cells $\Omega_i = \{z : D_W(z) = i\}$, i.e.\ the $(d{-}1)$-faces where tokens $i,j$ are co-maximal.
Only adjacent decision facets contribute: a non-adjacent pairwise hyperplane $\Sigma_{ij}$ lies interior to a third token's cell and carries no flips.
The posterior-mean readout $D_W(m_t)$ and the Bayes/MAP classifier $\argmax_y q_t(y\mid\cdot)$ share this leading constant, since their switching surfaces agree to $O(s^2)$ by Tweedie's formula (Appendix~\ref{app:noncommit}).
\end{enumerate}
\end{theorem}

\begin{proof}[Proof sketch (full proof in Appendix~\ref{app:noncommit})]
\emph{Equality.} Condition on $U_t = u$: the deterministic output $m_t(u)$ and hence the decoded token $d = D_W(m_t(u))$ are fixed, while the true token $Y$ is distributed as $q_t(\cdot \mid u)$.
Thus $\Prob(\hat{Y}_t \ne Y \mid U_t = u) = 1 - q_t(d \mid u)$; integrating gives the equality.
The Bayes bound follows because $q_t(d \mid u) \le \max_y q_t(y \mid u)$.

\emph{CE bound.} For any $y \ne d$, $-\log p_W(y \mid m_t) \ge \log(1 + e^{(w_d - w_y)^\top m_t}) \ge \log 2$ because $d$ is the $\argmax$.
Weighting by $q_t(y \mid u)$ and summing over $y \ne d$ gives the stated bound.

\emph{Tube law.} Rescale: $U_t/t = X + s(t)\eps$ with $s(t) = s_0(1-t)/t$, so $\max_y q_t(y \mid U_t)$ is the posterior of $Y$ given $X + s(t)\eps$.
The Bayes non-commitment at noise scale $s$ is the $\gauss$-blurred boundary mass.
A one-dimensional Laplace expansion across each active facet $F_{ij}$, using smoothness of $\rho_X$ and regularity of the boundaries, gives the leading term; non-adjacent hyperplanes contribute only $o(s)$.
The equivalence $\asymp \Prob[\delta^* \lesssim s(t)]$ follows because the normalized distance to the nearest boundary is $\delta^*$.
\end{proof}

\paragraph{Relevance to learned generators.}
Theorem~\ref{thm:noncommit} is stated for the population posterior mean.
A trained denoiser with squared-loss targets $m_t$; at convergence, its terminal step approximates the posterior-mean step.
The oracle roll-in (Section~\ref{sec:rollin}) confirms that the \emph{published ELF teacher's learned ODE} tracks the posterior-mean predictions: all $K$ collapse onto the $s(t)$ curve, and structured-vs-random ratios ($10$--$33\times$) match DABI.

\paragraph{Two features matching the data.}
\emph{(i) Only the row space matters.}
A flip requires $(w_Y - w_j)^\top (X - m_t) / \norm{w_Y - w_j} \ge \delta^*_{Yj}(X)$: only the projection of the residual $X - m_t$ onto the boundary normals $w_Y - w_j$ enters.
On ELF, the terminal residual is $97$--$99\%$ in the decoder null space, yet the small (${\sim}1\%$) structured component along boundary normals suffices to flip tokens.

\emph{(ii) Sharpness is decisive.}
Small normalized margins $\delta^*$ make a tiny structured displacement enough to flip the token.
A smooth image decoder, lacking a hard $\argmax$ boundary, tolerates the same posterior-mean blur that flips tokens under a sharp text readout.
This is the mechanism behind the image/text dichotomy.

\paragraph{Two contractions of one interface (the bridge).}
Theorem~\ref{thm:noncommit} (overlapping regime) and Theorems~\ref{thm:gamma}--\ref{thm:dim} (separated regime) are not a single scalar functional: $\Jint$ does \emph{not} predict real ELF residuals.
They are, however, two contractions of \emph{one} geometric object.
Consider the composite source-space classifier $F = D_W \circ T$.
A coarea identity (Theorem~\ref{thm:bridge}, Appendix~\ref{app:bridge}) writes the readout boundary mass $A_W = \sum_{i<j}\int_{F_{ij}}\rho_X\,d\mathcal H^{d-1}$ as the source-interface perimeter, discounted pointwise by the readout-normal stretch $\norm{DT^\top n_{ij}}$.
The separated energy $\Jint$ is the barrier-weighted stretch moment of that same source interface.
Composing with the tube law gives the rigorous accuracy--depth--stiffness product law: for a $K$-step generator $T = T_K \circ \cdots \circ T_1$,
\begin{equation}\label{eq:bridge-main}
\liminf_{s\downarrow0}\frac{\prod_{k}\Lip(T_k)}{s}\,\Prob(\hat Y_s \ne Y)\;\ge\;\sqrt{\tfrac{2}{\pi}}\,\frac{\Jint_{p,\sigma}(\pi;C)}{\kappa_{\max}(C)}.
\end{equation}
The separated $\Jint$ lower-bounds the deterministic stiffness budget; the overlapping tube law converts the leftover noise into token error.
Equation~\eqref{eq:bridge-main} holds under three per-system-checkable preconditions on $F = D_W\circ T$: (i) \emph{transversality} ($DT$ full rank, no tangential degeneracy at the interface); (ii) \emph{alignment} (the source interface is exhausted by $T^{-1}(F_{ij})$); (iii) \emph{regularity} (the finite-perimeter and density-trace conditions of Appendix~\ref{app:bridge}). On ELF the roll-in results (Section~\ref{sec:rollin}) are consistent with these premises at the probed resolution; we apply \eqref{eq:bridge-main} as a \emph{conditional} product law whose only alignment failure mode weakens the bound by an explicit, measurable off-facet energy $E_{\mathrm{off}}$ (Appendix~\ref{app:bridge}), not as a global certificate.
No $W$-blind scalar functional reproduces both constants (Proposition~\ref{prop:dichotomy}): the overlapping cost depends on the readout $W$ while the separated cost depends on the codebook scale, so they are genuinely distinct contractions of the shared interface, not a single functional.

\section{Empirical Evidence}\label{sec:empirics}

All experiments use published checkpoints: the ELF-B text autoencoder \citep{he2026elf} (latent dimension $d = 512$, sequence length $S = 128$), the released image-VAE decoders of Lumina-Next \citep{luminanext2024}, SANA-1.5 \citep{sana2024}, Z-Image, and FLUX.1 (AutoencoderKL / AutoencoderDC variants, used at their default configurations), and published AR models (Llama-2~\citep{touvron2023llama2}, Qwen-2.5~\citep{qwen2025}, Qwen-3).
Exact checkpoint identifiers and preprocessing are listed in the released code.
Decoder sensitivity uses $n = 5{,}000$ token positions (OpenWebText validation); roll-in uses $n = 512$ paired latents; the $\gamma$-sweep covers $60$ cells with $256$ samples per cell.
All results are deterministic given the checkpoint and data split.

\subsection{Decoder sensitivity}\label{sec:decsens}

We characterize the ELF-B decoder's response to structured versus random perturbations on $n = 5{,}000$ token positions (Figure~\ref{fig:decsens}).

\emph{Basin width and structured anisotropy.}
The ELF-B encoder achieves $96.7\%$ clean token recovery ($\eta_0 = 3.3\%$ label noise; the offset is additive, Appendix~\ref{app:noncommit}), unchanged under $\sigma = 1.0$ isotropic noise, consistent with a median normalized margin $\delta^* = 5.6$ (Eq.~\ref{eq:margin}; raw logit gap $21.5$), since isotropic noise has a unit-variance boundary-normal component and rarely crosses a $5.6\sigma$ margin.
The actual $K{=}1$ student residual at matched norm, however, flips $70.4\%$ of tokens (vs $3.3\%$ random), a $21.2\times$ ratio, despite a row-space component of only $2.49\%$; and the MSE gradient is orthogonal to the CE gradient (cosine $0.001$).

\emph{Continuous response and DABI.}
Sweeping the perturbation fraction $f \in [0, 1.5]$: at $f = 1.0$ (full student residual), the structured cross-entropy increase is $\Delta\mathrm{CE}_{\mathrm{struct}} = 6.07$ versus $\Delta\mathrm{CE}_{\mathrm{rand}} = 0.13$, giving $\DABI = 45.7\times$.
The structured response is superlinear with onset at $f \approx 0.8$, while the random response remains sublinear throughout.

\emph{Decoder-intervention probe.}
A natural fix is to blunt the readout. We test the cheapest decoder intervention, test-time logit smoothing ($D_\sigma(z) = \argmax_y \E_\xi[\,\text{logits}(z + \sigma\,\mathrm{rms}(z)\,\xi)\,]$, which averages the readout over an isotropic neighborhood), sweeping $\sigma \in [0, 0.8]$.
It does not lower DABI: $\DABI$ stays at $48$--$58$ and the structured flip rate at $0.86$ across the whole sweep, even as clean recovery begins to fall ($1.00 \to 0.99$).
Isotropic smoothing averages over the ${\approx}97\%$ null-space component and barely touches the ${\approx}1\%$ boundary-normal subspace where the sensitivity lives, so it cannot reduce DABI.
The implication is causal: a fix must reshape the boundary geometry, not average isotropically.
We test this by retraining the ELF readout head directly, transport held byte-identical (Appendix~\ref{app:causal}): neither max-margin retraining (which \emph{raises} $\DABI$, $453 \to 1{,}469\times$) nor retraining the readout to decode the few-step latents (even while collapsing its margins, $\delta^*\!: 5.4 \to 3.0$) helps, the latter recovering only $1.2$ of the $41$ flipped points at $K{=}4$ ($59.3 \to 60.5\%$ held-out accuracy).
No readout change recovers the flipped tokens: once the deterministic posterior mean averages over the branch ambiguity, the correct token is no longer a function of the latent.
Decoder sharpness amplifies a transport-side failure; it is not an independently fixable cause.

\begin{figure}[t]
\centering
\includegraphics[width=0.8\textwidth]{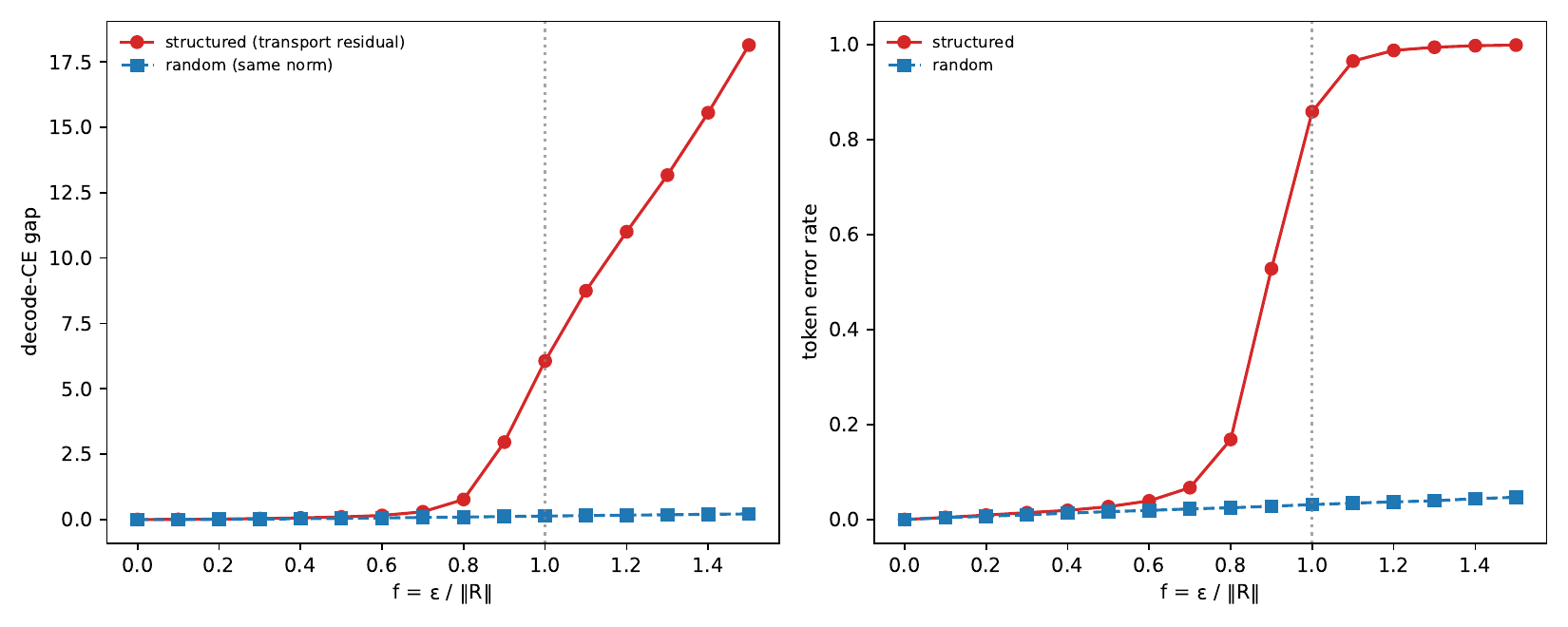}
\caption{\textbf{Decoder sensitivity.} ELF-B ($d = 512$, $S = 128$).
\textit{Left:} Structured versus isotropic CE response; $\DABI = 45.7\times$ at $f = 1$.
\textit{Right:} Flip rate; student residual at matched norm flips $21.2\times$ more tokens.}
\label{fig:decsens}
\end{figure}

\subsection{Image/text dichotomy}\label{sec:dabi}

\paragraph{One probe, four text codecs.}
To compare codecs apples-to-apples, we apply the \emph{same} margin-normal probe (Definition~\ref{def:dabi}: perturb each clean latent toward its nearest decision boundary, reaching the boundary at $f = 1$, against a norm-matched isotropic control) to all four published continuous-text decoders, with $95\%$ bootstrap CIs over positions (Table~\ref{tab:dabi_roster}; visualized in Appendix Figure~\ref{fig:dabi_scatter}).
The boundary-aligned perturbation flips $47$--$77\%$ of tokens; the isotropic one flips ${\approx}0\%$.
On ELF the isotropic response is small but nonzero, giving a well-conditioned $\DABI = 508\times$ $[446, 580]$ ($32{,}768$ positions); on LangFlow~\citep{langflow2026}, CoLa-DLM~\citep{cola2026}, and Cosmos~\citep{cosmos2025} the isotropic response sits at the floor, so $\DABI$ exceeds $10^{4}$.
Built very differently (frozen-T5~\citep{raffel2020t5} encoding, embedding rounding, stochastic VAE, non-autoregressive autoencoder), the four codecs share the same sharp readout.

\paragraph{Readout sharpness affects real generation.}
The margin-normal probe measures worst-case sharpness; for the two systems with a deterministic generator we can also use the \emph{realized} terminal residual.
The realized column is defined only here: CoLa-DLM (SDE${+}$AR) and Cosmos ship no few-step deterministic generator, so neither has a posterior-mean residual to probe.
On ELF it flips $21\times$ more tokens than a norm-matched isotropic residual ($\DABI = 45.7\times$ in CE); on LangFlow ($d = 768$, $99.99\%$ clean decoding) it flips ${\approx}2{,}600\times$ more ($95\%$ CI $[1{,}800, 3{,}800]$ over $24{,}960$ positions; $\Delta\mathrm{CE}_{\mathrm{struct}} = 40.5\,[39.8, 41.1]$ vs isotropic $3.9\times10^{-4}$).
LangFlow's deterministic few-step generation is correspondingly degenerate (Appendix~\ref{app:extended}).
The realized residual is ${\approx}99\%$ in the decoder null space, yet its small boundary-aligned component flips tokens (Theorem~\ref{thm:noncommit}).

\paragraph{Image decoders absorb the same perturbation.}
On four families of published image VAE decoders (Lumina-Next~\citep{luminanext2024}, SANA-1.5~\citep{sana2024}, Z-Image, FLUX.1; $5$ NFE each, Appendix Figure~\ref{fig:dabi_scatter}), the realized residual gives $\DABI \in [0.85, 1.94]$: lacking a categorical readout, an image decoder has no boundary to align with and absorbs the same posterior-mean blur that fails a text decoder. The image/text gap is more than $20\times$ with no overlap.
The ${\approx}1$ persists under the learned LPIPS perceptual metric (ratios $0.7$--$1.2$, Appendix~\ref{app:causal}), so it is not a pixel-$L^2$ artifact.
The dichotomy tracks the readout, not the domain: a \emph{categorical} (vector-quantized) image readout recovers the text signature, a boundary-aligned perturbation flipping the MoVQGAN code $167\times$ more than isotropic ($\DABI=42\times$) despite a $4$-dimensional latent with almost no null space (Appendix~\ref{app:causal}).
The conclusion rests on \emph{within}-domain ratios (image decoders near $1$, text codecs above $10^2$, no overlap), so it does not hinge on comparing cross-entropy to pixel-$L^2$ across domains: each $\DABI$ uses the natural output-space loss of its own domain (Definition~\ref{def:dabi}). AR and masked-dLM readouts are uniformly sharp under the same probe ($46$--$57\%$ vs ${\approx}0\%$; seven models, Appendix~\ref{app:recent}).

\begin{table}[t]
\centering
\caption{\textbf{Decoder amplification across systems} (margin-normal probe, $95\%$ bootstrap CIs over positions). All four continuous-text codecs flip $47$--$77\%$ of tokens vs ${\approx}0\%$ isotropic; ELF's isotropic response is nonzero (finite $\DABI$), the others are at the floor ($\DABI$ a lower bound). \emph{realized} column: realized terminal residual. Published checkpoints.}
\label{tab:dabi_roster}
\small
\setlength{\tabcolsep}{4.5pt}
\begin{tabular}{llccc}
\toprule
System & Latent (dim) & flip\,@$\kappa{=}1$ & $\DABI$ (mn) & realized \\
\midrule
ELF~\citep{he2026elf} & frozen-T5 ($512$) & $47\%$ & $508\times$ & $45.7\times$ \\
LangFlow~\citep{langflow2026} & embed-flow ($768$) & $76\%$ & ${\gg}10^{4}$ & ${\approx}2{,}600\times$ \\
CoLa-DLM~\citep{cola2026} & VAE${+}$AR ($16$) & $77\%$ & ${\approx}4{\times}10^{5}$ & --- \\
Cosmos~\citep{cosmos2025} & non-AR ($16{\times}768$) & $49\%$ & ${\approx}4{\times}10^{4}$ & --- \\
\midrule
Lumina/SANA/Z-Image/FLUX & image VAE & --- & n/a & $[0.85, 1.94]$ \\
\bottomrule
\end{tabular}
\end{table}

\subsection{Oracle roll-in}\label{sec:rollin}

We test Theorem~\ref{thm:noncommit} on the real ELF sampler (logit-normal grid, $p_{\mathrm{mean}} = -1.5$, $p_{\mathrm{std}} = 0.8$; Figure~\ref{fig:rollin}): for each grid time $t_i$ we build the exact forward marginal $U_{t_i} = t_i X + (1-t_i) s_0 \eps$ from paired clean data, roll the deterministic ODE suffix to $1$, and decode.
The terminal flip is mild and decreasing in $K$ ($41\%, 21\%, 12.1\%$ for $K = 4, 8, 16$, as the grid places the final step at later times, lower $s(t)$), but rolling in from earlier, more ambiguous states the flip climbs to ${\approx}90\%$; all $K$ collapse onto a single curve in $s(t)$, the per-step driver predicted by the theorem.
At every point the residual flips $10$--$33\times$ more tokens than a norm-matched isotropic control (the $45.7\times$ DABI of Section~\ref{sec:decsens}) with a constant ${\approx}1\%$ boundary-normal fraction: the mechanism is the composition of structured, readout-aligned posterior-mean errors, not a null-space effect.

\begin{figure}[t]
\centering
\includegraphics[width=0.82\textwidth]{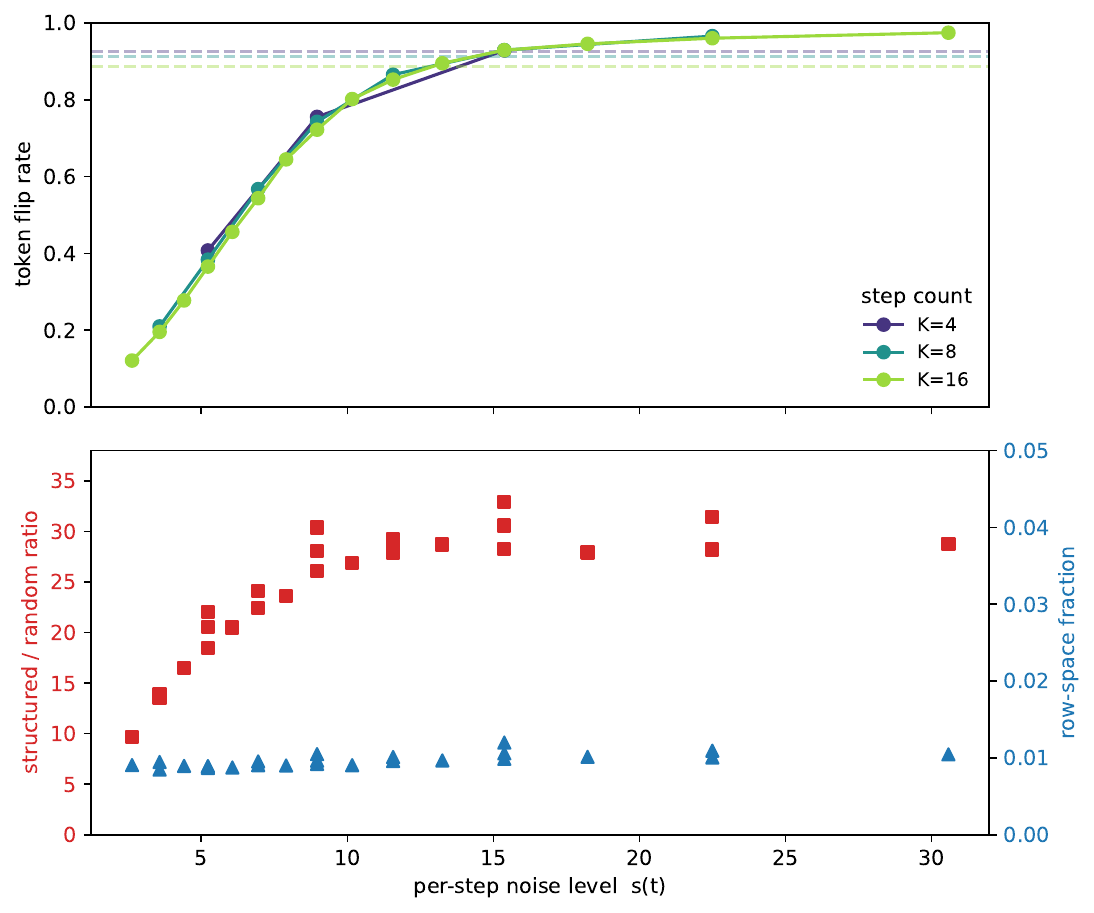}
\caption{\textbf{Oracle roll-in.}
ELF sampler, $n = 512$, $K \in \{4, 8, 16\}$.
\textit{Top:} flip rate versus $s(t)$; all $K$ collapse onto one curve; terminal flip $12$--$41\%$, cumulative ${\approx}90\%$.
\textit{Bottom:} structured-vs-random ratio ($10$--$33\times$) and row-space fraction (${\approx}1\%$).}
\label{fig:rollin}
\end{figure}

\paragraph{Quantitative tube-law test.}
The tube law predicts terminal flip from margin geometry, $\Prob(\hat Y_t \ne Y) \approx \E\,\Phi(-\delta^*/s(t))$ ($\hat f_{\delta^*}(0^+) \approx 0.064$ on ELF).
Its \emph{scaling} (flip linear in $s(t)$, all $K$ on one curve; Figure~\ref{fig:rollin}) holds on the learned ELF ODE; the \emph{absolute constant} exceeds the isotropic prediction by a stable factor (observed/predicted $2.5, 2.3, 2.2$ for $K = 4, 8, 16$; Figure~\ref{fig:p1}).
The idealized tube law assumes an isotropic residual; the learned residual is structured, so $\sqrt{2/\pi}$ is replaced by the realized active-facet normal moment, an anisotropy factor $c_{\mathrm{aniso}}$ (isotropic theorem $=1$; Appendix~\ref{app:noncommit}).
It is not a fitted intercept but the ratio $c_{\mathrm{aniso}} = \mathrm{obs}/(\sqrt{2/\pi}\,s\,A_W) \approx 2.3$ of the observed flip to a prediction whose boundary density $A_W$ is measured \emph{independently} from the clean margin distribution; it is \emph{distinct} from the CE amplification $\DABI{=}45.7\times$ (flip saturates at the first crossing, cross-entropy does not, so $\DABI \gg c_{\mathrm{aniso}}$).
The roll-in confirms the $s(t)$-scaling and this constant.

\subsection{Deterministic ODE versus stochastic SDE}\label{sec:gamma}

On the published ELF-B teacher and its progressive-distillation (PD) student~\citep{he2026elfpd} (same decoder and eval; only sampler randomness varies), we sweep $K \in \{1, \dots, 32\}$ and $\gamma \in \{0, \dots, 2\}$ ($60$ cells, Figure~\ref{fig:gamma}).
The teacher's deterministic ODE fails where the SDE succeeds: at $K = 4$, ODE PPL $294$ vs SDE $50$, the gap persisting at $K = 32$ ($68$ vs $20.5$); its $K = 1$ ODE reaches PPL $3.07$ only by mode collapse (entropy $2.01$).
Distillation does not remove the penalty: the student's ODE-to-SDE gap persists at every $K \ge 2$ ($15$--$38\%$; e.g.\ $K = 4$, $47.3 \to 34.8$), and its $K = 32$ SDE PPL $21.4$ reproduces the authors' reported $21.32$.
The SDE injects fresh noise per step, leaving the deterministic class of Theorem~\ref{thm:noncommit} (the lower bounds do not apply); but it is not free, since teacher SDE at $K = 4$ still produces multilingual fragments, so stochasticity needs enough steps.

\section{Escape Mechanisms}\label{sec:escape}

\subsection{Categorical commitment}\label{sec:cat}

Autoregressive and masked diffusion language models succeed at text generation with readouts far sharper than ELF's.
Their $\DABI$ values are higher, not lower: Llama-2~\citep{touvron2023llama2} (7B) has $\DABI = 19{,}805\times$; Qwen-2.5~\citep{qwen2025} (7B) has $\DABI = 915\times$ (base) and $1{,}350\times$ (instruct).
The difference is categorical commitment: at each step, a hard discrete token is selected and all subsequent computation conditions on that selection.

\begin{lemma}[categorical escape]\label{lem:escape}
Let $F : \R^N \to \mathcal{C}$ be the exact categorical generator $F(z) = c_{V_B(z)}$ outputting at least two distinct atoms from $\mathcal{C}$.
Then $F$ is discontinuous on the connected domain $\R^N$ and hence not Lipschitz.
The continuous lower bounds (Theorems~\ref{thm:gamma}--\ref{thm:dim}) do not apply.
\end{lemma}

\begin{proof}
A continuous image of a connected set is connected; a finite set with $\ge 2$ points is not.
\end{proof}

\paragraph{Commit ablation.}
Removing hard selection at matched readout sharpness collapses generation across three published systems spanning both the autoregressive (Llama-2-7B) and masked-diffusion (LLaDA-8B, Dream-7B) families: AR soft-carry diversity drops $d_2\!: 0.541 \to 0.135$ and LLaDA soft-refresh drops $d_2\!: 0.82 \to 0.47$ (Table~\ref{tab:commit}).
The cleanest control is a matched pair on Dream that holds the sampling distribution fixed and varies only commitment: hard-sample (commit \emph{and} sample) versus soft-refresh-sample (the same sampling, with only the in-loop commit removed).
Removing the commit step alone raises PPL from $7.3$ to $38.6$ ($5\times$) and collapses effective length from $99$ to $22$.
Because nothing but the commit step changes, this isolates commitment more tightly than the norm-matched controls, which also collapse (Appendix~\ref{app:extended}).
Categorical commitment, not readout smoothness or sampling temperature, is the operative escape.

\subsection{Stochastic re-injection}\label{sec:stochastic}

The SDE sampler injects fresh noise per step, leaving the deterministic trajectory class.
Theorem~\ref{thm:noncommit} does not apply.
The $\gamma$-sweep (Section~\ref{sec:gamma}) confirms: SDE improves over ODE at every tested $K \ge 2$, though the escape is not free (teacher SDE $K{=}4$ still produces multilingual fragments at PPL${}=50$).

\subsection{DABI$\times$CCI taxonomy}\label{sec:taxonomy}

The two diagnostics, $\DABI$ and $\CCI$, define a $2$D taxonomy that classifies all tested systems (Figure~\ref{fig:taxonomy}).
The two axes are orthogonal: $\DABI$ measures readout sharpness (a property of the decoder), while $\CCI$ measures categorical commitment (a property of the generator).

\emph{Low $\DABI$, $\CCI = 0$ (image VAEs).}
Deterministic few-step generation succeeds because the smooth decoder tolerates the posterior-mean blur.

\emph{High $\DABI$, $\CCI = 0$ (ELF text flow).}
Deterministic few-step generation fails: the sharp readout converts the small structured residual into token errors.
This is the failure regime.

\emph{High $\DABI$, high $\CCI$ (AR, masked dLM).}
Generation succeeds despite the sharp readout because categorical commitment makes the generator discontinuous, evading the continuous lower bound.

The escape arrows in Figure~\ref{fig:taxonomy}, from high-$\DABI$/CCI${}=0$ (fails) to high-$\DABI$/high-$\CCI$ (works), are grounded in the commit ablations of Section~\ref{sec:cat}: fixing $\DABI$ and varying $\CCI$ from $1$ to $0$ collapses generation.

\subsection{Recent continuous few-step text models}\label{sec:recent}

The taxonomy predicts that any system generating fluent open-ended text in few ($K \le 16$) \emph{deterministic} steps with a smooth continuous latent must leave the continuous-deterministic class.
Every 2025--2026 continuous-latent text generator we surveyed (public checkpoint or documented sampler; we found no counterexample) is consistent (Table~\ref{tab:recent}, Appendix~\ref{app:recent}): every deterministic-ODE system needs ${\ge}32$ steps~\citep{langflow2026,cosmos2025}, and every few-step system uses SDE sampling~\citep{he2026elf,cola2026} or categorical commitment~\citep{fastdiss2026,diladiff2026,codar2026}.
The sharpest apparent counterexample, FMLM~\citep{fmlm2026} (deterministic one-step), is categorical escape: a \emph{one-hot} simplex target with $\argmax$ decoding (Lemma~\ref{lem:escape}), not smooth transport.
Training-time robustness mechanisms do not move the classification: FastDiSS's SCP/MANS perturbations~\citep{fastdiss2026} harden the network but keep a categorical readout, and Loopholing's deterministic latent pathway~\citep{loopholing2026} \emph{complements} rather than replaces a stochastic one-hot step (Appendix~\ref{app:recent}).

\section{Separated-Mode Transport Theory}\label{sec:separated}

In the idealized regime of well-separated modes, we prove sharp Lipschitz transport laws (full statements and proofs in Appendices~\ref{app:thm1}--\ref{app:thm3}; synthetic verification in Appendix~\ref{app:synth}).
Here $\mathcal{P}_{n,M} = \inf \sum_{i<j}\norm{c_i - c_j}\,P_{\gauss_n}(B_i, B_j)$ is the minimal unit-barrier interface profile of an $M$-atom equal-mass codebook in $\R^n$ (the prescribed-mass profile $\Jint$ of \eqref{eq:Jdef} with $\kappa_p(i,j) \equiv \norm{c_i - c_j}$), and quantifies the deterministic stiffness $\Lambda$ a smooth map must pay to separate $M$ modes.
The three theorems are one statement at three resolutions: stiffness $\Lambda \gtrsim \Jint/W_p^p$, growing as $\sqrt{\log M}$ once the dimension is $\Omega(\log M)$, with a depth-$B$ factorization giving a $\sqrt{B}$-smaller per-step peak. All three are instantiated on explicit codebooks in Appendix~\ref{app:synth}.
Intuitively, $\Jint$ is the minimal barrier-weighted wall area a smooth map must build to separate $M$ modes; a \emph{finite-perimeter} (Caccioppoli) partition is just basins with rectifiable, finite-area walls.

\begin{theorem}[interface $\Gamma$-limit; informal]\label{thm:gamma}
$\lim_{\Lambda\to\infty} \Lambda \inf_{\Lip(T)\le\Lambda} W_p^p(T\push\sigma, \mu) = \Jint_{p,\sigma}(\pi;\mathcal{C})$, a weighted Caccioppoli interface energy~\citep{baldo1990,fonseca1989gradient}.
\end{theorem}
\vspace{-0.5em}
\begin{theorem}[dimension phase diagram; informal]\label{thm:dim}
$\mathcal{P}_{n,M} \asymp \sqrt{\log M}$ for $n \ge C\log M$; $\Theta(M^{1/n})$ for fixed $n$.
Combined: $\Lambda \gtrsim \Delta^{p+1}\sqrt{\log M}/\eta^p$.
\end{theorem}
\vspace{-0.5em}
\begin{theorem}[hierarchical separation; informal]\label{thm:hier}
A depth-$B$ categorical hierarchy reduces one $M$-way problem to $B$ local $m$-way ones (peak ratio $\sqrt{B}$); the factorized generator is discontinuous (Lemma~\ref{lem:escape}), so continuous bounds do not apply.
\end{theorem}

These results do \emph{not} predict real ELF residuals, but the overlapping and separated regimes are coarea-linked contractions of one composite interface (Theorem~\ref{thm:bridge}, Appendix~\ref{app:bridge}; schematic Figure~\ref{fig:interface}).

\section{Discussion and Limitations}\label{sec:discussion}

We have identified a structural reason deterministic few-step generation of continuous text latents produces incoherent output: the posterior-mean terminal step cannot resolve discrete branch choices before a sharp categorical readout (Theorem~\ref{thm:noncommit}), and the $\DABI{\times}\CCI$ taxonomy predicts which systems fail and which escape.

\paragraph{Limitations.}
Our formal results analyze deterministic posterior-mean transport under smooth densities and exact clean decoding, which learned generators only approximate; the bounds constrain deterministic samplers, not the stochastic and categorical generators that escape them.
The first-order law is robust to these idealizations (Appendix~\ref{app:noncommit}): label noise, finite-perimeter, or codimension-${\ge}2$ irregularities change the constant but not the linear-in-$s$ scaling, and a learned denoiser is covered whenever its terminal active-normal residual is stable and $O(s)$ (Proposition~\ref{prop:approx-denoiser}).
The open theory direction is a within-core $W_p$ theory for overlapping latents; practically, deterministic few-step text needs SDE or categorical generation (Appendix~\ref{app:causal}).

\bibliography{references}
\bibliographystyle{iclr2026_conference}

\newpage
\appendix

\section{Notation}\label{app:notation}
We use a single symbol for the interface energy throughout: $\Jint$ (the prescribed-mass weighted Caccioppoli interface energy of Theorem~\ref{thm:gamma}). $\mathcal{P}_{n,M}$ is \emph{not} a separate functional but the unit-barrier \emph{profile} of $\Jint$, i.e.\ $\Jint$ with $\kappa_p(i,j)\equiv\|c_i-c_j\|$ for an equal-mass $M$-atom codebook in $\R^n$. The symbol $\mathcal{I}_{\mathrm{crit}}$ denotes the critical-SNR time \emph{interval}, not an energy. Table~\ref{tab:notation} collects the recurring symbols.
\begin{table}[h]
\centering
\caption{Recurring notation.}
\label{tab:notation}
\small
\begin{tabular}{ll}
\toprule
Symbol & Meaning \\
\midrule
$D_W(z)=\argmax_y w_y^\top z$ & categorical readout (decoder); $w_y$ the row for token $y$ \\
$\delta^*(X,Y)$ & normalized margin to the nearest decision boundary (Eq.~\ref{eq:margin}) \\
$s(t)=s_0(1-t)/t$ & effective terminal noise scale at roll-in time $t$ \\
$F_{ij}$ & active co-maximal facet between tokens $i,j$ \\
$A_W=\sum_{i<j}\int_{F_{ij}}\rho_X\,d\mathcal H^{d-1}$ & active boundary mass (overlapping regime) \\
$c_{\mathrm{aniso}}$ & anisotropy factor of the realized residual ($=1$ if isotropic) \\
$\DABI=\Delta\mathrm{CE}_{\mathrm{struct}}/\Delta\mathrm{CE}_{\mathrm{rand}}$ & decoder amplification of boundary-aligned inputs \\
$\CCI$ & categorical commitment index (in-loop hard-selection fraction) \\
$\eta_{ij}=\|DT^\top n_{ij}\|$ & readout-normal stretch / transversality; $L_\perp=\esssup\eta_{ij}$ \\
$\Jint_{p,\sigma}(\pi;C)$ & separated-regime interface energy (Theorem~\ref{thm:gamma}) \\
$\mathcal{P}_{n,M}$ & unit-barrier profile of $\Jint$ (special case, see above) \\
$\kappa_p(i,j)$ & barrier metric between codebook atoms; $\kappa_{\max}=\max_{i<j}\kappa_p$ \\
$E_{\mathrm{off}}$ & off-facet (mis-aligned) interface energy in the bridge bound \\
$\Lambda$ & deterministic transport stiffness (Lipschitz budget) \\
\bottomrule
\end{tabular}
\end{table}

\section{Proof of Theorem~\ref{thm:gamma}: the hard-gradient multiwell interface limit}
\label{app:thm1}

We write this appendix in a self-contained form. Let
\[
        \sigma=\rho\,dx
        \qquad\text{on }\mathbb R^n,
        \qquad
        \rho\in C^1(\mathbb R^n;(0,\infty)),
        \qquad
        \int_{\mathbb R^n}\rho\,dx=1.
\]
Let \(C=\{c_1,\ldots,c_M\}\subset\mathbb R^d\) be a finite set of distinct Euclidean atoms,
let \(p\in[1,\infty)\), and let
\[
        V(y):=\operatorname{dist}(y,C)^p.
\]
For an absolutely continuous curve \(\xi:[0,1]\to\mathbb R^d\), define the \(V\)-length
\[
        \ell_V(\xi):=\int_0^1 V(\xi(s))\,|\xi'(s)|\,ds,
\]
and define the induced path pseudometric
\[
        d_V(y,z)
        :=
        \inf\{\ell_V(\xi):\xi(0)=y,\ \xi(1)=z,\ \xi\in AC([0,1];\mathbb R^d)\}.
\]
The well-to-well barrier is
\[
        \kappa_{ij}=\kappa_p(i,j):=d_V(c_i,c_j).
\]
For a Caccioppoli partition \(B=(B_1,\ldots,B_M)\) of \(\mathbb R^n\), define
\[
        P_\sigma(B_i,B_j)
        :=
        \int_{\partial^*B_i\cap\partial^*B_j}\rho\,d\mathcal H^{n-1},
\]
and
\[
        E_{p,\sigma}(B;C)
        :=
        \sum_{i<j}\kappa_{ij}P_\sigma(B_i,B_j).
\]
For prescribed masses \(\pi_i>0\), \(\sum_i\pi_i=1\), define
\[
        J_{p,\sigma}(\pi;C)
        :=
        \inf_{\sigma(B_i)=\pi_i} E_{p,\sigma}(B;C),
\]
where the infimum is over Caccioppoli partitions with those masses. Finally,
\[
        \mathcal F_\varepsilon(T)
        :=
        \begin{cases}
        \displaystyle
        \varepsilon^{-1}\int_{\mathbb R^n}V(T(x))\,d\sigma(x),
        & \operatorname{Lip}(T)\le \varepsilon^{-1},\\[1.2ex]
        +\infty,&\text{otherwise.}
        \end{cases}
\]
We use the \(L^1(\sigma;\mathbb R^d)\) topology.

Let \(Q_i\) denote the Euclidean Voronoi cell of \(c_i\), with a fixed lexicographic tie-breaking rule, and let
\[
        q(y)=i
        \quad\Longleftrightarrow\quad
        y\in Q_i.
\]
For a partition \(B\), put
\[
        u_B:=\sum_{i=1}^M c_i\mathbf 1_{B_i}.
\]

We use standard facts on \(BV\), reduced boundaries, traces of \(BV\) functions on Caccioppoli interfaces, and the structure theorem for finite partitions; see Ambrosio--Fusco--Pallara \citep{AmbrosioFuscoPallara2000}. For the density of regular polyhedral partitions with volume constraints we use the Braides--Conti--Garroni polyhedral density theorem \citep{BraidesContiGarroni2017}. The soft-gradient analogue is the Baldo/Fonseca--Tartar multiwell Modica--Mortola theory \citep{baldo1990,fonseca1989gradient}.

\begin{theorem}[Expanded form of Theorem~\ref{thm:gamma}]
\label{thm:gamma-expanded}
Assume \(\rho\in C^1(\mathbb R^n;(0,\infty))\), \(C=\{c_1,\ldots,c_M\}\subset\mathbb R^d\) is finite and consists of distinct points, \(p\in[1,\infty)\), and \(\pi_i>0\) with \(\sum_i\pi_i=1\). Then
\[
        \mathcal F_\varepsilon
        \xrightarrow{\Gamma}
        \mathcal F_0
        \qquad\text{in }L^1(\sigma;\mathbb R^d),
\]
where
\[
        \mathcal F_0(u)
        =
        \begin{cases}
        E_{p,\sigma}(B;C),
        & u=u_B=\sum_i c_i\mathbf 1_{B_i}
          \text{ for a Caccioppoli partition }B,\\
        +\infty,&\text{otherwise.}
        \end{cases}
\]
Moreover, if \(B\) has prescribed masses \(\sigma(B_i)=\pi_i\), then the recovery sequence may be chosen with exact nearest-well phase masses
\[
        \sigma\big(T_\varepsilon^{-1}(Q_i)\big)=\pi_i
        \qquad\text{for every }i.
\]
Consequently,
\[
        \lim_{\varepsilon\downarrow0}
        \inf_{\substack{\operatorname{Lip}(T)\le \varepsilon^{-1}\\
                        \sigma(T^{-1}(Q_i))=\pi_i}}
        \varepsilon^{-1}\int_{\mathbb R^n}V(T)\,d\sigma
        =
        J_{p,\sigma}(\pi;C).
        \tag{A.1}
\]
Equivalently, for \(\mu=\sum_i\pi_i\delta_{c_i}\),
\[
        \boxed{
        \lim_{\Lambda\to\infty}
        \Lambda\inf_{\operatorname{Lip}(T)\le\Lambda}
        W_p^p(T_\#\sigma,\mu)
        =
        J_{p,\sigma}(\pi;C).
        }
        \tag{A.2}
\]
\end{theorem}

\subsection{Barrier calibration}

\begin{lemma}[Barrier metric and scalar calibration]
\label{lem:barrier-calibration}
The function \(d_V\) is a pseudometric on \(\mathbb R^d\). For \(i\ne j\),
\[
        \kappa_{ij}>0.
\]
For
\[
        \phi_i(y):=d_V(c_i,y)
\]
one has the pointwise local Lipschitz estimate
\[
        \operatorname{lip}\phi_i(y)\le V(y).
        \tag{A.3}
\]
Let
\[
        K_\kappa:=\max_{a,b}\kappa_{ab},
        \qquad
        \overline\phi_i(y):=\min\{\phi_i(y),K_\kappa\}.
\]
If \(T\in\operatorname{Lip}_{\mathrm{loc}}(\mathbb R^n;\mathbb R^d)\), then
\[
        g_i:=\overline\phi_i\circ T\in W^{1,\infty}_{\mathrm{loc}}(\mathbb R^n),
\]
and a.e.
\[
        |\nabla g_i|
        \le
        V(T)\,\|DT\|_{\mathrm{op}}.
        \tag{A.4}
\]
In particular, if \(\operatorname{Lip}(T)\le\varepsilon^{-1}\), then
\[
        |\nabla g_i|
        \le
        \varepsilon^{-1}V(T)
        \qquad\text{a.e.}
        \tag{A.5}
\]
\end{lemma}

\begin{proof}
Nonnegativity, symmetry, and \(d_V(y,y)=0\) are immediate from the definition and reversal of curves. The triangle inequality follows by concatenating curves. Thus \(d_V\) is a pseudometric.

Let
\[
        \Delta_C:=\min_{a\ne b}|c_a-c_b|>0.
\]
Fix \(i\ne j\) and let \(\xi\) be any absolutely continuous curve from \(c_i\) to \(c_j\). Let
\[
        \tau:=\inf\{s:|\xi(s)-c_i|=\Delta_C/2\}.
\]
By continuity, \(\tau\) exists. On the ball \(B(c_i,\Delta_C/2)\), the closest codebook point is \(c_i\), hence
\[
        V(\xi(s))=|\xi(s)-c_i|^p
        \qquad\text{for }0\le s\le\tau.
\]
Set \(r(s):=|\xi(s)-c_i|\). Since \(|r'(s)|\le|\xi'(s)|\) a.e.,
\[
\begin{aligned}
        \ell_V(\xi)
        &\ge
        \int_0^\tau r(s)^p|\xi'(s)|\,ds                                      \\
        &\ge
        \int_0^\tau r(s)^p|r'(s)|\,ds                                           \\
        &\ge
        \int_0^{\Delta_C/2} r^p\,dr
        =
        \frac{(\Delta_C/2)^{p+1}}{p+1}.
\end{aligned}
\]
Taking the infimum over \(\xi\) gives \(\kappa_{ij}>0\).

Next, by the triangle inequality for \(d_V\),
\[
        |\phi_i(y)-\phi_i(z)|
        \le d_V(y,z).
\]
For \(z\) near \(y\), the straight segment \(\xi(\theta)=y+\theta(z-y)\) gives
\[
        d_V(y,z)
        \le
        |z-y|\int_0^1 V(y+\theta(z-y))\,d\theta.
\]
Dividing by \(|z-y|\) and sending \(z\to y\), using continuity of \(V\), yields
\[
        \operatorname{lip}\phi_i(y)\le V(y).
\]
The truncation \(r\mapsto\min\{r,K_\kappa\}\) is \(1\)-Lipschitz, so the same pointwise bound holds for \(\overline\phi_i\).

Since \(T\) is locally Lipschitz and \(\overline\phi_i\) is locally Lipschitz, \(g_i=\overline\phi_i\circ T\) lies in \(W^{1,\infty}_{\mathrm{loc}}\). By Rademacher's theorem and the metric chain rule,
\[
        |\nabla g_i(x)|
        \le
        \operatorname{lip}\overline\phi_i(T(x))\,\|DT(x)\|_{\mathrm{op}}
        \le
        V(T(x))\,\|DT(x)\|_{\mathrm{op}}
\]
for a.e. \(x\). This proves the lemma.
\end{proof}

\subsection{Compactness}

\begin{proposition}[Compactness of bounded-energy sequences]
\label{prop:gamma-compactness}
Let \(\varepsilon_k\downarrow0\) and suppose
\[
        \sup_k\mathcal F_{\varepsilon_k}(T_k)<\infty.
\]
Then, after passing to a subsequence, there exists a Caccioppoli partition
\[
        B=(B_1,\ldots,B_M)
\]
such that
\[
        T_k\to u_B:=\sum_i c_i\mathbf 1_{B_i}
        \qquad\text{in }L^1(\sigma;\mathbb R^d).
\]
\end{proposition}

\begin{proof}
The energy bound gives
\[
        \int_{\mathbb R^n}\operatorname{dist}(T_k,C)^p\,d\sigma
        \le
        C\varepsilon_k.
        \tag{A.6}
\]
Hence, for every \(r>0\),
\[
        \sigma\{\operatorname{dist}(T_k,C)\ge r\}
        \le
        C r^{-p}\varepsilon_k.
        \tag{A.7}
\]
In particular,
\[
        \operatorname{dist}(T_k,C)\to0
        \qquad\text{in }L^p(\sigma)
        \quad\text{and in }L^1(\sigma).
        \tag{A.8}
\]

For \(i=1,\ldots,M\), define
\[
        g_{i,k}:=\overline\phi_i\circ T_k.
\]
Let \(K\subset\mathbb R^n\) be compact. Since \(\rho>0\) and continuous,
\[
        \rho_K:=\min_K\rho>0.
\]
By Lemma~\ref{lem:barrier-calibration},
\[
\begin{aligned}
        \int_K |\nabla g_{i,k}|\,dx
        &\le
        \rho_K^{-1}
        \int_K |\nabla g_{i,k}|\,\rho\,dx                                      \\
        &\le
        \rho_K^{-1}
        \int_K \varepsilon_k^{-1}V(T_k)\,\rho\,dx                                \\
        &\le
        C_K.
\end{aligned}
        \tag{A.9}
\]
Also \(0\le g_{i,k}\le K_\kappa\). Thus \((g_{i,k})_k\) is bounded in \(BV(K)\). By BV compactness and a diagonal extraction over \(K=B_R(0)\), there are \(g_i\in BV_{\mathrm{loc}}(\mathbb R^n)\) such that
\[
        g_{i,k}\to g_i
        \qquad\text{in }L^1_{\mathrm{loc}}(\mathbb R^n).
\]
Write
\[
        \mathbf g_k:=(g_{1,k},\ldots,g_{M,k}),
        \qquad
        \mathbf g:=(g_1,\ldots,g_M).
\]

For each well \(c_j\), define the label vector
\[
        a_j:=(\kappa_{1j},\ldots,\kappa_{Mj})\in\mathbb R^M.
\]
These label vectors are separated. Indeed, if \(i\ne j\), then the \(i\)-th coordinate of \(a_i\) is \(0\), whereas the \(i\)-th coordinate of \(a_j\) is \(\kappa_{ij}>0\). Hence
\[
        \min_{i\ne j}|a_i-a_j|>0.
        \tag{A.10}
\]

Let \(q_k(x):=q(T_k(x))\). Since \(\overline\phi_i(c_j)=\kappa_{ij}\), continuity of the finitely many functions \(\overline\phi_i\) near the finite set \(C\) implies that, for
\[
        \omega(r):=
        \sup_{\operatorname{dist}(y,C)\le r}
        \big|\,
        (\overline\phi_1(y),\ldots,\overline\phi_M(y))
        -
        a_{q(y)}
        \,\big|,
\]
one has \(\omega(r)\downarrow0\) as \(r\downarrow0\). Therefore
\[
\begin{aligned}
        \sigma\{|\mathbf g_k-a_{q_k}|>\alpha\}
        &\le
        \sigma\{\operatorname{dist}(T_k,C)>r\}
\end{aligned}
\]
whenever \(\omega(r)<\alpha\). By (A.7),
\[
        \mathbf g_k-a_{q_k}\to0
        \qquad\text{in }\sigma\text{-measure}.
        \tag{A.11}
\]
Together with \(\mathbf g_k\to\mathbf g\) locally in measure, this implies
\[
        \mathbf g(x)\in\{a_1,\ldots,a_M\}
        \qquad\text{for a.e. }x.
\]
Define the limiting phase map \(q_\infty\) by
\[
        \mathbf g(x)=a_{q_\infty(x)}
\]
and set
\[
        B_i:=\{x:q_\infty(x)=i\}.
\]
Then \(B=(B_1,\ldots,B_M)\) is a measurable partition.

It remains to show that it is a Caccioppoli partition. For fixed \(i\),
\[
        g_i
        =
        \sum_{j=1}^M\kappa_{ij}\mathbf 1_{B_j}.
        \tag{A.12}
\]
Let
\[
        \kappa_*:=\min_{i\ne j}\kappa_{ij}>0.
\]
Choose a Lipschitz function \(h_i:\mathbb R\to[0,1]\) such that
\[
        h_i(0)=1,
        \qquad
        h_i(r)=0\quad\text{for }r\ge \kappa_*/2.
\]
Since \(g_i=0\) on \(B_i\) and \(g_i\ge\kappa_*\) on \(\mathbb R^n\setminus B_i\),
\[
        \mathbf 1_{B_i}=h_i(g_i).
\]
The BV chain rule for Lipschitz scalar functions gives
\[
        \mathbf 1_{B_i}\in BV_{\mathrm{loc}}(\mathbb R^n).
\]
Thus \(B\) is a Caccioppoli partition.

Finally,
\[
\begin{aligned}
        \int_{\mathbb R^n}|T_k-u_B|\,d\sigma
        &\le
        \int_{\mathbb R^n}|T_k-c_{q_k}|\,d\sigma
        +
        \int_{\mathbb R^n}|c_{q_k}-c_{q_\infty}|\,d\sigma.
\end{aligned}
\]
The first term tends to zero by (A.8). The second tends to zero because \(q_k\to q_\infty\) locally in measure, the codebook is finite, and the \(\sigma\)-tail can be made arbitrarily small. Hence \(T_k\to u_B\) in \(L^1(\sigma)\).
\end{proof}

\subsection{Liminf inequality}

\begin{proposition}[Liminf inequality]
\label{prop:gamma-liminf}
Let \(\varepsilon_k\downarrow0\), \(T_k\to u_B\) in \(L^1(\sigma;\mathbb R^d)\), and
\[
        \sup_k\mathcal F_{\varepsilon_k}(T_k)<\infty.
\]
Then
\[
        \liminf_{k\to\infty}\mathcal F_{\varepsilon_k}(T_k)
        \ge
        E_{p,\sigma}(B;C).
\]
\end{proposition}

\begin{proof}
Passing to a subsequence, assume the liminf is a limit. Define Radon measures
\[
        \mu_k
        :=
        \varepsilon_k^{-1}V(T_k)\rho\,dx.
\]
The total masses \(\mu_k(\mathbb R^n)\) are bounded. Passing to a further subsequence,
\[
        \mu_k \stackrel{*}{\rightharpoonup}\mu
\]
locally as Radon measures.

Let
\[
        g_{i,k}:=\overline\phi_i\circ T_k.
\]
As in the compactness proof,
\[
        g_{i,k}\to g_i:=\sum_{j=1}^M\kappa_{ij}\mathbf 1_{B_j}
        \qquad\text{in }L^1_{\mathrm{loc}}.
        \tag{A.13}
\]
By Lemma~\ref{lem:barrier-calibration},
\[
        \rho|\nabla g_{i,k}|\,dx
        \le
        \mu_k
        \qquad\text{as measures.}
        \tag{A.14}
\]
By lower semicontinuity of total variation under \(L^1_{\mathrm{loc}}\) convergence, with continuous positive weight \(\rho\),
\[
        \rho\,|Dg_i|
        \le
        \mu
        \qquad\text{as Radon measures.}
        \tag{A.15}
\]
Indeed, for every nonnegative \(\psi\in C_c(\mathbb R^n)\),
\[
        \int \psi\,\rho\,d|Dg_i|
        \le
        \liminf_k\int \psi\,\rho\,|\nabla g_{i,k}|\,dx
        \le
        \liminf_k\int\psi\,d\mu_k
        =
        \int\psi\,d\mu.
\]

Let
\[
        \Sigma_{ij}^B:=\partial^*B_i\cap\partial^*B_j.
\]
On \(\Sigma_{ij}^B\), the approximate traces of \(g_i\) are
\[
        0
        \quad\text{from the }B_i\text{ side},
        \qquad
        \kappa_{ij}
        \quad\text{from the }B_j\text{ side}.
\]
Therefore the jump part of \(Dg_i\) gives
\[
        |Dg_i|\llcorner\Sigma_{ij}^B
        \ge
        \kappa_{ij}\,\mathcal H^{n-1}\llcorner\Sigma_{ij}^B.
        \tag{A.16}
\]
Combining (A.15) and (A.16),
\[
        \mu\llcorner\Sigma_{ij}^B
        \ge
        \kappa_{ij}\rho\,\mathcal H^{n-1}\llcorner\Sigma_{ij}^B.
        \tag{A.17}
\]
For finite Caccioppoli partitions, pairwise reduced interfaces are \(\mathcal H^{n-1}\)-a.e. disjoint: the triple-junction set has zero \((n-1)\)-dimensional measure. Hence
\[
\begin{aligned}
        \liminf_k\mathcal F_{\varepsilon_k}(T_k)
        &=
        \liminf_k\mu_k(\mathbb R^n)                                      \\
        &\ge
        \mu(\mathbb R^n)                                                   \\
        &\ge
        \sum_{i<j}
        \kappa_{ij}
        \int_{\Sigma_{ij}^B}\rho\,d\mathcal H^{n-1}                         \\
        &=
        E_{p,\sigma}(B;C).
\end{aligned}
\]
\end{proof}

\subsection{One-dimensional profiles}

\begin{lemma}[Near-optimal one-dimensional profiles]
\label{lem:one-dimensional-profile}
Let
\[
        K:=\operatorname{conv}C.
\]
For every \(i\ne j\) and every \(\eta>0\), there exists a Lipschitz curve
\[
        \zeta_{ij}:\mathbb R\to K
\]
such that
\[
        \operatorname{Lip}(\zeta_{ij})\le1,
        \qquad
        \zeta_{ij}(s)=c_i\text{ for }s\le -L_{ij},
        \qquad
        \zeta_{ij}(s)=c_j\text{ for }s\ge L_{ij},
\]
for some \(L_{ij}<\infty\), and
\[
        \int_{\mathbb R}V(\zeta_{ij}(s))\,ds
        \le
        \kappa_{ij}+\eta.
        \tag{A.18}
\]
\end{lemma}

\begin{proof}
Choose an absolutely continuous curve \(\xi:[0,1]\to\mathbb R^d\) from \(c_i\) to \(c_j\) with
\[
        \ell_V(\xi)\le\kappa_{ij}+\eta.
\]
Let \(\Pi_K:\mathbb R^d\to K\) be the Euclidean metric projection. Since \(K\) is closed and convex, \(\Pi_K\) is \(1\)-Lipschitz. Since every \(c_a\in K\), projection does not increase distance to any \(c_a\):
\[
        |\Pi_K y-c_a|\le |y-c_a|
        \qquad\text{for every }a.
\]
Thus
\[
        V(\Pi_K y)\le V(y).
\]
Replacing \(\xi\) by \(\Pi_K\circ\xi\) does not increase either Euclidean speed or \(V\)-length, and the projected curve remains inside \(K\).

Reparameterize the projected curve by arclength. This gives a \(1\)-Lipschitz curve \(\zeta_{ij}\) on a finite interval with endpoints \(c_i,c_j\) and
\[
        \int V(\zeta_{ij}(s))\,ds
        =
        \ell_V(\zeta_{ij})
        \le
        \kappa_{ij}+\eta.
\]
Extend it constantly by \(c_i\) to the left and \(c_j\) to the right. Since \(V(c_i)=V(c_j)=0\), the integral is unchanged.
\end{proof}

\subsection{Limsup for regular polyhedral partitions}

We first prove recovery for a regular polyhedral partition. A regular polyhedral partition means that each interface \(\partial B_i\cap\partial B_j\) is, up to \(\mathcal H^{n-1}\)-null sets, a finite union of relatively open flat \((n-1)\)-polytopes, and the non-manifold skeleton where three or more faces meet is contained in a finite union of \((n-2)\)-polytopes. The case \(n=1\) is understood with empty codimension-two skeleton.

\begin{lemma}[Hard-gradient recovery for regular partitions]
\label{lem:hard-gradient-recovery}
Let \(B=(B_1,\ldots,B_M)\) be a regular polyhedral Caccioppoli partition with finite weighted energy. Then there exist maps
\[
        G_\varepsilon:\mathbb R^n\to K:=\operatorname{conv}C
\]
such that
\[
        \operatorname{Lip}(G_\varepsilon)
        \le
        \frac{1-\sqrt\varepsilon}{\varepsilon},
        \tag{A.19}
\]
\[
        G_\varepsilon\to u_B
        \qquad\text{in }L^1(\sigma;\mathbb R^d),
        \tag{A.20}
\]
and
\[
        \limsup_{\varepsilon\downarrow0}
        \mathcal F_\varepsilon(G_\varepsilon)
        \le
        E_{p,\sigma}(B;C).
        \tag{A.21}
\]
Moreover,
\[
        \sigma\big(G_\varepsilon^{-1}(Q_i)\triangle B_i\big)=O(\varepsilon)
        \qquad\text{for each }i.
        \tag{A.22}
\]
\end{lemma}

\begin{proof}
Fix \(\eta>0\). For each face \(F\subset\partial B_i\cap\partial B_j\), choose a profile \(\zeta_{ij}\) from Lemma~\ref{lem:one-dimensional-profile} with
\[
        \int_{\mathbb R}V(\zeta_{ij}(s))\,ds
        \le
        \kappa_{ij}+\eta.
        \tag{A.23}
\]
Orient the unit normal \(\nu_F\) so that it points from \(B_i\) into \(B_j\), and let \(r_F\) be the signed distance to the affine hyperplane containing \(F\), positive on the \(B_j\) side. Set
\[
        s_\varepsilon:=1-\sqrt\varepsilon.
\]

Let \(S\) be the codimension-two skeleton of the polyhedral partition. Choose \(R>0\) large enough, depending only on the profiles and on \(\operatorname{diam}K\), so that different face tubes outside
\[
        H_\varepsilon:=N_{R\varepsilon}(S)
\]
are separated at distance large enough to satisfy the Lipschitz estimates below. Since \(S\) is contained in finitely many \((n-2)\)-polytopes,
\[
        |H_\varepsilon|=O(\varepsilon^2)
\]
on every bounded region, and the same estimate holds with weight \(\rho\) because \(\rho\) is continuous.

On the closed set outside \(H_\varepsilon\), define a preliminary map \(\widetilde G_\varepsilon\) as follows. In the trimmed tube around a face \(F\subset\partial B_i\cap\partial B_j\), write uniquely
\[
        x=y+r\nu_F,
        \qquad
        y\in F,
        \qquad
        \operatorname{dist}(y,\partial F)>R\varepsilon.
\]
Set
\[
        \widetilde G_\varepsilon(x)
        :=
        \zeta_{ij}\!\left(\frac{s_\varepsilon}{\varepsilon}r_F(x)\right).
        \tag{A.24}
\]
Away from the union of these transition tubes, set
\[
        \widetilde G_\varepsilon(x)=c_i
        \qquad\text{on the }B_i\text{ side.}
\]
The profile is constant outside a bounded interval, so these definitions agree on the overlaps between the constant regions and the ends of a face tube.

There is no tangential derivative in (A.24): the map depends only on the signed normal coordinate \(r_F\). Since \(|\nabla r_F|=1\) and \(\operatorname{Lip}(\zeta_{ij})\le1\),
\[
        \|D\widetilde G_\varepsilon\|_{\mathrm{op}}
        \le
        \frac{s_\varepsilon}{\varepsilon}
        \qquad\text{inside each face tube.}
        \tag{A.25}
\]
On constant regions the derivative is zero. For two points lying in different face neighborhoods, the deletion of \(H_\varepsilon\) gives a separation of order \(R\varepsilon\), while all values lie in \(K\). Taking \(R\) large enough gives the global estimate
\[
        |\widetilde G_\varepsilon(x)-\widetilde G_\varepsilon(y)|
        \le
        \frac{s_\varepsilon}{\varepsilon}|x-y|
\]
on the closed set where \(\widetilde G_\varepsilon\) is defined.

By Kirszbraun's theorem, \(\widetilde G_\varepsilon\) extends to a map
\[
        \widehat G_\varepsilon:\mathbb R^n\to\mathbb R^d
\]
with the same Lipschitz constant \(s_\varepsilon/\varepsilon\). Finally define
\[
        G_\varepsilon:=\Pi_K\circ\widehat G_\varepsilon,
\]
where \(\Pi_K\) is the Euclidean projection onto \(K\). Since \(\Pi_K\) is \(1\)-Lipschitz and leaves \(\widetilde G_\varepsilon\) unchanged, \(G_\varepsilon\) still satisfies
\[
        \operatorname{Lip}(G_\varepsilon)
        \le
        \frac{s_\varepsilon}{\varepsilon}
        =
        \frac{1-\sqrt\varepsilon}{\varepsilon}.
\]
Moreover, because \(V(\Pi_K y)\le V(y)\), projection does not increase the energy.

The map equals \(u_B\) outside an \(O(\varepsilon)\)-neighborhood of the interfaces, except inside \(H_\varepsilon\). Since the profiles are bounded in \(K\), the transition region has \(\sigma\)-measure \(O(\varepsilon)\), and the skeleton hole has \(\sigma\)-measure \(O(\varepsilon^2)\), we get
\[
        G_\varepsilon\to u_B
        \qquad\text{in }L^1(\sigma).
\]
The same support estimate gives
\[
        \sigma(G_\varepsilon^{-1}(Q_i)\triangle B_i)=O(\varepsilon).
\]

It remains to compute the energy. The constant regions have zero cost. The skeleton hole satisfies
\[
        \varepsilon^{-1}\int_{H_\varepsilon}V(G_\varepsilon)\,d\sigma
        \le
        \varepsilon^{-1}\|V\|_{L^\infty(K)}\,\sigma(H_\varepsilon)
        =
        O(\varepsilon)
        \to0.
        \tag{A.26}
\]
For a face \(F\subset\partial B_i\cap\partial B_j\), the face-tube contribution is
\[
\begin{aligned}
        I_{\varepsilon,F}
        &=
        \varepsilon^{-1}
        \int_{F_\varepsilon}
        \int_{\mathbb R}
        V\!\left(
        \zeta_{ij}\!\left(\frac{s_\varepsilon}{\varepsilon}r\right)
        \right)
        \rho(y+r\nu_F)\,dr\,d\mathcal H^{n-1}(y),
\end{aligned}
\]
where \(F_\varepsilon:=\{y\in F:\operatorname{dist}(y,\partial F)>R\varepsilon\}\). With
\[
        \tau=\frac{s_\varepsilon}{\varepsilon}r,
        \qquad
        dr=\frac{\varepsilon}{s_\varepsilon}\,d\tau,
\]
we obtain
\[
\begin{aligned}
        I_{\varepsilon,F}
        &=
        \frac1{s_\varepsilon}
        \int_{F_\varepsilon}
        \int_{\mathbb R}
        V(\zeta_{ij}(\tau))
        \rho\!\left(y+\frac{\varepsilon}{s_\varepsilon}\tau\nu_F\right)
        d\tau\,d\mathcal H^{n-1}(y).
\end{aligned}
        \tag{A.27}
\]
The profile is constant outside a compact \(\tau\)-interval on which \(V(\zeta_{ij})\) is bounded. Since \(s_\varepsilon\to1\), \(F_\varepsilon\uparrow F\), and \(\rho\) is continuous, dominated convergence gives
\[
        \lim_{\varepsilon\downarrow0} I_{\varepsilon,F}
        =
        \left(\int_{\mathbb R}V(\zeta_{ij}(\tau))\,d\tau\right)
        \int_F \rho\,d\mathcal H^{n-1}.
        \tag{A.28}
\]
Using (A.23), summing over faces, and then sending \(\eta\downarrow0\), we obtain
\[
        \limsup_{\varepsilon\downarrow0}\mathcal F_\varepsilon(G_\varepsilon)
        \le
        \sum_{i<j}\kappa_{ij}
        \int_{\partial^*B_i\cap\partial^*B_j}
        \rho\,d\mathcal H^{n-1}.
\]
This proves the lemma.
\end{proof}

The preceding construction is the point at which the hard-gradient theorem differs from the soft Baldo/Fonseca--Tartar theory. In the soft Modica--Mortola functional, large gradients in a vanishingly small junction region are allowed as long as their integral cost vanishes. Here the constraint
\[
        \operatorname{Lip}(T)\le\varepsilon^{-1}
\]
is pointwise. Therefore one cannot simply paste several one-dimensional profiles at a triple junction and appeal to an integral estimate. The deletion of the \(O(\varepsilon)\)-tube around the codimension-two skeleton, the \(O(\varepsilon^2)\) volume of that hole, Kirszbraun extension, and projection back to \(\operatorname{conv}C\) are the hard-gradient replacement for the soft junction argument.

\subsection{Exact phase-volume fixing}

\begin{lemma}[Exact mass correction by a near-identity source diffeomorphism]
\label{lem:exact-mass-fixing}
Let \(B\) be a regular polyhedral partition with
\[
        \sigma(B_i)=\pi_i>0.
\]
Let \(G_\varepsilon\) be the recovery map constructed in Lemma~\ref{lem:hard-gradient-recovery}, and set
\[
        A_i^\varepsilon:=G_\varepsilon^{-1}(Q_i).
\]
Then, for all sufficiently small \(\varepsilon\), there exists a \(C^1\)-diffeomorphism
\[
        \Phi_\varepsilon:\mathbb R^n\to\mathbb R^n
\]
such that
\[
        \|\Phi_\varepsilon-\operatorname{Id}\|_{C^1}=O(\varepsilon),
        \qquad
        \operatorname{Lip}(\Phi_\varepsilon^{-1})\le1+C\varepsilon,
        \tag{A.29}
\]
and
\[
        \sigma(\Phi_\varepsilon(A_i^\varepsilon))=\pi_i
        \qquad\text{for every }i.
        \tag{A.30}
\]
Consequently
\[
        T_\varepsilon:=G_\varepsilon\circ\Phi_\varepsilon^{-1}
\]
satisfies
\[
        \sigma(T_\varepsilon^{-1}(Q_i))=\pi_i
        \qquad\text{for every }i,
        \tag{A.31}
\]
\[
        \operatorname{Lip}(T_\varepsilon)\le\varepsilon^{-1},
        \tag{A.32}
\]
and
\[
        \mathcal F_\varepsilon(T_\varepsilon)
        =
        \mathcal F_\varepsilon(G_\varepsilon)+O(\varepsilon).
        \tag{A.33}
\]
\end{lemma}

\begin{proof}
By (A.22),
\[
        \sigma(A_i^\varepsilon)=\pi_i+O(\varepsilon).
        \tag{A.34}
\]
Because \(B\) is a regular polyhedral partition with all \(\pi_i>0\), its positive-interface adjacency graph is connected after the standard polyhedral approximation step. Choose a spanning tree \(\mathcal T\) of this graph. For every oriented edge \(e=(i,j)\in\mathcal T\), choose a small ball \(U_e\) meeting a single flat face \(F_{ij}\subset\partial B_i\cap\partial B_j\), away from the skeleton, and choose these balls pairwise disjoint. Let \(\nu_{ij}\) be the normal pointing from \(B_i\) into \(B_j\). Choose a vector field
\[
        X_e\in C_c^\infty(U_e;\mathbb R^n)
\]
such that
\[
        \beta_e
        :=
        \int_{F_{ij}}\rho\, X_e\cdot\nu_{ij}\,d\mathcal H^{n-1}
        >
        0.
        \tag{A.35}
\]
For \(a=(a_e)_{e\in\mathcal T}\) small, let \(\Phi_a\) be the time-one flow of
\[
        X_a:=\sum_{e\in\mathcal T}a_eX_e.
\]
Then
\[
        \|\Phi_a-\operatorname{Id}\|_{C^1}\le C|a|,
        \qquad
        \operatorname{Lip}(\Phi_a^{-1})\le1+C|a|.
        \tag{A.36}
\]

Define the phase-mass map
\[
        m^\varepsilon(a)
        :=
        \big(\sigma(\Phi_a(A_1^\varepsilon)),\ldots,
             \sigma(\Phi_a(A_M^\varepsilon))\big).
\]
The derivative at \(a=0\) converges, as \(\varepsilon\downarrow0\), to the weighted tree-incidence map
\[
        L:\mathbb R^{M-1}\to\Big\{b\in\mathbb R^M:\sum_i b_i=0\Big\}
\]
given by
\[
        (Le)_i=
        \begin{cases}
        -\beta_e,& e=(i,j),\\
        +\beta_e,& e=(j,i),\\
        0,&\text{otherwise}
        \end{cases}
\]
on each oriented tree edge \(e\). This is the usual incidence matrix of a tree, with positive edge weights \(\beta_e\). It has rank \(M-1\), and hence is invertible onto the codimension-one mass hyperplane.

Since
\[
        m^\varepsilon(0)-\pi=O(\varepsilon)
        \quad\text{and}\quad
        Dm^\varepsilon(0)\to L,
\]
the inverse function theorem, uniformly for small \(\varepsilon\), gives \(a_\varepsilon=O(\varepsilon)\) such that
\[
        m^\varepsilon(a_\varepsilon)=\pi.
\]
Set
\[
        \Phi_\varepsilon:=\Phi_{a_\varepsilon}.
\]
This proves (A.29) and (A.30). Since
\[
        T_\varepsilon^{-1}(Q_i)
        =
        \Phi_\varepsilon(A_i^\varepsilon),
\]
(A.31) follows.

The Lipschitz estimate follows from Lemma~\ref{lem:hard-gradient-recovery} and (A.29):
\[
        \operatorname{Lip}(T_\varepsilon)
        \le
        \operatorname{Lip}(G_\varepsilon)\operatorname{Lip}(\Phi_\varepsilon^{-1})
        \le
        \frac{1-\sqrt\varepsilon}{\varepsilon}(1+C\varepsilon)
        \le
        \varepsilon^{-1}
\]
for all sufficiently small \(\varepsilon\).

Finally,
\[
\begin{aligned}
        \int V(T_\varepsilon)\,d\sigma
        &=
        \int V(G_\varepsilon(y))
        \rho(\Phi_\varepsilon(y))
        |\det D\Phi_\varepsilon(y)|\,dy.
\end{aligned}
\]
Since \(\Phi_\varepsilon=\operatorname{Id}+O(\varepsilon)\) in \(C^1\), \(\rho\in C^1\), and \(G_\varepsilon\) has unscaled cost \(O(\varepsilon)\),
\[
        \int V(T_\varepsilon)\,d\sigma
        =
        \int V(G_\varepsilon)\,d\sigma
        +
        O(\varepsilon^2).
\]
After multiplying by \(\varepsilon^{-1}\), this gives (A.33).
\end{proof}

\subsection{General partitions and the limsup inequality}

\begin{proof}[Proof of the limsup inequality]
Let \(B\) be a Caccioppoli partition with finite energy. By the polyhedral density theorem for partitions \citep{BraidesContiGarroni2017}, combined with a standard truncation argument and the continuity of the weight \(\rho\), there exist regular polyhedral partitions \(B^m\) such that
\[
        \sum_i\sigma(B_i^m\triangle B_i)\to0
\]
and
\[
        E_{p,\sigma}(B^m;C)\to E_{p,\sigma}(B;C).
\]
If the masses \(\sigma(B_i)\) are prescribed and positive, the approximation may be taken with
\[
        \sigma(B_i^m)=\sigma(B_i)
\]
for every \(i\). The volume correction is obtained by the same finite-dimensional flow argument used in Lemma~\ref{lem:exact-mass-fixing}: choose a spanning tree in the adjacency graph, choose disjoint vector fields crossing one face per tree edge, and apply the inverse function theorem to the weighted phase-volume map. The correction is \(C^1\)-small and changes the weighted perimeter by \(o(1)\).

For each \(m\), Lemma~\ref{lem:hard-gradient-recovery} gives a recovery sequence for \(B^m\). If exact masses are required, Lemma~\ref{lem:exact-mass-fixing} gives a corrected sequence with exact nearest-well phase masses. A diagonal choice \(m=m(\varepsilon)\to\infty\) yields
\[
        T_\varepsilon\to u_B
        \qquad\text{in }L^1(\sigma),
\]
and
\[
        \limsup_{\varepsilon\downarrow0}
        \mathcal F_\varepsilon(T_\varepsilon)
        \le
        E_{p,\sigma}(B;C).
\]
Together with Proposition~\ref{prop:gamma-liminf}, this proves the \(\Gamma\)-convergence.
\end{proof}

\subsection{The constrained energy minimum}

\begin{proof}[Proof of (A.1)]
For the lower bound, let \(T_\varepsilon\) be any sequence with
\[
        \operatorname{Lip}(T_\varepsilon)\le\varepsilon^{-1},
        \qquad
        \sigma(T_\varepsilon^{-1}(Q_i))=\pi_i,
\]
and bounded energy. By compactness,
\[
        T_\varepsilon\to u_B
        \qquad\text{in }L^1(\sigma)
\]
for a Caccioppoli partition \(B\). The exact phase constraints pass to the limit, so
\[
        \sigma(B_i)=\pi_i.
\]
The liminf inequality gives
\[
        \liminf_{\varepsilon\downarrow0}\mathcal F_\varepsilon(T_\varepsilon)
        \ge
        E_{p,\sigma}(B;C)
        \ge
        J_{p,\sigma}(\pi;C).
\]

For the upper bound, choose a Caccioppoli partition \(B\) with
\[
        \sigma(B_i)=\pi_i
\]
and
\[
        E_{p,\sigma}(B;C)
        \le
        J_{p,\sigma}(\pi;C)+\eta.
\]
The exact-mass recovery sequence constructed above satisfies
\[
        \sigma(T_\varepsilon^{-1}(Q_i))=\pi_i,
        \qquad
        \operatorname{Lip}(T_\varepsilon)\le\varepsilon^{-1},
\]
and
\[
        \limsup_{\varepsilon\downarrow0}
        \mathcal F_\varepsilon(T_\varepsilon)
        \le
        E_{p,\sigma}(B;C)
        \le
        J_{p,\sigma}(\pi;C)+\eta.
\]
Letting \(\eta\downarrow0\) proves (A.1).
\end{proof}

\subsection{The \(W_p\) corollary}

Let
\[
        \mu=\sum_{i=1}^M\pi_i\delta_{c_i}.
\]
We first record the exact identity used for the upper bound. If
\[
        \sigma(T^{-1}(Q_i))=\pi_i
        \qquad\text{for every }i,
\]
then
\[
        W_p^p(T_\#\sigma,\mu)
        =
        \int_{\mathbb R^n}\operatorname{dist}(T(x),C)^p\,d\sigma(x).
        \tag{A.37}
\]
Indeed, for any coupling between \(Y\sim T_\#\sigma\) and \(Z\in C\) with law \(\mu\),
\[
        |Y-Z|^p\ge \operatorname{dist}(Y,C)^p,
\]
so
\[
        W_p^p(T_\#\sigma,\mu)
        \ge
        \int\operatorname{dist}(T,C)^p\,d\sigma.
\]
Conversely, because the nearest-well phase masses match \(\pi_i\), the coupling
\[
        x\mapsto \big(T(x),c_{q(T(x))}\big)
\]
has second marginal \(\mu\) and cost exactly
\[
        \int |T(x)-c_{q(T(x))}|^p\,d\sigma(x)
        =
        \int\operatorname{dist}(T(x),C)^p\,d\sigma(x).
\]
This proves (A.37).

\begin{proof}[Proof of (A.2)]
Set \(\varepsilon=\Lambda^{-1}\).

The upper bound follows from the exact-mass recovery sequence. For such a sequence,
\[
\begin{aligned}
        \Lambda W_p^p(T_\varepsilon{}_\#\sigma,\mu)
        &=
        \varepsilon^{-1}
        \int\operatorname{dist}(T_\varepsilon,C)^p\,d\sigma                 \\
        &=
        \mathcal F_\varepsilon(T_\varepsilon)
        \to
        J_{p,\sigma}(\pi;C)
\end{aligned}
\]
after optimizing over partitions.

For the lower bound, let \(\Lambda_k\to\infty\), set \(\varepsilon_k=\Lambda_k^{-1}\), and choose \(T_k\) with
\[
        \operatorname{Lip}(T_k)\le\Lambda_k
\]
such that
\[
        \sup_k
        \Lambda_k W_p^p(T_k{}_\#\sigma,\mu)<\infty.
        \tag{A.38}
\]
Let \(Y_k=T_k(X)\) with \(X\sim\sigma\), and let \(Z_k\in C\) be coupled to \(Y_k\) optimally, with \(Z_k\sim\mu\). Since
\[
        |Y_k-Z_k|^p\ge\operatorname{dist}(Y_k,C)^p,
\]
(A.38) implies
\[
        \mathcal F_{\varepsilon_k}(T_k)
        =
        \varepsilon_k^{-1}
        \int\operatorname{dist}(T_k,C)^p\,d\sigma
        \le
        \Lambda_k W_p^p(T_k{}_\#\sigma,\mu)
        \le C.
        \tag{A.39}
\]
By compactness, after passing to a subsequence,
\[
        T_k\to u_B
        \qquad\text{in }L^1(\sigma)
\]
for a Caccioppoli partition \(B\).

It remains to identify the masses of \(B\). Let
\[
        \Delta_C:=\min_{i\ne j}|c_i-c_j|>0.
\]
If \(q(Y_k)\ne q(Z_k)\), then \(Y_k\) lies in a Voronoi cell different from the atom \(Z_k\). Therefore
\[
        |Y_k-Z_k|\ge \Delta_C/2,
\]
up to the fixed tie-breaking null ambiguity. Hence
\[
        \mathbb P(q(Y_k)\ne q(Z_k))
        \le
        \left(\frac{2}{\Delta_C}\right)^p
        \mathbb E|Y_k-Z_k|^p
        =
        O(\varepsilon_k).
        \tag{A.40}
\]
The law of \(q(Z_k)\) is \(\pi\). Thus the law of \(q(Y_k)\) converges to \(\pi\) in total variation. On the other hand, compactness gives
\[
        \sigma(q(T_k)=i)\to\sigma(B_i).
\]
Therefore
\[
        \sigma(B_i)=\pi_i
        \qquad\text{for every }i.
\]
Using (A.39) and the liminf inequality,
\[
\begin{aligned}
        \liminf_k
        \Lambda_kW_p^p(T_k{}_\#\sigma,\mu)
        &\ge
        \liminf_k \mathcal F_{\varepsilon_k}(T_k)                         \\
        &\ge
        E_{p,\sigma}(B;C)                                                   \\
        &\ge
        J_{p,\sigma}(\pi;C).
\end{aligned}
\]
This proves the lower bound and hence (A.2).
\end{proof}

\subsection{Scope and sharpness of the hypotheses}

\paragraph{Positive masses are necessary.}
The assumption \(\pi_i>0\) is not cosmetic. If zero-mass wells are allowed, exact nearest-well phase constraints can be incompatible with continuity. In one dimension, take
\[
        n=d=1,
        \qquad
        C=\{0,1,2\},
        \qquad
        \pi=(1/2,0,1/2).
\]
A continuous map \(T:\mathbb R\to\mathbb R\) that realizes positive mass near both \(0\) and \(2\) has connected image between those values and must cross the interior of the Voronoi region of the middle well \(1\). Under an exact phase constraint this forces positive phase mass for the middle well, contradicting \(\pi_2=0\). Thus the exact-volume theorem must assume \(\pi_i>0\).

\paragraph{Finite perimeter is a compactness conclusion, not a property of one Lipschitz map.}
A single \(\Lambda\)-Lipschitz map need not have finite-perimeter phase preimages. For example, if \(E\subset\mathbb R^n\) has infinite perimeter, the signed-distance function
\[
        s_E(x):=\operatorname{dist}(x,E^c)-\operatorname{dist}(x,E)
\]
is Lipschitz, and a clipped version of \(s_E\) has a phase preimage equal to \(E\) up to null sets. Therefore the correct statement is not that every Lipschitz phase preimage has finite perimeter. Rather, finite perimeter appears only for limits of bounded-energy sequences, as proved in Proposition~\ref{prop:gamma-compactness}.

\paragraph{Cores and non-Euclidean targets.}
For atomic Euclidean wells, the proof above is complete. For extended cores \(A_i\), the same calibration works only after replacing each core by its \(d_V\)-zero equivalence component: one needs
\[
        d_V(a,a')=0
        \qquad
        \text{for all }a,a'\in A_i
\]
inside each core component, and strictly positive separation between different components. Otherwise switching inside a core can carry hidden interface cost. For general metric targets there is a genuine additional gap: the recovery proof uses Kirszbraun extension and Euclidean projection onto \(\operatorname{conv}C\). Both are Euclidean-specific and are not automatic in arbitrary metric spaces.

\section{Proof of Theorem~\ref{thm:dim} (Gaussian Dimension Phase Diagram)}\label{app:thm2}

\paragraph{Part (i): universal lower bound.}
Gaussian isoperimetry per cell gives $P_{\gauss_n}(\partial^* B_i) \ge I_\gauss(\gauss_n(B_i)) = \varphi(\Phi^{-1}(\gauss_n(B_i)))$.
For $\gauss_n(B_i) = 1/M$, summing pairwise and using $\sum_{i<j} P_{\gauss_n}(B_i, B_j) = \frac{1}{2}\sum_i P_{\gauss_n}(\partial^* B_i) \ge \frac{M}{2}\varphi(\Phi^{-1}(1/M))$.
Mills ratio: $M\varphi(\Phi^{-1}(1/M)) \sim \sqrt{2\log M}$ as $M \to \infty$.

\paragraph{Part (ii): high-dimensional upper bound.}
Transitive cyclic Fourier code with $r = \lceil 16\log M \rceil$ random frequencies.
Hoeffding plus union bound gives $\Prob[\exists j \ne j'\!: \|c_j - c_{j'}\| < 1] \le M^{-1}$.
Cyclic isometry forces normal-fan cells of equal mass $1/M$.
Gaussian width $\E\max_i \inner{G}{c_i} \le \sqrt{2\log M}$ (subgaussian maxima), so Lemma~\ref{lem:normalfan} gives $\mathcal{P}_{n,M} \le \sqrt{2\log M}$ once $n \ge 35\log M$.

\paragraph{Part (iii): fixed dimension.}
Lower: relative isoperimetric inequality in a ball; $\gauss_n(B_i) \le C/M$ implies $P_{\gauss_n}(\partial^* B_i) \ge c_n \gauss_n(B_i)^{(n-1)/n} \ge c_n M^{1/n} \gauss_n(B_i)$.
Upper: product quantile slabs for $M = k^n$; balanced recursive quantile partition for general $M$; constants $C_n \le Cn$.

\begin{lemma}[Gaussian-width identity]\label{lem:normalfan}
Let $c_1, \ldots, c_M$ be the vertices of a polytope in $\R^n$, $f(x) = \max_i \inner{x}{c_i}$, and $(B_i)$ the normal-fan partition. Then $\sum_{i<j} \norm{c_i - c_j}\, P_{\gauss_n}(B_i, B_j) = \E \max_i \inner{G}{c_i}$, $G \sim \gauss_n$.
\end{lemma}
\paragraph{Proof.}
$f(x) = \max_i \inner{x}{c_i}$ is convex, piecewise linear.
The distributional Laplacian's interface density on $\partial^* B_i \cap \partial^* B_j$ is $\norm{c_i - c_j}$ (the outward normal between cells is $(c_j - c_i)/\norm{c_j - c_i}$).
Triple junctions have dimension $\le n - 2$.
Gaussian integration by parts with cutoff $\chi_R$ and Euler's identity $x^\top \nabla f = f$ give $\int \varphi_n\,d(\Delta f) = \E\max_i \inner{G}{c_i}$.

\section{Proof of Theorem~\ref{thm:noncommit} (Non-Commitment): active-facet tube law}
\label{app:noncommit}

We give the complete proof in the rescaled observation variable
\[
        Z_s := U_t/t = X+s\epsilon,\qquad
        s=s(t):=s_0(1-t)/t,\qquad \epsilon\sim N(0,I_d).
\]
Conditioning on \(U_t\) and conditioning on \(Z_s\) are equivalent, and
\[
        m_t(U_t)=m_s(Z_s),\qquad
        m_s(z):=\E[X\mid Z_s=z].
\]
All statements below are conditional on the readout matrix \(W\). If \(W\) is random,
the final identities may be averaged over \(W\).

Let
\[
        \Omega_i := \{z:\ w_i^\top z\ge w_k^\top z\ \text{for all } k\}
\]
be the closed argmax cell of token \(i\), with an arbitrary fixed tie-breaking rule
on the cell boundaries. Ties have zero \(X\)-probability under the assumptions below.
For \(i\ne j\), put
\[
        n_{ij}:=\frac{w_i-w_j}{\|w_i-w_j\|},
        \qquad
        \Sigma_{ij}:=\{z:\ w_i^\top z=w_j^\top z\}.
\]
We orient \(n_{ij}\) so that \(n_{ij}\) points from \(\Omega_j\) into \(\Omega_i\). The
\emph{active co-maximal facet} between \(i\) and \(j\) is
\[
        F_{ij}:=
        \Big\{z\in\Sigma_{ij}:\
        w_i^\top z=w_j^\top z=\max_k w_k^\top z
        \Big\}.
\]
Equivalently, \(F_{ij}=\partial\Omega_i\cap\partial\Omega_j\) up to lower-dimensional
tie sets. Non-active pairwise hyperplanes \(\Sigma_{ij}\) do not contribute to the
leading tube mass. Define the active boundary density
\[
        A_W
        :=
        \sum_{i<j}
        \int_{\Sigma_{ij}}
        \mathbf 1\!\left\{w_i^\top z=w_j^\top z=\max_k w_k^\top z\right\}
        \rho_X(z)\,d\mathcal H^{d-1}(z)
        =
        \sum_{i<j}\int_{F_{ij}}\rho_X\,d\mathcal H^{d-1}.
        \tag{C.1}
\]
The second equality uses the regularity assumption below.

\paragraph{Regularity assumption.}
We assume throughout this appendix that:
\begin{enumerate}
\item[(R1)] \(Y=D_W(X)\) almost surely, unless explicitly stated otherwise in
      the ``Imperfect clean decoding'' paragraph below.
\item[(R2)] \(\rho_X\) is \(C^2\) in a neighborhood of the active boundary
      \(\mathcal S:=\bigcup_{i<j}F_{ij}\), and the usual first-order tube formula
      holds with density \(\rho_X\) along each active facet.
\item[(R3)] The co-maximal skeleton
      \[
      \mathcal T:=\{z:\ \geq 3 \text{ logits } w_k^\top z \text{ tie for the maximum}\}
      \]
      has weighted tube mass \(\P[\operatorname{dist}(X,\mathcal T)\le r]=O(r^2)\) as \(r\downarrow0\).
      In particular, \(\mathcal H^{d-1}(\mathcal T)=0\) on the active boundary.
\item[(R4)] \(A_W<\infty\), and the same finiteness holds for the first two normal
      derivatives of \(\rho_X\) in a fixed small tube around \(\mathcal S\). This is
      only to justify exchanging Taylor expansion and integration on unbounded facets.
\end{enumerate}
These are the standard smooth-density, regular-polyhedral-boundary hypotheses under
which Gaussian small-noise classification has a first-order surface-area expansion.
They also rule out the degenerate case in which a pairwise bisector lies inside
a third token's cell: such a bisector is not co-maximal, hence is not part of any
\(F_{ij}\), and contributes no \(O(s)\) term.

\paragraph{Equality, Bayes bound, and cross-entropy bound.}
Let
\[
        \widehat Y_s:=D_W(m_s(Z_s)),
        \qquad
        q_s(y\mid z):=\P(Y=y\mid Z_s=z).
\]
Conditioning on \(Z_s=z\), the value \(m_s(z)\) and hence
\(d_s(z):=D_W(m_s(z))\) are deterministic, whereas \(Y\) is distributed according
to \(q_s(\cdot\mid z)\). Therefore
\[
        \P(\widehat Y_s\ne Y\mid Z_s=z)
        =
        1-q_s(d_s(z)\mid z),
\]
and hence
\[
        \P(\widehat Y_s\ne Y)
        =
        \E\!\left[1-q_s(D_W(m_s(Z_s))\mid Z_s)\right].
        \tag{C.2}
\]
Since \(q_s(d_s(z)\mid z)\le \max_y q_s(y\mid z)\),
\[
        \P(\widehat Y_s\ne Y)
        \ge
        B_s
        :=
        \E\!\left[1-\max_y q_s(y\mid Z_s)\right].
        \tag{C.3}
\]
Here \(B_s\) is the Bayes non-commitment at noise scale \(s\).

Let
\[
        p_W(y\mid z)
        :=
        \frac{\exp(w_y^\top z)}
        {\sum_k \exp(w_k^\top z)}
\]
be the softmax associated with the readout weights. For \(d=d_s(z)=D_W(m_s(z))\)
and any \(y\ne d\),
\[
\begin{aligned}
        -\log p_W(y\mid m_s(z))
        &=
        \log\sum_{k}\exp\!\big((w_k-w_y)^\top m_s(z)\big)   \\
        &\ge
        \log\!\left(1+\exp\!\big((w_d-w_y)^\top m_s(z)\big)\right)
        \ge \log 2,
\end{aligned}
\]
because \((w_d-w_y)^\top m_s(z)\ge0\). Averaging under
\(Y\mid Z_s=z\) gives
\[
\begin{aligned}
        \E[-\log p_W(Y\mid m_s(Z_s))]
        &\ge
        (\log 2)\,
        \E\!\left[1-q_s(d_s(Z_s)\mid Z_s)\right]       \\
        &\ge
        (\log 2)\,
        \E\!\left[1-\max_y q_s(y\mid Z_s)\right].
\end{aligned}
        \tag{C.4}
\]

\paragraph{Bayes tube law on active facets.}
For each token \(i\), define the class-conditional blurred density
\[
        a_{i,s}(z)
        :=
        \int_{\Omega_i}\rho_X(x)\,\varphi_s(z-x)\,dx,
        \qquad
        \varphi_s(v):=(2\pi s^2)^{-d/2}\exp(-\|v\|^2/2s^2).
\]
Then
\[
        q_s(i\mid z)=\frac{a_{i,s}(z)}{\sum_k a_{k,s}(z)}
\]
and the Bayes non-commitment can be written exactly as
\[
        B_s
        =
        \int_{\mathbb R^d}
        \left(\sum_k a_{k,s}(z)-\max_k a_{k,s}(z)\right)\,dz.
        \tag{C.5}
\]

We now compute the first-order contribution of one active facet. Fix \(i<j\) and
work on a compact subset of the relative interior of \(F_{ij}\), away from the
triple skeleton. In the normal coordinates
\[
        z=x+s a\,n_{ij},\qquad x\in F_{ij},\quad a\in\mathbb R,
\]
the cells are locally the two half-spaces
\[
        \Omega_i=\{r\ge0\},\qquad \Omega_j=\{r\le0\}
\]
in the signed normal coordinate \(r=\langle x'-x,n_{ij}\rangle\). All other cells
are separated by a positive logit gap in this local chart, and their blurred
contributions are exponentially small in \(s^{-2}\). Taylor expanding \(\rho_X\)
in the normal and tangential variables gives, uniformly for bounded \(a\),
\[
        a_{i,s}(x+s a n_{ij})
        =
        \rho_X(x)\Phi(a)+O(s),
        \qquad
        a_{j,s}(x+s a n_{ij})
        =
        \rho_X(x)\Phi(-a)+O(s),
        \tag{C.6}
\]
where \(\Phi\) is the standard normal cdf. Thus, on the \(F_{ij}\)-tube,
\[
        \sum_k a_{k,s}(z)-\max_k a_{k,s}(z)
        =
        \rho_X(x)\min\{\Phi(a),\Phi(-a)\}+O(s),
        \tag{C.7}
\]
up to an exponentially small error from non-competing classes.

Using \(dz=s\,da\,d\mathcal H^{d-1}(x)\) in these flat normal coordinates, the
contribution of \(F_{ij}\) is
\[
\begin{aligned}
        I_{ij,s}
        &=
        s
        \int_{F_{ij}}\rho_X(x)\,d\mathcal H^{d-1}(x)
        \int_{\mathbb R}\min\{\Phi(a),\Phi(-a)\}\,da
        +o(s)                                                        \\
        &=
        s
        \int_{F_{ij}}\rho_X(x)\,d\mathcal H^{d-1}(x)
        \left(2\int_0^\infty \Phi(-a)\,da\right)
        +o(s)                                                        \\
        &=
        \sqrt{\frac{2}{\pi}}\,s
        \int_{F_{ij}}\rho_X(x)\,d\mathcal H^{d-1}(x)
        +o(s).
\end{aligned}
        \tag{C.8}
\]
The identity
\[
        2\int_0^\infty \Phi(-a)\,da
        =
        2\E[G_+]
        =
        \sqrt{\frac{2}{\pi}},
        \qquad G\sim N(0,1),
\]
is the one-dimensional Laplace constant. This proves that the per-facet
\(\sqrt{2/\pi}\) constant is unchanged after restricting to active co-maximal
facets.

It remains to justify that no other region contributes at order \(s\). Away from
an \(O(s\sqrt{\log(1/s)})\)-tube of the active boundary, the posterior mass of every
non-argmax token is exponentially small. An \(O(s\sqrt{\log(1/s)})\)-tube of the
triple skeleton has probability \(O(s^2\log(1/s))=o(s)\) by (R3). Finally, a
pairwise hyperplane \(\Sigma_{ij}\) that is not co-maximal is either inactive in a
neighborhood of the decision boundary or meets it only in the triple skeleton; it
therefore contributes \(o(s)\). Summing (C.8) over active facets yields
\[
        B_s
        =
        \sqrt{\frac{2}{\pi}}\,s
        \sum_{i<j}\int_{F_{ij}}\rho_X\,d\mathcal H^{d-1}
        +o(s)
        =
        \sqrt{\frac{2}{\pi}}\,s\,A_W+o(s).
        \tag{C.9}
\]

\paragraph{Posterior-mean readout versus MAP readout.}
Let
\[
        d_s^{\mathrm{MAP}}(z):=\arg\max_y q_s(y\mid z),
        \qquad
        d_s^{\mathrm{mean}}(z):=D_W(m_s(z)).
\]
The exact discrepancy identity is
\[
\begin{aligned}
        \P(d_s^{\mathrm{mean}}(Z_s)\ne Y)-B_s
        &=
        \E\!\left[
        q_s(d_s^{\mathrm{MAP}}(Z_s)\mid Z_s)
        -
        q_s(d_s^{\mathrm{mean}}(Z_s)\mid Z_s)
        \right]                                                     \\
        &\ge0.
\end{aligned}
        \tag{C.10}
\]
Consequently,
\[
        B_s
        \le
        \P(d_s^{\mathrm{mean}}(Z_s)\ne Y)
        \le
        B_s+
        \P\!\left[d_s^{\mathrm{mean}}(Z_s)\ne d_s^{\mathrm{MAP}}(Z_s)\right].
        \tag{C.11}
\]
Thus the posterior-mean flip rate and the Bayes non-commitment have the same
leading constant if and only if the expectation in (C.10) is \(o(s)\). A sufficient
and easy-to-check condition is
\[
        \P\!\left[d_s^{\mathrm{mean}}(Z_s)\ne d_s^{\mathrm{MAP}}(Z_s)\right]=o(s).
        \tag{C.12}
\]

We now verify (C.12) under the smooth clean-decoding assumptions above. The only
possible \(O(s)\)-mass disagreement region lies in active facet tubes. Fix a regular
point \(x\in F_{ij}\setminus\mathcal T\), and write
\[
        z=x+s a n_{ij},\qquad
        \kappa_{ij}(x):=\partial_{n_{ij}}\log\rho_X(x).
\]
First, by the same half-space calculation as above,
\[
        a_{i,s}(z)-a_{j,s}(z)
        =
        \rho_X(x)\big(2\Phi(a)-1\big)
        +
        2s\,\partial_{n_{ij}}\rho_X(x)\,\varphi(a)
        +
        O(s^2),
        \tag{C.13}
\]
where \(\varphi\) is the standard normal density. Since the common denominator
\(\sum_k a_{k,s}(z)\) is positive, \(d_s^{\mathrm{MAP}}(z)\) switches between \(i\)
and \(j\) where the right-hand side of (C.13) is zero. Near \(a=0\),
\[
        2\Phi(a)-1=2\varphi(0)a+O(a^3),
        \qquad
        \varphi(a)=\varphi(0)+O(a^2),
\]
so the MAP switching surface has signed normal coordinate
\[
        a^{\mathrm{MAP}}_{ij}(x,s)
        =
        -s\,\kappa_{ij}(x)+O(s^2).
        \tag{C.14}
\]

Second, the posterior mean satisfies Tweedie's formula
\[
        m_s(z)
        =
        z+s^2\nabla\log p_s(z),
        \qquad
        p_s:=\rho_X*\varphi_s.
        \tag{C.15}
\]
Since \(p_s=\rho_X+O(s^2)\) in \(C^1\) near the active boundary,
\[
\begin{aligned}
        \delta^*_{ij}(m_s(x+s a n_{ij}))
        &=
        n_{ij}^\top m_s(x+s a n_{ij}) - n_{ij}^\top x               \\
        &=
        s a+s^2\partial_{n_{ij}}\log\rho_X(x)+O(s^3)                \\
        &=
        s\big(a+s\kappa_{ij}(x)+O(s^2)\big).
\end{aligned}
        \tag{C.16}
\]
Therefore the posterior-mean readout switches between \(i\) and \(j\) at
\[
        a^{\mathrm{mean}}_{ij}(x,s)
        =
        -s\,\kappa_{ij}(x)+O(s^2).
        \tag{C.17}
\]
Equations (C.14) and (C.17) show that the MAP and posterior-mean switching
surfaces differ by \(O(s^2)\) in the \(a\)-coordinate, hence by \(O(s^3)\) in the
original signed normal coordinate. Their disagreement tube around \(F_{ij}\) has
weighted mass \(o(s)\). Summing over finitely many active facets and adding the
\(o(s)\) triple-skeleton contribution proves (C.12). Combining with (C.9) and
(C.11) gives the active-facet form of the tube law:
\[
        \boxed{
        \P(\widehat Y_s\ne Y)
        =
        \sqrt{\frac{2}{\pi}}\,s
        \sum_{i<j}\int_{F_{ij}}\rho_X\,d\mathcal H^{d-1}
        +o(s).
        }
        \tag{C.18}
\]
Returning to \(s=s(t)=s_0(1-t)/t\),
\[
        \P(\widehat Y_t\ne Y)
        =
        \sqrt{\frac{2}{\pi}}\,s(t)
        \sum_{i<j}\int_{F_{ij}}\rho_X\,d\mathcal H^{d-1}
        +o(s(t)).
        \tag{C.19}
\]

\paragraph{Geometric meaning of the MAP--mean condition.}
The proof above uses smoothness of \(\rho_X\) to show that MAP and posterior-mean
readout have the same switching surface to first order. In a more general labelled
mixture model, this is a genuine hypothesis, not a free consequence.

Near a regular \(F_{ij}\), write
\[
        \mu_i^s(z):=\E[X\mid Z_s=z,Y=i],
        \qquad
        \mu_j^s(z):=\E[X\mid Z_s=z,Y=j].
\]
Ignoring exponentially small non-competing classes,
\[
        m_s(z)
        =
        q_s(i\mid z)\mu_i^s(z)+q_s(j\mid z)\mu_j^s(z)+o(s).
\]
Since the readout boundary between \(i\) and \(j\) is linear,
\[
        \delta^*_{ij}(m_s(z))
        =
        q_s(i\mid z)\delta^*_{ij}(\mu_i^s(z))
        +
        q_s(j\mid z)\delta^*_{ij}(\mu_j^s(z))
        +o(s).
        \tag{C.20}
\]
The posterior-mean readout chooses \(i\) rather than \(j\) exactly when the left-hand
side of (C.20) is nonnegative. The MAP rule chooses \(i\) when
\(q_s(i\mid z)\ge q_s(j\mid z)\). Hence the sharp general condition is:
\[
        \E\!\left[
        q_s(d_s^{\mathrm{MAP}}(Z_s)\mid Z_s)
        -
        q_s(d_s^{\mathrm{mean}}(Z_s)\mid Z_s)
        \right]
        =
        o(s).
        \tag{C.21}
\]

A simple sufficient geometric version is that, for \(\mathcal H^{d-1}\)-a.e.
\(x\in F_{ij}\), the two conditional cluster means lie on opposite sides of
\(\Sigma_{ij}\), and the zero of
\(\delta^*_{ij}(m_s(x+s a n_{ij}))\) differs by \(o(1)\) in the \(a\)-coordinate
from the zero of \(q_s(i\mid x+s a n_{ij})-q_s(j\mid x+s a n_{ij})\). Equivalently,
the MAP and posterior-mean switching surfaces are separated by an \(o(s)\) physical
normal distance. Under the clean-label smooth-density assumptions (R1)--(R4), this
condition is automatic by (C.14)--(C.17). Without it, the sharp statement is the
inequality (C.11); an \(O(s)\)-measure MAP--mean disagreement set can change the
leading constant.

\paragraph{Row-space flip criterion.}
For later reference, the posterior-mean decoded token differs from the clean decoded
token \(Y\) only if some competitor \(j\ne Y\) satisfies
\[
        w_j^\top m_s(Z_s)\ge w_Y^\top m_s(Z_s).
\]
Equivalently,
\[
        \frac{(w_Y-w_j)^\top(X-m_s(Z_s))}
        {\|w_Y-w_j\|}
        \ge
        \delta^*_{Yj}(X).
        \tag{C.22}
\]
Thus only the projection of \(X-m_s(Z_s)\) onto the readout row-space directions
\(w_Y-w_j\) can cause a token flip. Components orthogonal to all readout differences
are invisible to the argmax readout.

\paragraph{Imperfect clean decoding.}\label{par:clean-accuracy-perturbation}
The preceding theorem should be interpreted as a statement about \emph{readout
self-consistency}. If the dataset token \(Y^{\rm raw}\) is not always equal to the
clean decoded token, define
\[
        Y^\circ:=D_W(X),
        \qquad
        \eta_0:=\P(Y^{\rm raw}\ne Y^\circ).
\]
Then Theorem~3 applies exactly to \(Y^\circ\). For the raw token labels,
\[
\begin{aligned}
        \left|
        \P(\widehat Y_s\ne Y^{\rm raw})
        -
        \P(\widehat Y_s\ne Y^\circ)
        \right|
        &\le
        \P(Y^{\rm raw}\ne Y^\circ)
        =
        \eta_0.
\end{aligned}
        \tag{C.23}
\]
Indeed, for any deterministic decoded token \(d\),
\[
        \left|
        \mathbf 1\{d\ne Y^{\rm raw}\}
        -
        \mathbf 1\{d\ne Y^\circ\}
        \right|
        \le
        \mathbf 1\{Y^{\rm raw}\ne Y^\circ\}.
\]
Therefore
\[
        \P(\widehat Y_s\ne Y^{\rm raw})
        =
        \sqrt{\frac{2}{\pi}}\,s\,A_W+o(s)+O(\eta_0).
        \tag{C.24}
\]
The \(O(\eta_0)\) term is an additive clean-autoencoder error offset, not part of
the readout tube constant. However, if one insists on plotting raw-token error
without subtracting or controlling the clean error, \(\eta_0\) gives a non-vanishing
intercept as \(s\downarrow0\). In fact,
\[
        \lim_{s\downarrow0}\P(\widehat Y_s\ne Y^{\rm raw})=\eta_0,
\]
because \(m_s(Z_s)\to X\) in probability and \(D_W(m_s(Z_s))\to D_W(X)=Y^\circ\)
away from measure-zero ties.

Thus the cleanest theorem and the cleanest experiment should use
\[
        Y:=Y^\circ=D_W(X)
\]
by definition. This gains an exact readout-calibrated statement. What is lost is
direct accounting of the autoencoder's semantic clean error; that error should be
reported separately as \(\eta_0\). For ELF-style clean accuracy \(1-\eta_0\), the
raw-token curve should be interpreted as ``clean-error offset plus tube error,''
whereas the \(Y^\circ\)-curve is the actual test of Theorem~3.

If one wants a slope statement for raw labels, an additional hypothesis is required:
the misdecoded set \(\{Y^{\rm raw}\ne Y^\circ\}\) must not concentrate in an
\(O(s)\)-tube of the active boundary. Under such a no-boundary-concentration
condition, the raw-label slope differs from the clean-decoded slope by a term
controlled by the active-facet boundary density of the misdecoded subset. Without
that extra condition, (C.23) is the sharp uniform perturbation bound.

\paragraph{Empirical predictor from the normalized-margin distribution.}
Let the clean-decoded margin be
\[
        \delta^*(X,Y^\circ)
        =
        \min_{j\ne Y^\circ}
        \frac{(w_{Y^\circ}-w_j)^\top X}
        {\|w_{Y^\circ}-w_j\|}.
\]
Under (R1)--(R4), \(\delta^*(X,Y^\circ)\ge0\) almost surely, and near zero it is the
Euclidean distance to the nearest active readout facet. Non-active pairwise
hyperplanes and triple intersections contribute only \(o(\varepsilon)\) mass.
Hence
\[
\begin{aligned}
        \P\!\left[0\le \delta^*(X,Y^\circ)\le \varepsilon\right]
        &=
        2\varepsilon
        \sum_{i<j}\int_{F_{ij}}\rho_X\,d\mathcal H^{d-1}
        +o(\varepsilon)                                             \\
        &=
        2\varepsilon A_W+o(\varepsilon).
\end{aligned}
        \tag{C.25}
\]
The factor \(2\) is important: the margin distribution sees both sides of each
active facet, one from \(\Omega_i\) and one from \(\Omega_j\). Therefore the
right-density of the margin at zero is
\[
        f_{\delta^*}(0+)
        :=
        \lim_{\varepsilon\downarrow0}
        \frac{\P[0\le\delta^*(X,Y^\circ)\le\varepsilon]}{\varepsilon}
        =
        2A_W.
        \tag{C.26}
\]
Combining (C.18) and (C.26), the predicted posterior-mean flip rate at roll-in time
\(t_i\) is
\[
        \boxed{
        \operatorname{predicted\_flip}(t_i)
        =
        \frac{s(t_i)}{\sqrt{2\pi}}\,
        f_{\delta^*}(0+)
        +o(s(t_i)).
        }
        \tag{C.27}
\]
Equivalently,
\[
        \operatorname{predicted\_flip}(t_i)
        =
        \sqrt{\frac{2}{\pi}}\,
        s(t_i)\,A_W
        +o(s(t_i)).
        \tag{C.28}
\]

Given clean latents \(X_1,\ldots,X_N\), set
\[
        Y_n^\circ:=D_W(X_n),
        \qquad
        \delta_n:=
        \delta^*(X_n,Y_n^\circ).
\]
For a small bandwidth \(\varepsilon\) in the linear part of the empirical margin
cdf, estimate
\[
        \widehat f_{\delta^*}(0+;\varepsilon)
        :=
        \frac{1}{N\varepsilon}
        \sum_{n=1}^N
        \mathbf 1\{0\le\delta_n\le\varepsilon\},
        \qquad
        \widehat A_W(\varepsilon)
        :=
        \frac12\,\widehat f_{\delta^*}(0+;\varepsilon).
        \tag{C.29}
\]
Then the plotted prediction should be
\[
        \widehat{\operatorname{flip}}(t_i)
        =
        \frac{s(t_i)}{\sqrt{2\pi}}\,
        \widehat f_{\delta^*}(0+;\varepsilon)
        =
        \sqrt{\frac{2}{\pi}}\,
        s(t_i)\,\widehat A_W(\varepsilon).
        \tag{C.30}
\]
The bandwidth \(\varepsilon\) should be chosen by a stability plot of
\(\P[0\le\delta^*\le\varepsilon]/\varepsilon\) over a small-margin range; the theorem
predicts a plateau as \(\varepsilon\downarrow0\).

Equivalently, using an independent \(G\sim N(0,1)\),
\[
        \frac12\,\P\!\left[\delta^*(X,Y^\circ)\le s|G|\right]
        =
        \frac12\,f_{\delta^*}(0+)\,s\,\E|G|+o(s)
        =
        \frac{s}{\sqrt{2\pi}}\,f_{\delta^*}(0+)+o(s).
        \tag{C.31}
\]
The factor \(1/2\) in (C.31) corrects for the fact that the empirical margin cdf is
two-sided across each active facet, whereas the active-facet surface density
\(A_W\) counts each facet once.

\paragraph{Anisotropic residual correction (the realized constant).}
The predictor (C.27)--(C.31) assumes the terminal residual is the isotropic forward
noise, \(X-m_s^\star(Z_s)=s\,\varepsilon\) with \(\varepsilon\sim N(0,I)\), where
\(m_s^\star=\E[X\mid Z_s]\) is the exact posterior mean. A learned \(K\)-step ODE produces a
terminal map \(\widetilde m_{K,s}\) and a \emph{structured} residual
\(R_s:=X-\widetilde m_{K,s}(Z_s)=s\,\Xi_s\) whose normalized law \(\Xi_s\) need not be
isotropic. Repeating the one-dimensional Laplace expansion of the active-facet tube with
\(\Xi_s\) in place of \(\varepsilon\) replaces the universal normal moment
\(\E|G|=\sqrt{2/\pi}\) by the realized active-normal moment. Flipping is a \emph{one-sided}
boundary crossing, so with unit normal \(\widehat n_{ij}=(w_i-w_j)/\|w_i-w_j\|\) oriented from
cell \(j\) into cell \(i\) the leading flip mass uses the one-sided positive-part moments
\[
        \mu_{ij}(x)=\mu_{ij}^{+}(x)+\mu_{ij}^{-}(x),\qquad
        \mu_{ij}^{\pm}(x)=\lim_{\eta\downarrow0}\E\!\left[\bigl(\pm\langle\Xi_s,\widehat n_{ij}\rangle\bigr)_+\,\middle|\,X=x\pm\eta\widehat n_{ij}\right],
\]
\[
        \P(\widehat Y_s\ne Y^\circ)
        =
        s\sum_{i<j}\int_{F_{ij}}\rho_X(x)\,\mu_{ij}(x)\,d\mathcal H^{d-1}(x)+o(s)
        =
        \sqrt{\tfrac{2}{\pi}}\,s\,A_W\,c_{\rm aniso}+o(s),
        \tag{C.32}
\]
\[
        c_{\rm aniso}
        :=
        \frac{\sum_{i<j}\int_{F_{ij}}\rho_X\,\mu_{ij}\,d\mathcal H^{d-1}}
             {\sqrt{2/\pi}\,\sum_{i<j}\int_{F_{ij}}\rho_X\,d\mathcal H^{d-1}}.
        \tag{C.33}
\]
When the residual law is continuous across the facet with no one-sided sign asymmetry, this
reduces to the symmetric-trace form \(\mu_{ij}(x)=\E[\,|\langle\Xi_s,\widehat n_{ij}\rangle|\mid X=x]\).
The isotropic predictor (C.28) is the further special case \(\Xi_s\Rightarrow N(0,I)\), for which
\(\mu_{ij}\equiv\sqrt{2/\pi}\) and \(c_{\rm aniso}=1\); equivalently the effective normal
scale is \(s_{\rm eff}=c_{\rm aniso}\,s\). Here \(c_{\rm aniso}\) is the active-facet
weighted normal \(L^1\) scale of the realized residual, \emph{not} its total norm and
\emph{not} the row-space energy fraction.

\paragraph{Direct estimator of \(c_{\rm aniso}\).}
With signed boundary distance \(g_{ij}(x)=(\ell_i(x)-\ell_j(x))/\|w_i-w_j\|\) where
\(\ell_k=w_k^\top x\), the indicator \(a_{ij}(x)\) that \(i,j\) are the two largest logits at
\(x\), and a kernel \(K_\delta\),
\[
        \widehat c_{\rm aniso}(\delta)
        =
        \frac{\sum_{n}\sum_{i<j} a_{ij}(X_n)\,K_\delta(g_{ij}(X_n))\,
              \bigl|\langle R_{s,n},\widehat n_{ij}\rangle\bigr|/s}
             {\sqrt{2/\pi}\,\sum_{n}\sum_{i<j} a_{ij}(X_n)\,K_\delta(g_{ij}(X_n))}.
        \tag{C.34}
\]
Equation~(C.34) is the \(s\downarrow0\) asymptotic estimator; at the finite effective noise of the
roll-in (\(s\approx2.6\)--\(5.2\)) the first-moment identity is not yet tight, so we report the
empirical anisotropy as the observed/predicted ratio
\(c_{\rm aniso}=\mathrm{obs}/(\sqrt{2/\pi}\,s\,A_W)\approx2.3\), whose denominator uses the boundary
density \(A_W=f_{\delta^*}(0^+)/2\) estimated \emph{independently} from the clean margin distribution
(C.29)--(C.30): the factor is a ratio of independent measurements, not a fitted intercept.
The collapse of the \(K=4,8,16\) flip curves as functions of \(s(t)\) is
exactly the statement that the normalized active-normal residual law
\(\langle R_{K,t},\widehat n_{ij}\rangle/s(t)\) is approximately \(K\)-invariant: changing
\(K\) changes the scalar scale \(s(t)\), not the anisotropic shape, so the three curves share
the single multiplier \(c_{\rm aniso}\).

\paragraph{The flip multiplier is not DABI.}
The constant \(c_{\rm aniso}\) is a boundary-local, clipped, \(0/1\) crossing functional: to
first order \(\P(\mathrm{flip})\) depends on \(\E[(\langle R_s,\widehat n_{ij}\rangle)_+]\),
which saturates once the prediction has crossed the facet. The decoder-amplification index
\(\DABI\) is a different functional of the same residual, the cross-entropy increase
\(\E[\ell_{\rm CE}(X-R_s,Y)-\ell_{\rm CE}(X,Y)]\) relative to a norm-matched isotropic
perturbation. It is not clipped at the first crossing (wrong-side logit gap keeps growing
after a flip), and the linear CE term of a symmetric isotropic baseline cancels in
expectation, whereas a residual aligned with the CE gradient produces a first-order increase.
Hence \(\DABI\) (\(45.7\times\) in CE on ELF) can be far larger than the flip multiplier
\(c_{\rm aniso}\) (\(\approx2.3\)): \(\DABI\) certifies that the terminal residual is highly
structured, but the flip-rate tube constant is \(c_{\rm aniso}\), not \(\DABI\).

\paragraph{Scope beyond the ideal assumptions.}
The clean coefficient \(\sqrt{2/\pi}\,A_W\) uses exact clean decoding, smooth density, regular
boundaries, and exact posterior-mean denoising; the first-order law itself is more robust.
\emph{(i) Label noise} is already handled in (C.23)--(C.24): it contributes an
\(s\)-independent floor \(\eta_0\) and can reweight the slope through boundary label gaps, but
does not create the linear-in-\(s\) term; under a one-sided boundary gap
\(\pi_{D_W(x)}(x)-\pi_j(x)\ge\kappa>0\) the noisy slope is at least \(\kappa\) times the clean
slope. \emph{(ii) Non-smooth density.} Smoothness can be weakened to a weighted
finite-perimeter condition: the \(O(s)\) law holds with \(F_{ij}\) replaced by the reduced
boundary \(\partial^*\Omega_i\cap\partial^*\Omega_j\) and \(\rho_X\) by its approximate trace,
provided \(\sum_{i<j}\int \rho_X^*\,\E|\langle\Xi,\widehat n_{ij}\rangle|\,d\mathcal H^{d-1}<\infty\);
heavy tails away from the boundary are harmless. The coefficient changes, and the rate can be
\(o(s)\) if the trace vanishes on the boundary, or have a different exponent under a
non-integrable boundary singularity. \emph{(iii) Singular strata.} Codimension-\(\ge2\)
nonsmooth strata and triple intersections have \(s\)-tube volume \(O(s^2)\), negligible against
the \(O(s)\) codimension-one facets; only \emph{positive-mass} tie sets genuinely break the law
and require an explicit tie-break model. \emph{(iv) Approximate denoiser.} The proof needs only
first-order active-normal residual moments, not the exact posterior mean: if \(R_s/s\) has
stable active-facet normal \(L^1\) limits, (C.32) applies with those limits. Writing
\(e_s=m_s-\E[X\mid Z_s]\), the posterior-mean approximation is negligible when
\(\sum_{i<j}\int_{F_{ij}}\rho_X\,\E[|\langle e_s,\widehat n_{ij}\rangle|/s]\to0\); a finite
nonzero limit merely modifies \(c_{\rm aniso}\), while an \(\omega(s)\) normal error or a
nonvanishing normal bias at \(s=0\) breaks the linear law. The observed collapse of the
\(K=4,8,16\) curves is direct evidence for the weaker condition the theorem actually needs:
a stable \(O(s(t))\) active-normal terminal residual.

\emph{Summary of scope.} The lower-bound mechanism is stable under label noise with a positive
boundary label gap, weighted finite-perimeter boundaries, codimension-\(\ge2\) singular strata,
and approximate denoisers with \(O(s)\) active-normal error; these change the first-order
coefficient (through label weights, density traces, or anisotropic moments) but preserve the
linear-in-\(s\) law. The theorem no longer controls the leading order when \(P_X\) has atoms or positive mass on tie
boundaries, when the boundary Minkowski content is not linear in \(s\), or when the learned
terminal residual has a normal component larger than \(O(s)\).

\paragraph{Observable remainder form (finite-sample readout).}
The asymptotic coefficient in (C.27)--(C.31) is controlled by a single observable: the failure of
the empirical small-margin cdf to be linear at the origin. With
\(F_{\delta^*}(r):=\P[0\le\delta^*(X,Y^\circ)\le r]\) and the linearization remainder
\[
        R_{\delta^*}(r):=F_{\delta^*}(r)-2A_W\,r,
        \tag{C.35}
\]
the isotropic active-facet expansion (C.25)--(C.31) gives the exact one-step identity
\(\P(D_W(m_s^\star(Z_s))\ne Y^\circ)=\tfrac12\,\E[F_{\delta^*}(s|G|)]+o(s)\) with \(G\sim N(0,1)\),
and hence the explicit remainder bound
\[
        \Bigl|\,\P\!\left(D_W(m_s^\star(Z_s))\ne Y^\circ\right)-\sqrt{\tfrac2\pi}\,sA_W\,\Bigr|
        \le\tfrac12\,\E\!\left[\,\bigl|R_{\delta^*}(s|G|)\bigr|\,\right]+o(s).
        \tag{C.36}
\]
The approximation error is thus \emph{exactly} the non-linearity of the measured margin cdf near
zero: under (R2) \(R_{\delta^*}(r)=o(r)\), and if the boundary tube content is \(O(r^2)\) the
remainder in (C.36) is \(O(s^2)\). In this sense the prediction is finite-sample observable (the
estimator \(\widehat A_W(\varepsilon)\) and the empirical curvature of \(\widehat F_{\delta^*}\)
bound their own error), and we report bootstrap intervals over sampled positions as a descriptive,
not i.i.d., summary, since text positions are correlated (we deliberately avoid DKW/Hoeffding
rates, which would import an i.i.d.\ assumption the analysis does not need). For raw labels the
uniform bound \(|\P(\widehat Y_s\ne Y^{\rm raw})-\P(\widehat Y_s\ne Y^\circ)|\le\eta_0\) with
\(\eta_0:=\P(Y^{\rm raw}\ne Y^\circ)\) is sharp without assumptions on where clean-decoding errors lie.

\paragraph{Approximate terminal denoisers (learned generators).}
Theorem~\ref{thm:noncommit} is stated for the exact posterior mean \(m_s^\star=\E[X\mid Z_s]\); the
following isolates the weaker hypothesis a \emph{learned} terminal denoiser must satisfy, so the
law is not tied to training convergence.
\begin{proposition}[approximate-denoiser tube law]\label{prop:approx-denoiser}
Let \(\widetilde m_s\) be any deterministic terminal denoiser, \(\widetilde R_s:=X-\widetilde m_s(Z_s)\),
and suppose the one-sided active-normal traces
\(\widetilde\mu_{ij}^{\pm}(x)=\lim_{s\downarrow0}\lim_{\eta\downarrow0}
\E[(\pm\langle\widetilde R_s,\widehat n_{ij}\rangle/s)_+\mid X=x\pm\eta\widehat n_{ij}]\) exist in
\(L^1(\rho_X\,d\mathcal H^{d-1})\) on the active facets. Then
\[
        \P\!\left(D_W(\widetilde m_s(Z_s))\ne Y^\circ\right)
        =s\sum_{i<j}\int_{F_{ij}}\rho_X(x)\bigl[\widetilde\mu_{ij}^{+}(x)+\widetilde\mu_{ij}^{-}(x)\bigr]\,d\mathcal H^{d-1}(x)+o(s).
\]
If \(e_s:=\widetilde m_s-m_s^\star\) has vanishing active-normal projection,
\(\sum_{i<j}\int_{F_{ij}}\rho_X\,\E[|\langle e_s,\widehat n_{ij}\rangle|/s\mid X\to x]\,d\mathcal H^{d-1}\to0\),
the exact coefficient \(\sqrt{2/\pi}\,A_W\) is unchanged; a finite nonzero limit leaves the exponent
at one and rescales only the first-order constant (the \(c_{\rm aniso}\) of (C.33)); an \(\omega(s)\)
normal error breaks the linear law.
\end{proposition}
The proof is the Laplace expansion already used for (C.32), with \(\widetilde\mu^\pm\) replacing the
isotropic moment. The roll-in collapse for \(K=4,8,16\) (Figure~\ref{fig:rollin}) is therefore not
invoked as a proof that training drives \(\widetilde m_{K,s}\to m_s^\star\); it is direct evidence
for the weaker hypothesis Proposition~\ref{prop:approx-denoiser} needs: that the normalized
active-normal residual \(s(t)^{-1}\langle\widetilde R_{K,t},\widehat n_{ij}\rangle\) has a stable
active-facet \(L^1\) law across \(K\). Convergence of the learned denoiser under realistic
optimization, finite data, or distillation is outside our scope: the population law is
posterior-mean, and a learned system is covered exactly when its terminal active-normal residual
satisfies the displayed condition.

\paragraph{Boundary-concentrated label noise.}
The \(s\)-independent floor \(\eta_0=\P(Y^{\rm raw}\ne Y^\circ)\) bounds the gap between raw-label and
clean-readout error uniformly, but it does not describe how label noise \emph{at the boundary} reshapes
the slope. Let \(\beta_{ij}(x)\in[0,1]\) be the local cross-facet probability that the raw label
disagrees across \(F_{ij}\). Under symmetric active-normal residual traces, the raw-label error keeps the
linear-in-\(s\) form with a reweighted slope,
\[
        \mathcal E_s^{\rm raw}
        =\eta_0+\sqrt{\tfrac2\pi}\,s\sum_{i<j}\int_{F_{ij}}\rho_X(x)\,\bigl(1-2\beta_{ij}(x)\bigr)\,d\mathcal H^{d-1}(x)+o(s).
\]
Boundary label noise thus attenuates the raw-label slope by the measurable factor \(1-2\beta_{ij}\), and
nulls it where labels are locally \(50/50\) (\(\beta_{ij}=1/2\)); the clean-readout flip slope
\(\sqrt{2/\pi}\,A_W\) is unchanged. If the raw-label law has no finite boundary trace, or the mislabelled
\(r\)-tube has non-linear Minkowski content, the raw-label curve can take a different exponent, while the
self-consistency theorem for \(Y^\circ=D_W(X)\) still applies.

\paragraph{Non-smooth readouts.}
The theorem does not require \(D_W\) to be smooth. For a linear argmax readout, \(D_W\) is discontinuous
exactly on the codimension-one facets \(F_{ij}\), and the tube law is the first-order measure of those
discontinuity surfaces: regularity is asked of the boundary \emph{measure}, not of \(D_W\). The same holds
for nonlinear logits \(h_k\) whose active surfaces \(F_{ij}=\{h_i=h_j=\max_k h_k,\ \|\nabla(h_i-h_j)\|>0\}\)
are countably \(C^1\) (or form a locally finite-perimeter partition with finite weighted Minkowski
content), on replacing the unit normal \(\widehat n_{ij}=(w_i-w_j)/\|w_i-w_j\|\) by
\(\widehat n_{ij}(x)=\nabla(h_i-h_j)(x)/\|\nabla(h_i-h_j)(x)\|\). Only positive-mass ties, non-integrable
boundary traces, or codimension-one tangential contact \(\nabla(h_i-h_j)=0\) fall outside the theorem.

\section{The Bridge: Readout--Transport Interface Product Law (Theorem~\ref{thm:bridge})}
\label{app:bridge}

Theorem~\ref{thm:noncommit} (overlapping regime) and Theorem~\ref{thm:gamma} (separated
regime) are not the same scalar functional. This appendix shows they are two different
\emph{contractions of one geometric object}: the source-space readout interface of the
composite classifier
\[
        F = D_W \circ T : \R^n \to \{1,\ldots,V\},
        \qquad
        X = T(S),\quad S\sim\sigma.
\]
The readout boundary mass $A_W$ of Theorem~\ref{thm:noncommit} is this interface measured
with weight $1$; the separated interface energy $\Jint$ of Theorem~\ref{thm:gamma} is a
barrier-weighted normal-stretch moment of the same interface. A coarea identity relates
them, with an unavoidable \emph{readout-normal stretch} $\eta_{ij}(x)=\norm{DT(x)^\top n_{ij}}$.

\begin{figure}[t]
\centering
\begin{tikzpicture}[>=Latex]
\draw[gray!55, rounded corners] (0,0) rectangle (4.2,3.3);
\node[gray!75] at (2.1,3.05) {\small (a) overlapping: tube mass $A_W$};
\fill[blue!10] (0.5,0.25) -- (3.9,2.55) -- (3.9,2.95) -- (0.5,0.65) -- cycle;
\draw[thick] (0.5,0.45) -- (3.9,2.75);
\node at (3.45,1.05) {\small $F_{ij}$};
\draw[densely dashed,red!65] (1.85,1.45) ellipse (1.0 and 0.62);
\node[red!75] at (0.85,2.05) {\small $\rho_X$};
\fill (1.95,1.6) circle (1.3pt);
\node at (1.62,1.8) {\small $m_t$};
\draw[->,thick,blue!65!black] (1.95,1.6) -- (2.6,0.93);
\node[blue!65!black] at (2.95,1.45) {\small $\widehat n_{ij}$};
\node at (2.1,-0.45) {\scriptsize $\Pr(\mathrm{flip})=\sqrt{2/\pi}\,s\,A_W,\quad A_W=\!\sum_{i<j}\!\int_{F_{ij}}\!\rho_X\,d\mathcal H^{d-1}$};
\begin{scope}[shift={(6.2,0)}]
\draw[gray!55, rounded corners] (0,0) rectangle (4.2,3.3);
\node[gray!75] at (2.1,3.05) {\small (b) separated: barrier energy $\Jint$};
\fill[orange!16] (1.93,0.35) rectangle (2.27,2.75);
\draw[orange!70!black, very thick] (2.1,0.35) -- (2.1,2.75);
\fill[red!55] (1.05,1.5) circle (4pt);
\fill[red!55] (3.15,1.5) circle (4pt);
\node at (0.9,0.95) {\small $c_i$};
\node at (3.3,0.95) {\small $c_j$};
\node[orange!70!black] at (2.1,2.45) {\small $\kappa_{ij}$};
\node at (2.1,-0.45) {\scriptsize $\Jint=\!\sum_{i<j}\!\kappa_{ij}\,P_\sigma(B_i,B_j),\quad P_\sigma=\!\int_{\widetilde\Sigma_{ij}}\!\rho\,d\mathcal H^{n-1}$};
\end{scope}
\end{tikzpicture}
\caption{\textbf{One interface, two contractions} (the bridge, Theorem~\ref{thm:bridge}).
The \emph{same} source interface $\widetilde\Sigma_{ij}=\partial^*B_i\cap\partial^*B_j$ is read two ways.
\textbf{(a)} Theorem~\ref{thm:noncommit} weights it by $1$: a posterior-mean point $m_t$ in the $O(s)$
tube around an active readout facet $F_{ij}$ flips to the wrong token with rate $\sqrt{2/\pi}\,sA_W$.
\textbf{(b)} Theorem~\ref{thm:gamma} weights the same interface by the barrier $\kappa_{ij}$ and the
normal stretch, giving the separation energy $\Jint$. The coarea identity \eqref{eq:bridge-coarea}
links the two through $\eta_{ij}=\norm{DT^\top n_{ij}}$; no single readout-blind scalar reproduces
both (Proposition~\ref{prop:dichotomy}).}
\label{fig:interface}
\end{figure}

\paragraph{Setup and notation.}
Let $\sigma=\rho\,dx$ on $\R^n$ with $\rho\in C^1$, $\rho>0$. Let $T:\R^n\to\R^d$ be $C^1$ in
a neighborhood of the source-space readout interfaces. Write the argmax cells
$\Omega_i=\{z:D_W(z)=i\}$, the active co-maximal facets $F_{ij}=\partial\Omega_i\cap\partial\Omega_j$,
and the unit target normals $n_{ij}=(w_i-w_j)/\norm{w_i-w_j}$ (oriented from $\Omega_j$ into
$\Omega_i$). Define the composite phases and source interfaces
\[
        B_i := T^{-1}(\Omega_i),
        \qquad
        \widetilde\Sigma_{ij} := \partial^* B_i \cap \partial^* B_j \cap T^{-1}(F_{ij}),
        \qquad
        \eta_{ij}(x) := \norm{DT(x)^\top n_{ij}}.
\]
Let $A_W=\sum_{i<j}\int_{F_{ij}}\rho_X\,d\mathcal H^{d-1}$ be the readout boundary mass of
Theorem~\ref{thm:noncommit} (here $\rho_X$ is the density of $T\push\sigma$ near the active
boundary), and let $\kappa_{\max}(C)=\max_{i<j}\kappa_p(i,j)$.

\paragraph{Assumptions.}
\begin{enumerate}
\item[(B1)] \textbf{Alignment and finite-energy partition.} The vocabulary is identified with the codebook ($V=M$ after relabeling, each readout cell $\Omega_i$ carrying atom $c_i$); the composite partition $B=(B_i)_{i=1}^M$ with $B_i=T^{-1}(\Omega_i)$ is a Caccioppoli partition with $\sigma(B_i)=\pi_i>0$ and finite weighted interface energy $E_{p,\sigma}(B;C)<\infty$.
\item[(B2)] \textbf{Nondegenerate pushforward and Theorem~\ref{thm:noncommit} regularity.} $T\in C^1$ with $\operatorname{rank} DT(x)=d$ for $\sigma$-a.e.\ $x$ in a neighborhood of the active source interfaces (in particular $n\ge d$), so that $T\push\sigma=\rho_X\,dz$ there with $\rho_X$ the density appearing in Theorem~\ref{thm:noncommit}; moreover $X=T(S)$ and $W$ satisfy all hypotheses of Theorem~\ref{thm:noncommit}, $\rho_X$ has finite normal traces on the active facets, $\int_{\widetilde\Sigma_{ij}}\rho/\eta_{ij}\,d\mathcal H^{n-1}<\infty$, and for each active pair the normalized source mass of the active normal slab of half-width $r$ over $\widetilde\Sigma_{ij}$ converges as $r\downarrow0$ to $\int_{\widetilde\Sigma_{ij}}\rho/\eta_{ij}\,d\mathcal H^{n-1}$ (every facet point is a Lebesgue point of the sliced density).
\item[(B3)] \textbf{Transversality.} $0<\eta_{ij}(x)\le L_\perp<\infty$ for $\mathcal H^{n-1}$-a.e.\ $x\in\widetilde\Sigma_{ij}$, where
\[
        L_\perp = L_\perp(T,W) := \max_{i<j}\esssup_{x\in\widetilde\Sigma_{ij}}\norm{DT(x)^\top n_{ij}}
        \;\le\; \Lip(T);
\]
for the $K$-step form each $\Lip(T_k)<\infty$.
\item[(B4)] \textbf{Interface exhaustion.} For every active pair, $\partial^*B_i\cap\partial^*B_j\subseteq T^{-1}(F_{ij})$ up to $\mathcal H^{n-1}$-null sets (no off-facet or inactive-tie interface mass), and the active triple skeleton carries no $\mathcal H^{n-1}$ interface mass. The posterior-mean/MAP switching agreement of Theorem~\ref{thm:noncommit} holds for $X$ (Appendix~\ref{app:noncommit}, equations (C.12)/(C.21)).
\end{enumerate}

\begin{theorem}[readout--transport interface product law]\label{thm:bridge}
Under \textup{(B1)--(B4)}:
\begin{enumerate}
\item[\textup{(i)}] \textbf{Coarea transfer.} The readout boundary mass is the
normal-speed-discounted source perimeter,
\begin{equation}\label{eq:bridge-coarea}
        A_W
        =
        \sum_{i<j}\int_{\widetilde\Sigma_{ij}}
        \frac{\rho(x)}{\norm{DT(x)^\top n_{ij}}}\,d\mathcal H^{n-1}(x)
        =
        \sum_{i<j}\mathfrak m_{ij}^{T,W}(\widetilde\Sigma_{ij}),
\end{equation}
where $d\mathfrak m_{ij}^{T,W}:=\rho\,\norm{DT^\top n_{ij}}^{-1}\,d\mathcal H^{n-1}$ on
$\widetilde\Sigma_{ij}$ is the \emph{normal-stretch readout interface measure}.
\item[\textup{(ii)}] \textbf{Same interface, two moments.} The source perimeter and the
separated interface energy are weighted moments of the same measure,
\begin{equation}\label{eq:bridge-moments}
\begin{aligned}
        P_\sigma(B_i,B_j)&=\int_{\widetilde\Sigma_{ij}}\rho\,d\mathcal H^{n-1}
        =\int_{\widetilde\Sigma_{ij}}\norm{DT^\top n_{ij}}\,d\mathfrak m_{ij}^{T,W},\\
        E_{p,\sigma}(B;C)&=\sum_{i<j}\kappa_{ij}P_\sigma(B_i,B_j).
\end{aligned}
\end{equation}
Consequently
\begin{equation}\label{eq:bridge-AW-J}
        \Jint_{p,\sigma}(\pi;C)\;\le\;E_{p,\sigma}(B;C)\;\le\;\kappa_{\max}(C)\,L_\perp\,A_W,
        \qquad\text{i.e.}\qquad
        A_W\;\ge\;\frac{\Jint_{p,\sigma}(\pi;C)}{\kappa_{\max}(C)\,L_\perp}.
\end{equation}
\item[\textup{(iii)}] \textbf{Product law.} With $\widehat Y_s=D_W(m_s(X+s\eps))$, $Y=D_W(X)$,
Theorem~\ref{thm:noncommit} gives $\Prob(\widehat Y_s\ne Y)=\sqrt{2/\pi}\,s\,A_W+o(s)$, hence
\begin{equation}\label{eq:bridge-product}
        \liminf_{s\downarrow0}\frac{L_\perp}{s}\,\Prob(\widehat Y_s\ne Y)
        \;\ge\;
        \sqrt{\frac{2}{\pi}}\,\frac{\Jint_{p,\sigma}(\pi;C)}{\kappa_{\max}(C)}.
\end{equation}
For a $K$-step deterministic generator $T=T_K\circ\cdots\circ T_1$, $L_\perp\le\Lip(T)\le\prod_k\Lip(T_k)$, so
\begin{equation}\label{eq:bridge-Kstep}
        \liminf_{s\downarrow0}\frac{\prod_{k=1}^K\Lip(T_k)}{s}\,\Prob(\widehat Y_s\ne Y)
        \;\ge\;
        \sqrt{\frac{2}{\pi}}\,\frac{\Jint_{p,\sigma}(\pi;C)}{\kappa_{\max}(C)}.
\end{equation}
\end{enumerate}
\end{theorem}

\begin{proof}
\emph{Step 1 (coarea through the readout facet).}
Fix an active facet $F_{ij}$, written locally as the zero set of the affine
$\phi_{ij}(z)=n_{ij}^\top z-b_{ij}$, and set $g_{ij}(x):=\phi_{ij}(T(x))=n_{ij}^\top T(x)-b_{ij}$.
Then $\nabla g_{ij}=DT^\top n_{ij}$, so by (B3) $\abs{\nabla g_{ij}(x)}=\eta_{ij}(x)>0$ on $\widetilde\Sigma_{ij}$.
For $r>0$ let
\[
        U_{ij,r}:=\{\,y+t\,n_{ij}:\ y\in F_{ij},\ \dist(y,\partial F_{ij})>r,\ \abs t\le r\,\}\subset\R^d
\]
be the $d$-dimensional normal prism of half-width $r$ over the relative interior of $F_{ij}$ (away from the codimension-two skeleton); its preimage $T^{-1}(U_{ij,r})$ is the \emph{active} normal slab over $\widetilde\Sigma_{ij}$, a subset of $\{\abs{g_{ij}}\le r\}$ trimmed to the active facet. By the pushforward identity (B2),
\[
        \int_{U_{ij,r}}\rho_X\,dz
        \;=\;
        (T\push\sigma)(U_{ij,r})
        \;=\;
        \sigma\big(T^{-1}(U_{ij,r})\big)
        \;=\;
        \int_{T^{-1}(U_{ij,r})}\rho\,dx .
\]
Divide by $2r$ and let $r\downarrow0$. The target-side normal-trace tube formula yields $\int_{F_{ij}}\rho_X\,d\mathcal H^{d-1}$; the source-side active-prism average converges to $\int_{\widetilde\Sigma_{ij}}\rho/\eta_{ij}\,d\mathcal H^{n-1}$ (coarea for $g_{ij}$ on the trimmed prism, with $\abs{\nabla g_{ij}}=\eta_{ij}$, the convergence and finiteness being part of (B2)); the skeleton tube contributes $o(1)$ by (B4). Summing over active pairs gives \eqref{eq:bridge-coarea}.

\emph{Step 2 (compare to source perimeter).}
By (B4) the reduced interface $\partial^*B_i\cap\partial^*B_j$ agrees with $\widetilde\Sigma_{ij}$ up to $\mathcal H^{n-1}$-null sets, so
$P_\sigma(B_i,B_j)=\int_{\widetilde\Sigma_{ij}}\rho\,d\mathcal H^{n-1}
=\int_{\widetilde\Sigma_{ij}}\eta_{ij}\,d\mathfrak m_{ij}^{T,W}$, which is
\eqref{eq:bridge-moments}. Since $\eta_{ij}\le L_\perp$,
$\sum_{i<j}P_\sigma(B_i,B_j)\le L_\perp A_W$.

\emph{Step 3 (insert interface energy).}
$E_{p,\sigma}(B;C)=\sum_{i<j}\kappa_{ij}P_\sigma(B_i,B_j)\le\kappa_{\max}\sum_{i<j}P_\sigma(B_i,B_j)
\le\kappa_{\max}L_\perp A_W$. Because $\Jint$ is the infimum of $E_{p,\sigma}$ over Caccioppoli
partitions with masses $\pi$ (and $B$ is one such partition by (B1)), $\Jint\le E_{p,\sigma}(B;C)$,
giving \eqref{eq:bridge-AW-J}.

\emph{Step 4 (insert Theorem~\ref{thm:noncommit}).}
The active-facet tube law (Appendix~\ref{app:noncommit}) gives
$\Prob(\widehat Y_s\ne Y)=\sqrt{2/\pi}\,s\,A_W+o(s)$, with the MAP/posterior-mean agreement
supplied by (B4). Combining with \eqref{eq:bridge-AW-J} and dividing by $s/L_\perp$ gives
\eqref{eq:bridge-product}. \emph{Step 5} substitutes $L_\perp\le\prod_k\Lip(T_k)$ for a $K$-step
composition, yielding \eqref{eq:bridge-Kstep}.
\end{proof}

\paragraph{Reading the bridge.}
$A_W$ is \emph{not} the source perimeter: it is the source perimeter discounted pointwise by
the readout-normal speed $\eta_{ij}=\norm{DT^\top n_{ij}}$. Theorem~\ref{thm:noncommit} reads
this interface with weight $1$ (noise-to-error conversion); Theorem~\ref{thm:gamma} reads the
same interface with the barrier-weighted stretch $\kappa_{ij}\eta_{ij}$ (source-interface/stiffness
law). Equation~\eqref{eq:bridge-Kstep} is the rigorous form of the accuracy--depth--stiffness
tradeoff: to keep the terminal flip rate small, the deterministic stiffness budget
$\prod_k\Lip(T_k)$ must grow at least like $\Jint/\kappa_{\max}$, which by
Theorem~\ref{thm:dim} is $\gtrsim\sqrt{\log M}$.

\paragraph{Why the naive composition fails.}
A tempting shortcut identifies $A_W$ with $\sum_{i<j}P_\sigma(B_i,B_j)$ and $\Jint$ with $A_W$
up to constants. Both are wrong: the first omits the stretch factor $\eta_{ij}$, and the second
ignores that $\Jint$ carries the barrier weights $\kappa_{ij}$ and the codebook scale. The honest
statement is one interface measure, two contractions. This is made sharp by the following
impossibility result, which also explains why no single \emph{scalar} functional can serve.

\begin{proposition}[no $W$-blind scalar functional reproduces both regimes]\label{prop:dichotomy}
\textup{(a)} There is no functional $G(T)$ of the transport map alone (blind to the readout $W$)
that equals $A_W$ for all readouts. \textup{(b)} There is no functional $H(T,W)$ blind to the
codebook scale that equals $\Jint$ for all codebooks. Hence the two constants are genuinely
distinct scalar contractions of the shared interface, not one functional.
\end{proposition}

\begin{proof}
\emph{(a)} Take $n=d=2$, $T=\mathrm{Id}$, $X\sim\N(0,\diag(a^2,b^2))$ with $a\ne b$, and the
two two-token readouts with decision boundaries $\Sigma^{(1)}=\{z_1=0\}$ and
$\Sigma^{(2)}=\{z_2=0\}$. Both give token masses $(\tfrac12,\tfrac12)$, so $T$, $\sigma$, $\pi$
are identical. But
\[
        A_{W^{(1)}}=\int_{\{z_1=0\}}\rho_X\,d\mathcal H^1=\frac{1}{\sqrt{2\pi}\,a},
        \qquad
        A_{W^{(2)}}=\int_{\{z_2=0\}}\rho_X\,d\mathcal H^1=\frac{1}{\sqrt{2\pi}\,b},
\]
which differ since $a\ne b$. No $W$-blind $G(T)$ can equal both.
\emph{(b)} Fix $T,W,\pi$ and any source/codebook with $\Jint_{p,\sigma}(\pi;C)>0$ (e.g.\ $M\ge2$ distinct atoms, $p\ge1$, $\pi_i>0$, and a source whose weighted isoperimetric profile is bounded below). Scale the codebook by $r>0$: with $V(y)=\dist(y,rC)^p$ and the path reparametrization $\xi\mapsto r\xi$, $\kappa_p^{(r)}(i,j)=r^{p+1}\kappa_p^{(1)}(i,j)$, hence
$\Jint_{p,\sigma}(\pi;rC)=r^{p+1}\Jint_{p,\sigma}(\pi;C)$, while $A_W$ is unchanged. No
codebook-scale-blind $H(T,W)$ can equal both.
\end{proof}

\paragraph{Scope.}
Theorem~\ref{thm:bridge} requires a single composite $T$ that is simultaneously (B1) aligned to
the readout and (B2)--(B3) smooth and transverse near the active boundary with pushforward
density $\rho_X$. A near-optimal separated transport for Theorem~\ref{thm:gamma} need not satisfy
this (atomic outputs destroy the smooth $\rho_X$ of (B2)), and a smooth overlapping-regime $T$
satisfying Theorem~\ref{thm:noncommit} need not minimize $\Jint$. The bridge is therefore an
\emph{inequality} \eqref{eq:bridge-AW-J} on any aligned transversal composite, tight when the
stretch $\eta_{ij}$ is constant along the active interface and the barriers $\kappa_{ij}$ are
equal; it is not an identity between $A_W$ and $\Jint$.

\paragraph{Approximate alignment (relaxing (B4)).}
Interface exhaustion (B4) is the premise most likely to fail in practice: a learned $T$ may carry
interface mass that does not map to an \emph{active} readout facet. Suppose this off-facet
interface has finite barrier-weighted size
\[
        E_{\mathrm{off}}
        :=
        \sum_{i<j}\kappa_{ij}
        \int_{(\partial^*B_i\cap\partial^*B_j)\setminus T^{-1}(F_{ij})}
        \rho\,d\mathcal H^{n-1}.
\]
Then Steps~2--3 of the proof bound only the \emph{aligned} interface by $\kappa_{\max}L_\perp A_W$,
so $E_{p,\sigma}(B;C)\le\kappa_{\max}(C)\,L_\perp A_W+E_{\mathrm{off}}$ and, since
$\Jint_{p,\sigma}(\pi;C)\le E_{p,\sigma}(B;C)$,
\[
        \liminf_{s\downarrow0}\frac{L_\perp}{s}\,\Prob(\widehat Y_s\ne Y)
        \;\ge\;
        \sqrt{\frac{2}{\pi}}\,
        \frac{\bigl(\Jint_{p,\sigma}(\pi;C)-E_{\mathrm{off}}\bigr)_+}{\kappa_{\max}(C)} .
\]
Loss of alignment therefore weakens the product law only through the explicitly measurable
off-facet energy $E_{\mathrm{off}}$, reducing to \eqref{eq:bridge-product} when
$E_{\mathrm{off}}=0$. Likewise a non-constant stretch enters only through the single scalar
$L_\perp=\esssup_{i,j,x}\norm{DT^\top n_{ij}}$ already isolated in (B3). We therefore apply
(B1)--(B4) as a \emph{conditional} product law whose failure mode is quantified rather than
assumed away: on ELF the roll-in diagnostics (Appendix~\ref{app:noncommit}) are read as consistent
with bounded transversality and small off-facet energy at the probed resolution, not as a global
certificate that (B1)--(B4) hold everywhere.

\paragraph{What is measurable on an overlapping codec.}
On a real overlapping text autoencoder, $A_W$ and $c_{\rm aniso}$ are intrinsic readout quantities, but
$\Jint_{p,\sigma}(\pi;C)$ and $\kappa_{\max}(C)$ are not: they are defined only after one fixes a separated
atomic codebook $C$, a barrier $V(y)=\dist(y,C)^p$, and prescribed masses $\pi$. We therefore do not
estimate $\Jint/\kappa_{\max}$ on ELF and do not fit a numerical right-hand side of
\eqref{eq:bridge-product} to an overlapping codec; a numerical $\Jint/\kappa_{\max}$ is meaningful only in
the separated synthetic setting (Appendix~\ref{app:synth}), where $C$ and $V$ are specified. What is
intrinsic and measurable is the overlapping side together with the conditioning of the bridge map:
$\widehat A_W=\tfrac12\widehat f_{\delta^*}(0^+)\approx0.03$ on ELF, the realized $c_{\rm aniso}$, and, in
principle, the active-facet transversality $\eta_{ij}=\norm{DT^\top\widehat n_{ij}}$ (a vector--Jacobian
product of the few-step map), of which $L_\perp$ is the essential supremum. Alignment can fail in two
distinct ways: part of the source interface maps to non-active facets (the off-facet term
$E_{\mathrm{off}}$ above, itself defined only against a separated abstraction), or the interface is
tangential, so $\eta_{ij}=0$ and $1/\eta_{ij}$ is non-integrable (a failure of transversality, not an
$E_{\mathrm{off}}$ correction). Equation~\eqref{eq:bridge-product} is thus a conditional inequality on any
aligned separated abstraction, not an empirical identity to be fitted on an overlapping system.

\section{Dynamics: Propositions and Empirical Confirmation}\label{app:dynamics}

With $C \sim \mu$, $Z \sim \N(0, I_d)$, $X_t = (1-t)Z + tC$, the canonical independent-coupling velocity is $v_t(x) = \E[C - Z \mid X_t = x] = (m_t(x) - x)/(1-t)$ with posterior $w_i \propto \exp(\lambda_t \inner{x}{c_i})$ for equal radii, $\lambda_t = t/(1-t)^2$.

\begin{proposition}[critical acceleration]\label{prop:accel}
For a $\Delta$-separated codebook in dimension $d = O(\log M)$, in the critical-SNR window where $\sigma_t/t \in [A\Delta/\sqrt{\log M}, 2A\Delta/\sqrt{\log M}]$ ($A \ge 2$), $\mmse(C \mid X_t) \ge c\Delta^2$, hence $\E \tr \Cov(C - Z \mid X_t) \gtrsim \log M / t^2$, and $\int_{\mathcal{I}_{\mathrm{crit}}} \norm{a_t}_{L^2(p_t)}\,dt \gtrsim \sqrt{\log M}.$
\end{proposition}

\begin{proposition}[Euler local barrier]\label{prop:eulerlocal}
$W_2((I + hv_t)\push p_t, p_{t+h}) \gtrsim h^2 \log M/\Delta$ on the critical window.
\end{proposition}

\begin{proposition}[two-step stiffness escape]\label{prop:twostep}
For a centered, transitive, equal-radius codebook, the two-step scheme $0 \to 1 - h \to 1$ reaches the target as $h \downarrow 0$ with $\Lip \gtrsim \delta^2/h$ (divergent off-path stiffness).
\end{proposition}

\begin{corollary}[accuracy--depth--stiffness tradeoff]\label{cor:tradeoff}
No terminal lower bound on $W_2(\hat{p}_K, \mu)$ depends on $(M, \Delta, K)$ alone.
With $F = F_K \circ \cdots \circ F_1$, $\prod_k \Lip(F_k) \gtrsim \Delta^{p+1}\sqrt{\log M}/\eta^p$, so a per-step cap $S$ forces $K \ge \log(c\Delta^{p+1}\sqrt{\log M}/\eta^p)/\log S$.
\end{corollary}

\section{Synthetic Verification of Theorems~\ref{thm:gamma}--\ref{thm:dim}}\label{app:synth}

\emph{Interface $\Gamma$-limit (Theorem~\ref{thm:gamma}).}
For an explicit equal-mass point-well construction in $\R^1$, the optimal $W_2^2$ follows the $1/\Lambda$ law: the log--log slope of $W_2^2$ against the achieved Lipschitz constant is $-0.97, -0.95, -0.93, -0.90$ for $M = 4, 8, 16, 32$ (Figure~\ref{fig:synth}a), approaching the predicted $-1$ as $\Lambda$ grows.
At the converged scale the renormalized ratio $\Lambda \cdot W_2^2 / [3(2\ln 2 - 1) \Jint]$ matches the predicted logistic constant $3(2\ln 2 - 1) = 1.1589$ to within a few percent.

\emph{Dimension phase diagram (Theorem~\ref{thm:dim}).}
At $n = 256$, the interface profile $\mathcal{P}_{n,M}$ tracks $\sqrt{\log M}$; at fixed $n \in \{1, 2, 3, 4\}$ the $M^{1/n}$ growth exponents are $1.14, 0.57, 0.38, 0.285$, matching the predicted $1/n$.

\begin{figure}[h]
\centering
\includegraphics[width=\textwidth]{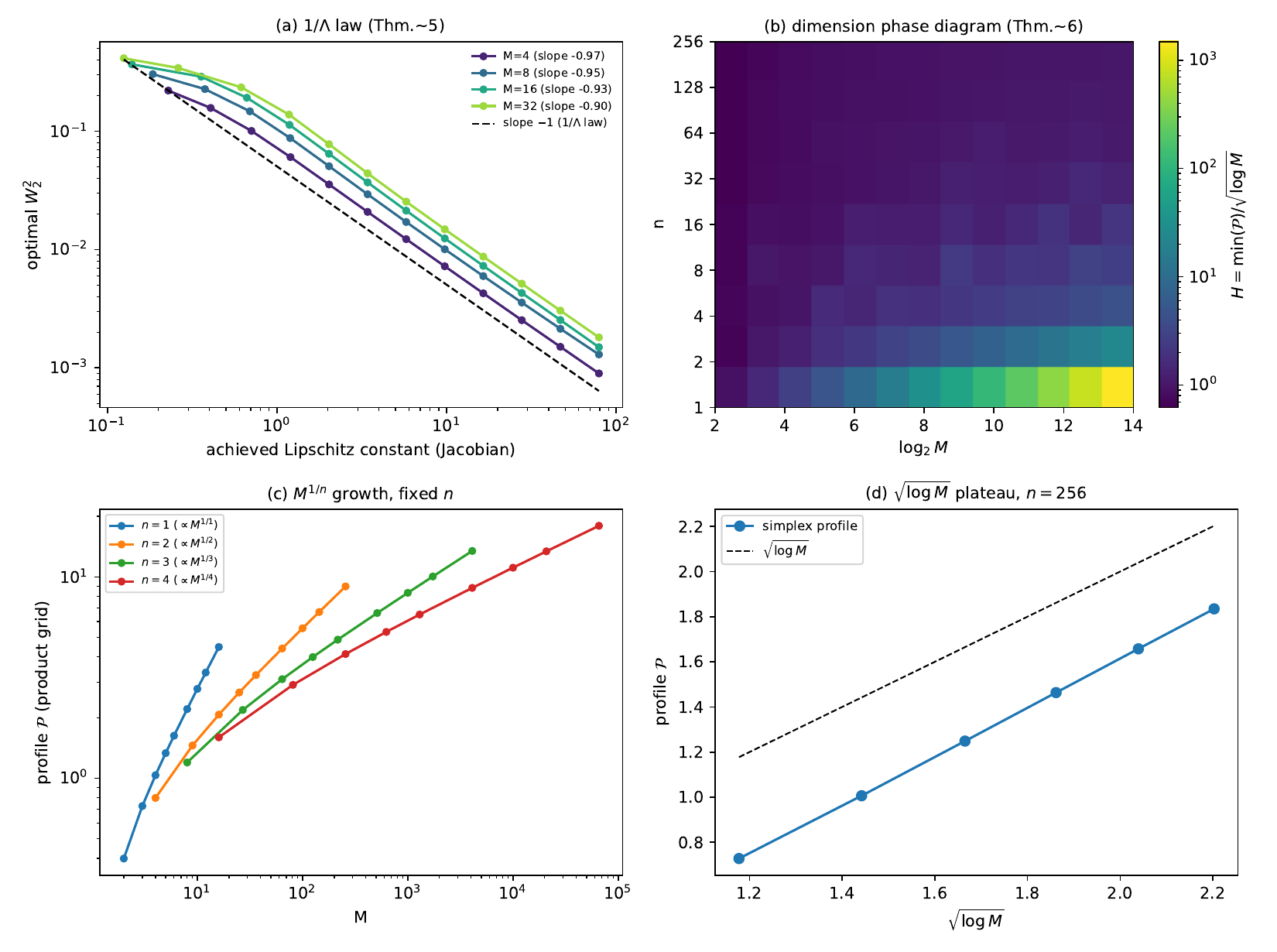}
\caption{\textbf{Synthetic verification of the separated-regime theorems.} \textbf{(a)} The analytic equal-mass construction realizes the $1/\Lambda$ law of Theorem~\ref{thm:gamma} (log--log slope $-0.90$ to $-0.97$, approaching $-1$). \textbf{(b)} The normalized interface profile $H = \min(\mathcal{P})/\sqrt{\log M}$ across dimension $n$ and mode count $M$. \textbf{(c)} At fixed $n \in \{1,2,3,4\}$ the profile grows as $M^{1/n}$. \textbf{(d)} At $n = 256$ the profile tracks $\sqrt{\log M}$, confirming Theorem~\ref{thm:dim}.}
\label{fig:synth}
\end{figure}

\section{Admissible Hierarchy: Full Definition and Proof of Theorem~\ref{thm:hier}}\label{app:thm3}

\begin{definition}[admissible categorical hierarchy]\label{def:adm}
A depth-$B$ admissible representation of $\mu = \sum_\ell \pi_\ell \delta_{c_\ell}$ consists of:
\textup{(1)}~a tree disintegration $\pi_u = \pi_v \pi_{u|v}$, $\pi_v > 0$;
\textup{(2)}~independent fresh source blocks $Z_b \sim \sigma_b$;
\textup{(3)}~semantic local transport $T_v$ with exact conditional phases, charged by Theorem~\ref{thm:gamma};
\textup{(4)}~metric normalization $\alpha \norm{a_{v,u} - a_{v,u'}} \le \dist(C_u, C_{u'}) \le \beta \norm{a_{v,u} - a_{v,u'}}$;
\textup{(5)}~immutable prefix commitment;
\textup{(6)}~decoder with additive sensitivity.
\end{definition}

\paragraph{Path-law induction.}
By induction on depth $b$: at $b = 0$, $V_0 = \varnothing$ and $\pi_\varnothing = 1$.
Assuming $\Prob(V_{b-1} = v) = \pi_v$ for all $v \in \mathcal{V}_{b-1}$, the event $\{V_b = u\}$ requires $\{V_{b-1} = v\}$ (where $v$ is the parent of $u$) and $T_{v,\Lambda_v}(Z_b) \in Q_{v,u}$.
By independence of $Z_b$ from the prefix and the exact phase constraint:
\[
\Prob(V_b = u) = \Prob(V_{b-1} = v) \cdot \sigma_b(T_{v,\Lambda_v}^{-1}(Q_{v,u})) = \pi_v \cdot \pi_{u|v} = \pi_u.
\]
Hence $X^* = c_{V_B} \sim \mu$.

\paragraph{Upper bound (i).}
By decoder sensitivity~(6):
\[
\|D_{\bm{v}}(Y_1, \ldots, Y_B) - c_{V_B}\|
\le \textstyle\sum_b S_{v_{b-1}} \dist(Y_b, \mathcal{A}_{v_{b-1}}).
\]
Minkowski's inequality over rounds gives
$W_p(\hat\mu, \mu) \le \sum_b (A_b/\Lambda_b)^{1/p}$
where $A_b = \sum_{|v|=b-1} \pi_v S_v^p \Jint_v$.

\paragraph{Continuous lower bound (iii).}
Any single continuous $\Lambda$-Lipschitz $F\!:\!\R^N\!\to\!\R^d$ satisfies
$\Lambda \cdot W_p^p(F\push\bar\sigma, \mu) \ge \Jint_{\mathrm{global}}$ by Theorem~\ref{thm:gamma}, and
$\Jint_{\mathrm{global}} \gtrsim \Delta^{p+1}\sqrt{\log M}$ by Theorem~\ref{thm:dim}.

\paragraph{Discontinuity (iv).}
By Lemma~\ref{lem:escape}: the factorized generator $F(z_1, \ldots, z_B) = c_{V_B(z)}$ outputs at least two atoms, hence is discontinuous on the connected domain $\R^{n_1 + \cdots + n_B}$, hence not Lipschitz.
The lower bound (iii) does not apply.

\section{Recent Continuous Few-Step Text Models: Full Case Analysis}\label{app:recent}

Throughout, a \emph{continuous latent} means a smooth embedding (nearby latents decode to related tokens, as in ELF), \emph{not} the categorical probability simplex whose vertices are one-hot token codes.

\paragraph{Case analysis: FMLM.}
Flow Map Language Models~\citep{fmlm2026} appear to refute us directly: a deterministic flow map that generates text in a single step.
They do not, because FMLM is an instance of our categorical escape rather than smooth continuous transport.
The flow interpolates between Gaussian noise and a \emph{one-hot} encoding on the vocabulary simplex, and the population-optimal denoiser output is the posterior probability over tokens followed by an $\argmax$ readout (their Lemma~3.1, cross-entropy trained).
The terminal map is thus the discontinuous categorical generator of Lemma~\ref{lem:escape}: ``continuous denoising'' toward one-hot vertices with $\argmax$ \emph{is} categorical commitment ($\CCI = 1$), not transport of a smooth latent.
Its reported open-ended quality (PPL $129$ at $1$ step, $76$ at $4$ steps on OpenWebText) is consistent with this: fluency comes from the categorical target geometry.
This placement is not merely conceptual.
$\CCI = 1$ holds by construction (the $\argmax$ readout of their Lemma~3.1 \emph{is} categorical commitment), and we measure DABI directly on the official FMLM-B-OWT checkpoint~\citep{fmlm2026}: under the same margin-normal probe a boundary-aligned perturbation flips ${\approx}50\%$ of tokens ($[42\%, 57\%]$) versus ${\approx}0\%$ for a norm-matched isotropic one (label-DABI $3{,}822\times$, $95\%$ CI $[1591, 24837]$; the isotropic response is at the measurement floor, so $\DABI \gg 10^{4}$).
This is the same sharp-readout signature as the AR and masked decoders.
Because FMLM is an encoder-less one-hot-simplex generator, there is no clean-text reconstruction operating point as for the autoencoder codecs, so the probe is applied around the model's own predictions; the readout sharpness is a property of the output projection regardless.
We use this only to place FMLM by readout sharpness and $\CCI$, \emph{not} as a reconstruction-accuracy or realized-residual test of Theorem~\ref{thm:noncommit}.
The sharp margin-normal readout coexists with successful one-step generation precisely because the one-hot target gives \emph{large} margins ($\delta^* = \Theta(1)$; the decoder-geometry ablation of Appendix~\ref{app:extended}): the realized residual stays in-cell, so categorical escape lives in the target geometry, not in a non-sharp readout.

\paragraph{The remaining field.}
Among genuinely smooth-continuous latents, every deterministic-ODE system needs \emph{many} steps: LangFlow~\citep{langflow2026} adopts a deterministic ODE for distillability and pays $128$--$1024$ steps (its authors note ``stochasticity inherently resists flow-based distillation''); Cosmos~\citep{cosmos2025} runs a deterministic Euler ODE at $200$ steps with no few-step results.
Every few-step continuous system instead injects stochasticity or commits categorically: ELF~\citep{he2026elf} reports SDE $\gg$ ODE in the few-step regime; CoLa-DLM~\citep{cola2026} pairs an SDE latent prior with an autoregressive categorical decoder; FastDiSS~\citep{fastdiss2026} uses stochastic ancestral sampling on \emph{conditional} seq2seq with a categorical readout; DiLaDiff~\citep{diladiff2026} consistency-distills only the latent prior, with text realized by a categorical decoder.
Independently, CoDAR~\citep{codar2026} names token rounding ``the primary bottleneck'' of continuous diffusion LMs and escapes via a categorical autoregressive decoder with a temperature knob, an external rediscovery of the readout-sharpness mechanism.
Positioned by readout geometry rather than by sampler, these systems are uniform along the dimension the reviewer's taxonomy cares about: all decode through a \emph{sharp categorical} head (nearest-neighbor embedding rounding for Cosmos/LangFlow/FastDiSS, an $\argmax$/autoregressive head for CoLa-DLM/DiLaDiff/CoDAR/FMLM), so all sit at high $\DABI$ and high $\CCI$.
They differ only in how they avoid the deterministic few-step regime our bound forbids: many ODE steps (Cosmos $200$, LangFlow $128$--$1024$), stochastic re-injection (ELF, FastDiSS), or an autoregressive/categorical spine (CoLa-DLM, DiLaDiff, CoDAR, FMLM). This is exactly the choice between the two escape columns of Table~\ref{tab:recent}.

\paragraph{Case analysis: training-time mitigations (FastDiSS).}
FastDiSS~\citep{fastdiss2026} is a continuous-latent diffusion language model (Gaussian diffusion over token embeddings) with two \emph{training-time} regularizers: Self-conditioning Perturbation (SCP), which noises the self-conditioning signal during training so the network is robust to the noisier estimates it sees at inference, and Model-aware Noise Scaling (MANS), which allocates more noise to tokens the model already reconstructs.
Neither changes the readout: discrete tokens are obtained by nearest-neighbor rounding to the embedding table, i.e.\ categorical commitment.
SCP/MANS are therefore orthogonal to our taxonomy rather than counterexamples: they harden the network against the compounding self-conditioning (posterior-mean) error that is itself a symptom of the non-commitment mechanism, while the final output still arrives through a categorical readout.
(The paper does not state whether its \emph{inference} sampler is deterministic or stochastic; we rely on neither.)

\paragraph{Case analysis: Loopholing.}
Loopholing Discrete Diffusion~\citep{loopholing2026} is a discrete/categorical diffusion model whose subtitle (``deterministic bypass of the sampling wall'') invites reading as a deterministic counterexample.
It is not: each denoising step emits \emph{both} a stochastic one-hot sample \emph{and} a deterministic continuous vector, and the deterministic latent pathway \emph{complements} rather than replaces the categorical spine, carrying soft distributional information across steps (a self-conditioning-style side-channel) so that information is not collapsed into one-hot vectors between steps.
The categorical commitment (escape (i)) is fully retained; ``deterministic'' here refers to feature propagation, not to eliminating the categorical sample.
Loopholing thus sits inside the categorical-escape column, augmented by a continuous side-channel, consistent with rather than contradicting our analysis.

\begin{table}[h]
\centering
\caption{Recent continuous-latent text generators by sampler and latent geometry. The deterministic-ODE $\times$ smooth-continuous $\times$ few-NFE cell is empty across the surveyed literature.}
\label{tab:recent}
\begin{tabular}{lll}
\toprule
sampler $\backslash$ latent & smooth-continuous & one-hot / simplex \\
\midrule
det.\ ODE, few-NFE ($\le 16$) & \emph{(empty --- our claim)} & FMLM (categorical) \\
det.\ ODE, many-NFE ($\ge 32$) & Cosmos ($200$), LangFlow ($128$--$1024$) & --- \\
stochastic SDE, few-NFE & ELF, CoLa-DLM, FastDiSS & discrete dLM family \\
\bottomrule
\end{tabular}
\end{table}

\section{Extended Experiments}\label{app:extended}

\begin{figure}[h]
\centering
\includegraphics[width=0.85\textwidth]{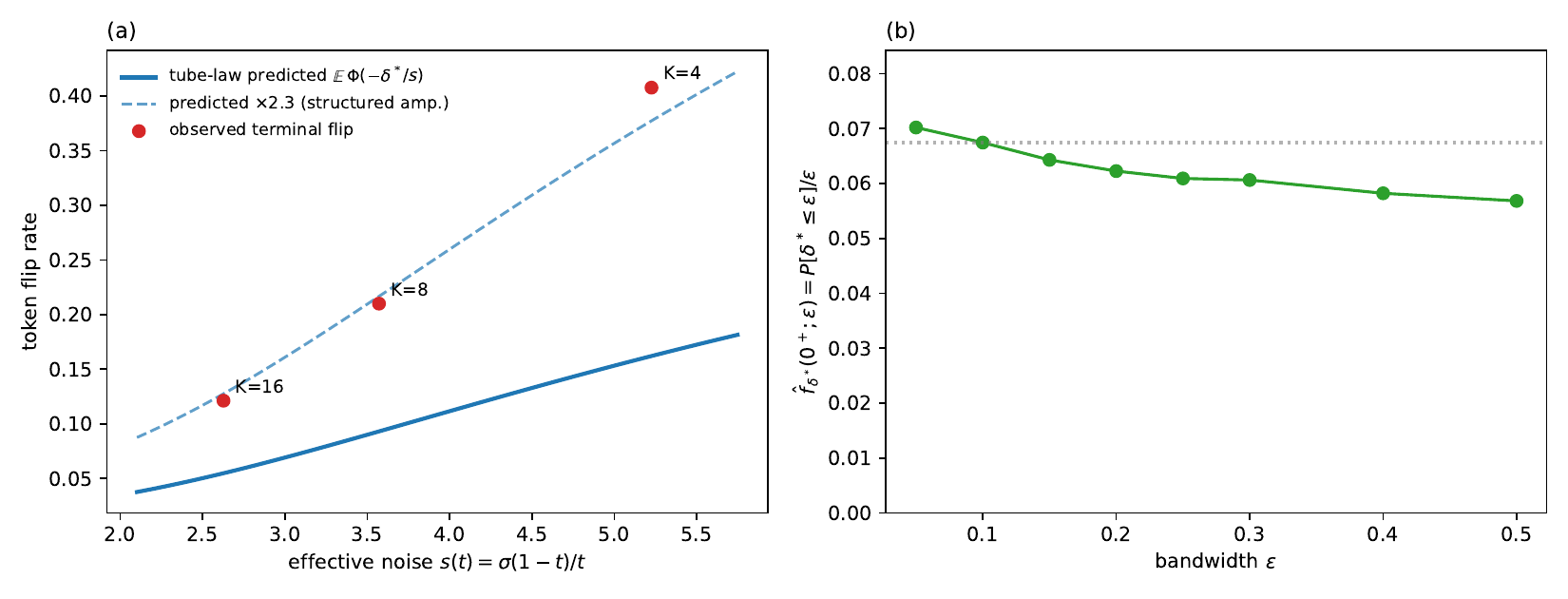}
\caption{\textbf{Tube-law test.}
\textit{(a)} Predicted versus observed terminal flip at $K \in \{4, 8, 16\}$; proportionality ${\approx}2.3\times$.
\textit{(b)} Margin density at zero $\hat f_{\delta^*}(0^+) \approx 0.064$.}
\label{fig:p1}
\end{figure}

\paragraph{LangFlow: second-system replication.}
We repeat the decoder-sensitivity and few-step protocols on the published LangFlow OpenWebText checkpoint~\citep{langflow2026} (embedding-space continuous diffusion, $d = 768$, GPT-2~\citep{radford2019gpt2} rounding readout), using $256$ real Wikitext sequences of length $128$ ($24{,}960$ positions).
\emph{DABI.} Clean decoding is $99.99\%$. At the posterior-mean residual of the terminal step ($K = 3$--$4$), the structured residual flips $59$--$69\%$ of tokens versus ${\approx}0.02\%$ for a norm-matched isotropic control, a flip ratio of $2{,}600\times$ ($95\%$ bootstrap CI $[1{,}800, 3{,}800]$); $\Delta\mathrm{CE}_{\mathrm{struct}} = 40.5\,[39.8, 41.1]$ versus isotropic $3.9\times10^{-4}$. The response is superlinear in the perturbation fraction (flip rate $0.02\% \to 2.3\% \to 59\%$ at $f = 0.25, 0.5, 1.0$ for $K=4$), matching the ELF onset.
\emph{Few-step degeneracy.} Deterministic Euler--EDM sampling is degenerate at low step counts: $K = 1$ attains generated-PPL $10.1$ only through mode collapse (entropy $1.7$ bits, repeated tokens), $K = 4$ produces incoherent text (PPL $1{,}401$), and $K \le 16$ produces repetitive fragments; fluency requires $K \ge 64$. This is an independent instance of the perplexity-collapse hazard: perplexity alone is minimized by degenerate outputs, so it must be read alongside entropy and samples. LangFlow is ODE-native, so we do not transfer the ELF SDE-versus-ODE comparison to it; its escape is step count.

\paragraph{Decoder geometry ablation.}
Fixing the ELF-B decoder and varying the codebook geometry (one-hot FMLM-style, VQ, smooth ELF) isolates the role of $\delta^*$ distribution.
One-hot readout: large margins ($\delta^* = \Theta(1)$), so the realized residual rarely reaches a boundary and the \emph{realized}-residual $\DABI \approx 1$, even though the margin-normal readout is itself sharp (this large-margin geometry is the categorical escape; cf.\ FMLM, Appendix~\ref{app:recent}).
Smooth ELF: $\delta^*$ has a heavy left tail, so the realized residual concentrates near boundaries and $\DABI = 45.7\times$.

\paragraph{Dynamics on the ELF text encoder.}
Integrated acceleration: $4.2\times$ isotropic-Gaussian control, peak $58\times$ at $t \approx 0.2$.
Euler local error: peak at $t \approx 0.2$, confirming Proposition~\ref{prop:eulerlocal}.

\paragraph{Teacher few-step decoding.}
Incoherent output for $K \le 16$; usable at $K \ge 32$--$64$; teacher ceiling at $K = 256$ (PPL $= 36.1$, entropy $= 3.45$).

\paragraph{Prefix conditioning.}
Conditioning on a clean prefix of up to $96/128$ latent tokens: the suffix acceleration peak is unchanged.
The multimodality is per-token ($\mu_0^{\otimes S}$-product); continuous latent factorization cannot escape.

\paragraph{Autoregressive commit ablation: details.}
Llama-2 (7B) on 200 prompts, max 128 tokens.
Hard-$\argmax$: diversity $d_2 = 0.541$, coherent text.
Soft-carry (logit distribution, no commitment): $d_2 = 0.135$, monotonic collapse into repetition.
Delayed commit (commit at position $b$): $d_2$ rises monotonically with $b$; collapse onset matches the branching point predicted by Theorem~\ref{thm:noncommit}.
Norm-matched soft-carry (rescaling soft embeddings to match hard embedding norms): still collapses, confirming the failure is discreteness, not norm.

\paragraph{Masked-diffusion commit ablation: details.}
LLaDA-8B, 128 steps.
Hard (standard): $d_2 = 0.82$, PPL $= 2.4$.
Soft-refresh (carry distribution, no commit, with resampling): $d_2 = 0.47$.
No-commit (carry distribution, no resample): $d_2 = 0.01$ (immediate collapse).
Dream-7B, 128 steps.
Hard-sample: PPL $= 7.3$, effective length $= 99$, cross-PPL $= 7.3$.
Soft-refresh-sample (remove only in-loop commit, keep sampling): PPL $= 38.6$, effective length $= 22$, cross-PPL $= 38.6$.
Matched-pair ablation isolates the commit/no-commit confound from the sample/greedy confound.
\emph{Protocol (Dream).} Generation length $128$; $\{128, 32, 8\}$ steps with timesteps $\mathrm{linspace}(1, \epsilon, K{+}1)$ and global \texttt{number\_transfer}; soft-distribution top-$k = 256$; no classifier-free guidance.
Hard variants commit by greedy $\argmax$ (temperature $0$) at each step; \texttt{*-sample} variants draw from the \emph{same} per-step categorical distribution (temperature $1$), so the matched pair (hard-sample vs.\ soft-refresh-sample) holds the sampling law fixed and removes only the in-loop commit; a norm-matched soft variant additionally controls for embedding-norm changes (all collapse).
Cross-PPL is scored by Qwen-2.5-3B and sequences are truncated at the first EOS to keep pad tokens from contaminating the degeneracy metrics; the decoding path matches Dream's official \texttt{generation\_utils.py} (right-shifted logits, global number transfer).
Table~\ref{tab:commit} collects the three families.

\paragraph{Interventions and metrics.}
The interventions hold the readout and the sampling distribution fixed and toggle only the in-loop categorical step. \emph{Hard} and \emph{hard-sample} commit each step by $\argmax$ (temperature $0$) or by a draw (temperature $1$); \emph{soft-refresh} carries the expected embedding $\E_p[e]$ instead of committing and recomputes it each step; \emph{no-commit} carries the soft embedding with neither resampling nor commitment; the \emph{norm-matched} soft variant rescales the soft embedding to the hard-embedding norm. Beyond perplexity and entropy we report three confound controls: the $4$-gram repetition rate (rep-$4$), distinct-$2$ diversity, and cross-perplexity under an external Qwen-2.5-3B scorer (coherence), together with the commitment fraction CCI-AUC (the path average of the hard-committed share). Removing commitment moves all of them together in the collapse direction: on LLaDA, hard${\to}$soft-refresh gives rep-$4$ $0.08{\to}0.47$, distinct-$2$ $0.82{\to}0.46$, cross-PPL $2.4{\to}18.7$; the Dream matched pair gives rep-$4$ $0.01{\to}0.38$, distinct-$2$ $0.97{\to}0.76$, cross-PPL $7.3{\to}38.6$, effective length $99{\to}22$, with CCI-AUC $0.50{\to}0$ by construction. Generation scripts and per-condition metrics for all three families are released.

\begin{table}[h]
\centering
\caption{\textbf{Commit ablations.} Holding the readout fixed and removing in-loop categorical commitment collapses generation across autoregressive and masked-diffusion families. The cleanest control is the Dream matched pair, which removes \emph{only} the in-loop commit step while keeping the sampling distribution identical.}
\label{tab:commit}
\small
\begin{tabular}{lllcccc}
\toprule
Model & Decoding & Commit variant & PPL & $d_2$ & eff.\ len & verdict \\
\midrule
Llama-2-7B & AR & hard $\argmax$ & --- & $0.541$ & --- & coherent \\
Llama-2-7B & AR & soft-carry (no commit) & --- & $0.135$ & --- & collapse \\
Llama-2-7B & AR & soft-carry, norm-matched & --- & $0.14$ & --- & collapse \\
\midrule
LLaDA-8B & masked & hard (standard) & $2.4$ & $0.82$ & --- & coherent \\
LLaDA-8B & masked & soft-refresh & --- & $0.47$ & --- & degrades \\
LLaDA-8B & masked & no-commit & --- & $0.01$ & --- & collapse \\
\midrule
Dream-7B & masked & hard-sample & $7.3$ & --- & $99$ & coherent \\
Dream-7B & masked & soft-refresh-sample & $38.6$ & --- & $22$ & collapse \\
\bottomrule
\end{tabular}
\end{table}

\section{Decoder Causal Intervention}\label{app:causal}
We test whether the few-step failure can be fixed at the readout. ELF's flow (transport) head and decoder (readout) head are separate modules; we finetune \emph{only} the readout (the post-transformer unembedding $\{$proj, unembed$\}$, $17$M parameters) on the published checkpoint, holding clean token recovery near $96.7\%$ with a recovery cross-entropy term. Because the transport map is untouched, the \emph{same} few-step latent is decoded by the original and the retrained readout, isolating the readout's contribution. Per-token unembedding columns are renormalized to their initial norms each step, so a margin objective cannot trivially rescale logits (a no-op for the normalized margin $\delta^*$ and for $\argmax$ decoding). Metrics use a held-out $512$-position subset, so the absolute $\DABI$ ($\approx 450$) is below the $508\times$ headline ($32{,}768$ positions); the \emph{relative} changes are what the test turns on.

\begin{table}[h]
\centering
\caption{\textbf{Decoder causal intervention} (ELF readout finetuned, transport frozen, clean recovery held). Terminal recon accuracy decodes the controlled posterior-mean roll-in latent of Theorem~\ref{thm:noncommit} to the correct token (baseline $59.3\% = 1 - 0.41$ terminal flip at $K{=}4$, matching Figure~\ref{fig:rollin}). No intervention recovers the flipped tokens.}
\label{tab:causal}
\small
\begin{tabular}{lccccc}
\toprule
Intervention & med.\ margin & $\delta^*$ & $\DABI$@$\kappa{=}1$ & clean acc & terminal recon ($K{=}4$) \\
\midrule
Baseline (frozen) & $22$ & $5.4$ & $453\times$ & $1.00$ & $59.3\%$ \\
Isotropic smoothing & --- & --- & $48$--$58\times$ & $0.99$ & --- \\
Max-margin & $70$ & $6.3$ & $1{,}469\times$ & $0.98$ & --- \\
Residual-targeted & $6$ & $3.0$ & $422\times$ & $0.96$ & $60.5\%$ \\
\bottomrule
\end{tabular}
\end{table}

\emph{Max-margin} maximizes the smooth-max normalized margin. It raises the median raw margin ($22 \to 70$) but \emph{increases} $\DABI$ ($453 \to 1{,}469\times$) and degrades generation at every step count ($K{=}4$ PPL $799 \to 895$): a more confident readout has a \emph{sharper} decision cliff, so a few-step residual that crosses a boundary incurs a larger cross-entropy jump, not a smaller one.

\emph{Residual-targeted} training instead minimizes the cross-entropy of decoding the realized few-step latents to the correct tokens, the most direct ``make the readout decode the few-step latent'' objective. On the controlled terminal posterior-mean step (baseline flip $41\%$), $4{,}000$ steps \emph{collapse} the readout margins ($\delta^*: 5.4 \to 3.0$, a much less sharp readout) yet raise held-out token accuracy by only $1.2$ points ($59.3 \to 60.5\%$). On uncoupled few-step generation, perplexity moves only through mode collapse: the retrained readout's outputs are degenerate repetitions (entropy $4.4 \to 3.2$, e.g.\ \texttt{/s/s/s}\dots), so the apparent change is the perplexity-collapse artifact, not recovery.

\emph{Nonlinear readout.} The residual-targeted test retrains a \emph{linear} readout; to rule out a linearity artifact, we replace it with a $2$-layer MLP (GELU, $1024$ hidden, $34$M parameters) trained on the identical residual objective and the same frozen terminal hidden. With strictly more capacity it overfits the training tokens to $99.6\%$, yet its held-out terminal recon is $56.8\%$, no better than a linear probe on the same hidden ($57.8\%$) or the model's own readout ($59.9\%$); the gap is the overfitting, not recovery. Added nonlinearity recovers none of the flipped tokens, which is the information-theoretic content of Theorem~\ref{thm:noncommit}: once the terminal step averages over the branch ambiguity, the correct branch is not a measurable function of the terminal latent, so no readout, linear or nonlinear, can recover it on held-out data.

The conclusion is consistent across all four interventions (isotropic smoothing, max-margin, and linear and nonlinear residual-retraining): the readout-only interventions we tested do not recover the tokens flipped by the deterministic terminal step. Under the posterior-mean model, once the terminal step averages over the branch ambiguity and lands in the $O(s(t))$ boundary tube, the branch identity is not present as a deterministic function of the terminal latent, so a readout-only fix is not expected. This places the failure on the transport side, with decoder sharpness as its faithful amplifier (Theorem~\ref{thm:noncommit}).

\paragraph{Adversarial/margin training along boundary normals.}
Targeted adversarial or margin training that pushes apart the boundary normals is exactly the max-margin intervention above: it widens the raw margins ($\delta^*\!: 5.4\to6.3$) but \emph{raises} $\DABI$ ($453\to1{,}469\times$) and degrades generation at every step count, because a more confident readout has a sharper decision cliff, so a boundary-crossing residual incurs a larger jump. Reshaping the boundary geometry faces the same wall as the nonlinear-readout head: once the posterior mean averages over the branch ambiguity, the correct token is not a function of the latent, so no boundary reshaping recovers it without distorting clean decoding (which collapses margins, the residual-targeted arm). Decoder-side margin engineering thus does not reduce few-step failure without sacrificing diversity or clean accuracy; the leverage is on the transport (sampler) side.

\paragraph{Direction and norm sensitivity of DABI.}
The margin-normal direction is the steepest-descent direction of the active margin, hence the worst-case boundary-aligned perturbation; $\DABI$ is therefore an upper envelope over locally adversarial directions rather than a knife-edge artifact of one choice. The response is also robust to the norm calibration $f$: the structured curve is monotone with onset near $f\approx0.8$ and reaches a $0.86$ flip rate by $f{=}1$ on ELF, while the isotropic curve stays at the floor throughout (Figure~\ref{fig:decsens}), so $f$ slightly below or above $1$ moves the absolute rate smoothly without changing the orders-of-magnitude gap.

\begin{figure}[t]
\centering
\includegraphics[width=0.85\textwidth]{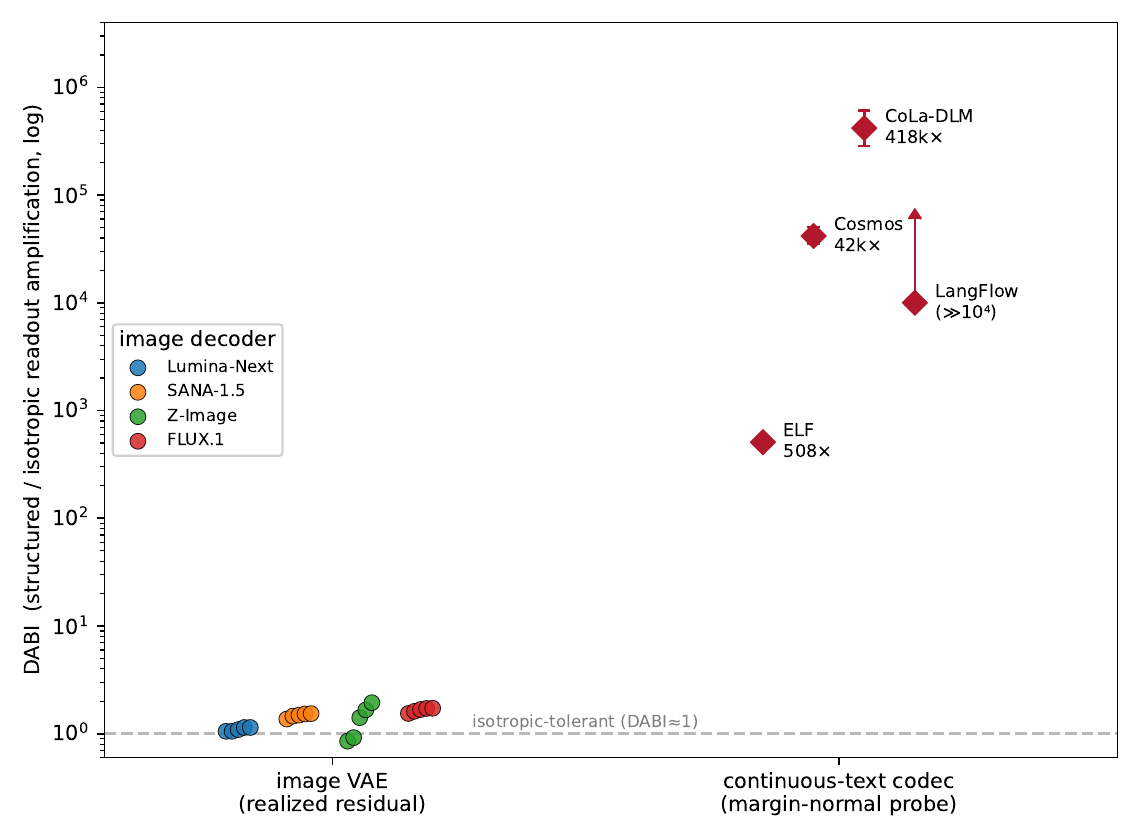}
\caption{\textbf{DABI image/text dichotomy} (visualizes Table~\ref{tab:dabi_roster}).
Four continuous-text codecs under the margin-normal probe: ELF $508\times$, Cosmos ${\approx}4{\times}10^{4}$, CoLa-DLM ${\approx}4{\times}10^{5}$, LangFlow ${\gg}10^{4}$ (isotropic at floor).
Image VAEs (Lumina-Next, SANA-1.5, Z-Image, FLUX.1), lacking a categorical readout, absorb their realized terminal residual at $\DABI \approx 1$ ($[0.85, 1.94]$).
The clusters are separated by orders of magnitude with no overlap.}
\label{fig:dabi_scatter}
\end{figure}

\paragraph{Image-side robustness to the loss metric.}
The image $\DABI \approx 1$ is not an artifact of the pixel-$L^2$ readout loss. With the learned LPIPS perceptual metric~\citep{zhang2018lpips} (AlexNet backbone, on the same decoded images) in place of pixel-$L^2$, the realized-residual probe on Lumina-Next, SANA-1.5, and FLUX.1 gives structured/random ratios of $0.93$--$1.08$ (Lumina), $1.15$--$1.25$ (SANA), and $0.66$--$0.95$ (FLUX) across $K \in \{1,2,4,8\}$: all ${\approx}1$, the same regime as pixel-$L^2$ ($[0.85,1.94]$), with no boundary amplification. A semantically-aware metric that up-weights perceptually salient changes still finds no categorical boundary to align with: the cross-domain gap is a property of the readout geometry, not of the choice of output-space loss.

\paragraph{Vector-quantized image readout (cross-domain mechanism check).}
A vector-quantized decoder reads out by nearest-codebook lookup, $\mathrm{code}(z)=\argmin_k\|z-e_k\|^2=\argmax_k(2e_k^\top z-\|e_k\|^2)$, a linear-argmax readout of the same form as a text decoder's $D_W(z)=\argmax_y w_y^\top z$ (here $w_k=2e_k$). We apply the margin-normal probe directly to the published MoVQGAN codebook ($16{,}384$ codes, latent dimension $d=4$), sampling latents near codebook points. A boundary-aligned perturbation flips the assigned code in $50.3\%$ of cases at $f{=}1$ versus $0.3\%$ for a norm-matched isotropic one (flip ratio $167\times$), with code cross-entropy $\DABI=42\times$: the same sharp-readout signature as the text and AR decoders, now in an image model. Because the VQ latent is only $4$-dimensional, the isotropic null space is tiny, so the $167\times$ gap cannot be attributed to null-space avoidance; it is boundary-normal alignment. A smooth image VAE has no such categorical boundary ($\DABI\approx1$); endowing an image decoder with one reproduces the failure signature, confirming the mechanism is the categorical readout, not the modality.

\paragraph{Absolute flip rates (not only ratios).}
The large $\DABI$ ratios are not ratio-inflation artifacts of a tiny denominator: the absolute structured flip rates are themselves high while the isotropic ones sit at the floor. At $\kappa{=}1$ the boundary-aligned perturbation flips $85.9\%$ of tokens on ELF (isotropic $3.1\%$), $100\%$ on LLaDA-8B and Dream-7B (isotropic $0.0\%$), ${\approx}50\%$ on FMLM (isotropic ${\approx}0\%$), and $59$--$69\%$ on LangFlow (isotropic $0.02\%$). The conclusion rests on the absolute structured flips being near-total and the isotropic flips being near-zero; the ratio only summarizes a gap that is already evident in the raw rates.

\paragraph{Row-space control (anisotropy beyond the null space).}
The realized residual is ${\approx}99\%$ in the decoder null space, so the isotropic control of Figure~\ref{fig:decsens} spends most of its norm where the readout is blind. To isolate anisotropy \emph{within} the row space, we sketch the full cross-position row space $\mathrm{row}(J)$ of the logit Jacobian $J=\partial(\text{all logits})/\partial(\text{all latents})$ (the active-pair boundary normals at every content position plus $k$ random-logit-cotangent VJPs, $k\in\{64,256\}$; per-sample SVD basis) and compare, at matched total norm, the structured margin-normal push (boundary-aligned, lying in $\mathrm{row}(J)$) against a random direction drawn \emph{within} $\mathrm{row}(J)$. On $4{,}096$ ELF positions the boundary-aligned push flips $46.9\%$ of tokens, a norm-matched random row-space direction flips $3.2\%$, and an isotropic full-space direction flips $0.9\%$. The structured-vs-isotropic ratio ($50\times$) reproduces the realized-residual DABI ($45.7\times$, a consistency check); the structured-vs-random-row ratio is $14.5\times$ ($15.3\times$ at $k{=}64$, so stable in the sketch size). Row-space membership accounts for only a ${\approx}3.5\times$ factor over isotropic; the remaining order of magnitude is specific boundary-normal alignment, so the damage is anisotropy, not mere row-space energy. We work in the first-order margin-normal regime, where the row/null decomposition is valid (the full realized residual is a large, partly nonlinear step); code is released.

\begin{figure}[t]
\centering
\includegraphics[width=0.82\textwidth]{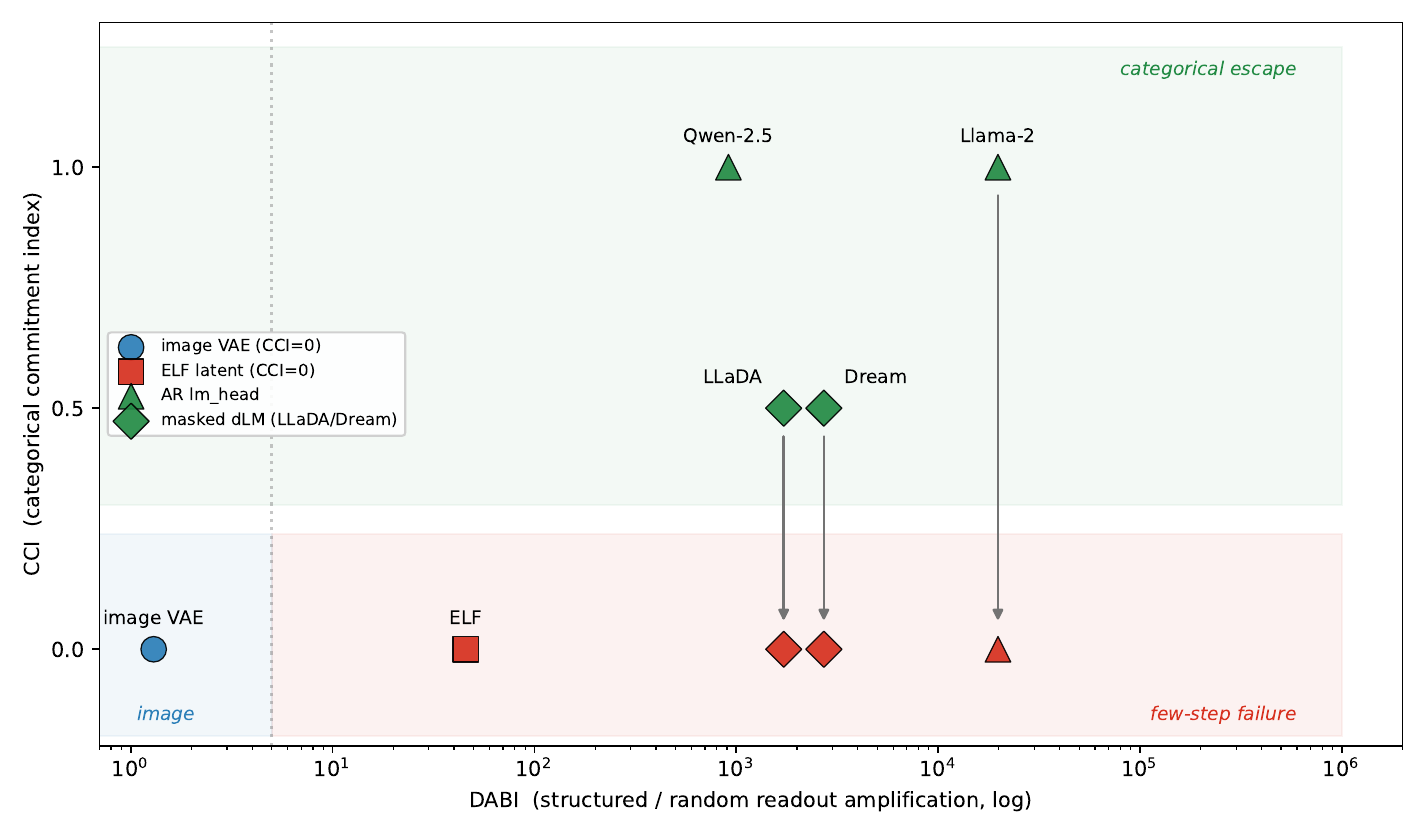}
\caption{\textbf{$\DABI \times \CCI$ taxonomy.}
Image VAEs: low $\DABI$, works. ELF: high $\DABI$, $\CCI = 0$, fails. AR/masked dLM: high $\DABI$, high $\CCI$, works. Arrows: commit-ablation collapse.}
\label{fig:taxonomy}
\end{figure}

\begin{figure}[t]
\centering
\includegraphics[width=0.9\textwidth]{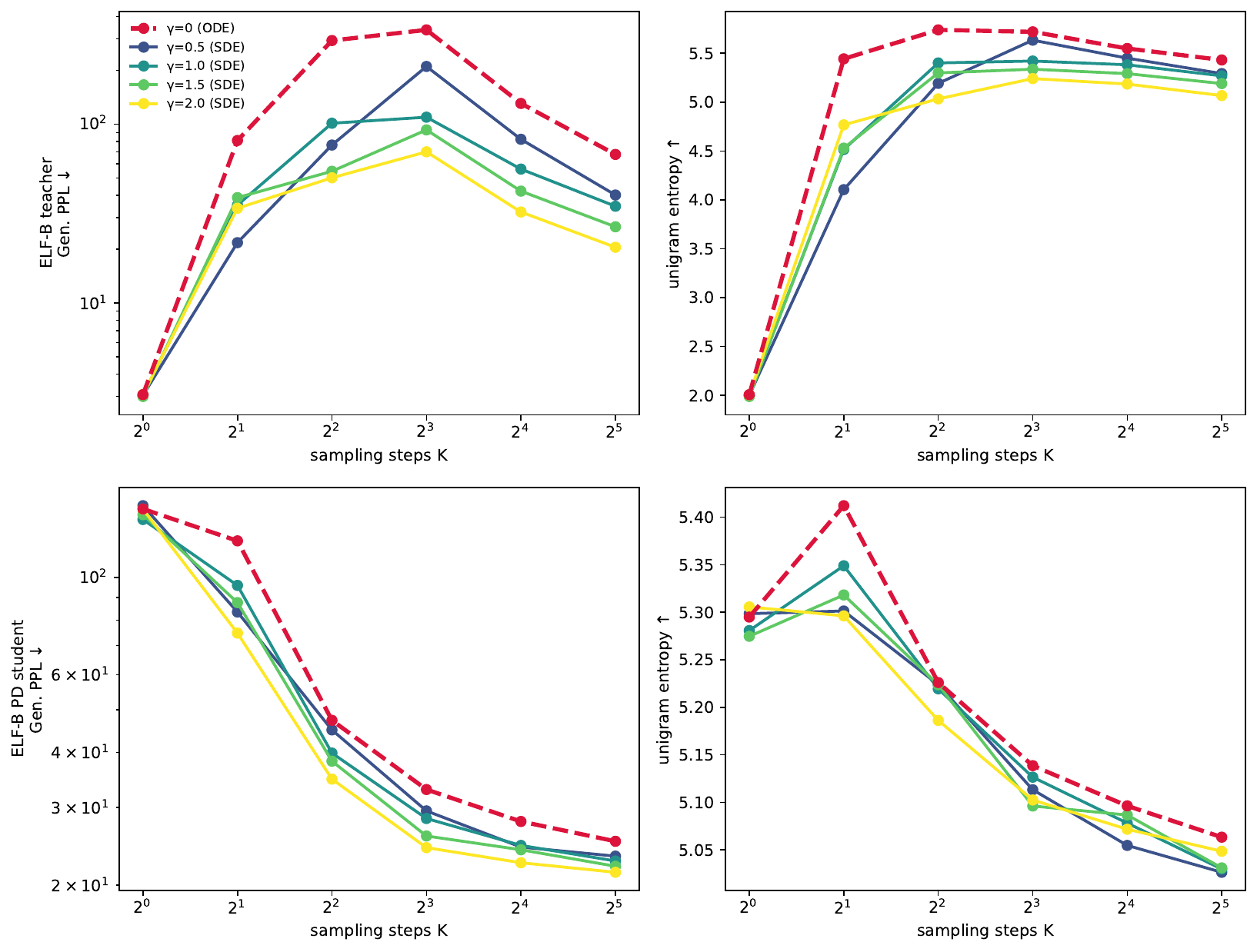}
\caption{\textbf{$\gamma$-sweep: ODE versus SDE.} ELF-B teacher (top) and PD student (bottom).
\textit{Left:} PPL versus $K$; ODE ($\gamma = 0$) stays high, SDE ($\gamma > 0$) escapes.
\textit{Right:} Entropy; teacher $K = 1$ ODE has PPL $3.07$ but entropy $2.01$ (mode collapse).
An independent $5$-seed rerun is stable: the ODE/SDE ordering and multilingual-collapse signature hold across seeds (per-cell SD ${<}10\%$).}
\label{fig:gamma}
\end{figure}

\section{Extended Related Work}\label{app:related}

\paragraph{Flow matching and rectified flow.}
\citet{lipman2023flow} and \citet{liu2023rectified} introduced flow matching and rectified flow for learning deterministic transport between distributions.
Consistency models \citep{song2023consistency} learn single-step generators via consistency training.
Progressive distillation \citep{salimans2022progressive} reduces the step count of a pretrained diffusion model.
All achieve few-step generation of images; none demonstrate few-step generation of text latents without stochastic sampling.

\paragraph{Lipschitz limitations for generative models.}
\citet{salmona2022lipschitz} showed that Lipschitz pushforward measures are limited on multimodal targets and that per-step noise injection can help.
Our Theorem~\ref{thm:gamma} provides the finer multiway, dimension-sensitive interface law.
The Lipschitz-constrained perspective also appears in Wasserstein GANs \citep{arjovsky2017wasserstein}, where the constraint is on the discriminator; here it is on the generator.

\paragraph{Adversarial robustness and margin geometry.}
DABI is, geometrically, a boundary-aligned sensitivity analysis: it displaces a latent along the readout normal $w_i - w_j$ and measures how few units flip the decoded token, precisely the small, boundary-normal perturbations that define adversarial examples~\citep{szegedy2014intriguing,goodfellow2015fgsm} and whose worst case PGD~\citep{madry2018pgd} searches for. That literature studies how to \emph{shrink} this sensitivity (large-margin and certified training); we use it as a diagnostic and in the opposite direction. A sharp text readout has small normalized margins ($\delta^*$ of order the residual scale near the boundary), so it is \emph{maximally} susceptible to a structured, boundary-aligned displacement, and the posterior-mean transport error supplies exactly such a displacement for free, an ``adversarial'' perturbation produced by the generator rather than by an attacker (Theorem~\ref{thm:noncommit}). The image-decoder $\DABI \approx 1$ is the same statement read the other way: a smooth decoder with no categorical readout has no low-margin boundary to attack. Few-step text failure is thus the generator inadvertently constructing a near-worst-case input for its own readout, which is why decoder-side margin shaping (Appendix~\ref{app:causal}) does not help: it moves the boundary the residual is already aligned to.

\paragraph{$\Gamma$-convergence and phase transitions.}
The Modica--Mortola $\Gamma$-limit \citep{modica1977esempio} and its multiwell generalizations \citep{baldo1990,fonseca1989gradient} are the soft-gradient analogue of our Theorem~\ref{thm:gamma}.
The difference is the constraint: we cap the pointwise operator norm of the gradient (a Lipschitz condition) rather than penalizing it with a double-well potential.

\paragraph{Gaussian isoperimetry and Gaussian widths.}
The Gaussian isoperimetric inequality \citep{borell1975,sudakov1978extremal} and generic chaining \citep{talagrand2005chaining} are the standard tools for our Theorem~\ref{thm:dim}.
The Milman--Neeman multi-bubble theorem \citep{milman2022gaussian} gives the exact minimizer when $M \le n + 1$; our Fourier-code construction complements this by handling $n \ll M$.

\paragraph{Text generation via continuous latents.}
ELF \citep{he2026elf} encodes text into a continuous latent and generates via flow matching.
FMLM \citep{fmlm2026} uses one-hot encoding with $\argmax$ decoding, an instantiation of categorical escape.
RAE \citep{rae2025} uses a frozen semantic encoder with a trained decoder and high-dimensional latent diffusion at $50$ NFE, consistent with our finding that many steps are needed for continuous text latents.
The concurrent ELF+PD \citep{he2026elfpd} confirms the ODE-vs-SDE gap on the same model we study, with the authors explicitly noting that SDE corrects ``deterministically amplifying imperfect trajectories.''

\paragraph{Autoregressive and masked diffusion language models.}
Autoregressive models \citep{brown2020language,touvron2023llama2} generate tokens left-to-right with categorical commitment at each position.
Masked diffusion language models \citep{sahoo2024mdlm,nie2025llada,dream2025} generate by iteratively remasking and predicting discrete tokens.
Both escape the continuous lower bounds via categorical commitment (Lemma~\ref{lem:escape}).

\paragraph{Margin geometry and interface diagnostics.}
DABI is an interface audit of the decoder: it measures sensitivity along the readout-normal direction, the same boundary-aligned geometry studied by margin-based robustness and Lipschitz-margin bounds for classifiers~\citep{szegedy2014intriguing,goodfellow2015fgsm,madry2018pgd}, where a small normalized margin certifies vulnerability to a boundary-normal perturbation. Our contribution is to turn this diagnostic on the \emph{generator}: the transport residual supplies the boundary-aligned perturbation for free, so a sharp decoder (small margins) is exactly the regime a few-step generator cannot satisfy. The image/text dichotomy and the VQ control (Appendix~\ref{app:causal}) read this interface across modalities; CCI complements it with the orthogonal commitment axis.

\paragraph{Samplers that re-inject categorical or endpoint uncertainty.}
Several continuous-text samplers restore uncertainty during sampling rather than committing deterministically: predictor--corrector and bridge-style samplers add a corrector or a stochastic bridge step that re-injects endpoint or categorical noise. In our taxonomy these are instances of the stochastic re-injection escape (Section~\ref{sec:stochastic}): leaving the deterministic transport class is precisely what defeats the lower bound, whether the injection is Gaussian (SDE), categorical (commitment), or a corrector that re-randomizes the endpoint. The bound constrains the deterministic-continuous corner; every such sampler steps out of it by construction.

\end{document}